\documentclass[10pt,journal,twocolumn]{IEEEtran}

\usepackage{url}
\usepackage{textcomp}
\usepackage{footnote}
\usepackage{url}
\usepackage{color}
\usepackage{amsmath,amsfonts,amssymb,amsthm,bbm}
\usepackage{commath}
\usepackage{mathtools}
\usepackage{thmtools}
\usepackage{bm}
\usepackage[T1]{fontenc}
\usepackage{graphicx}
\usepackage{subfigure}
\usepackage{booktabs}
\usepackage{multirow}
\usepackage{pst-all}
\usepackage{pstricks}
\usepackage{pstricks-add}
\usepackage{pst-node}
\usepackage{pst-grad}
\usepackage{pst-plot}
\usepackage{pstricks-add}
\usepackage{pst-sigsys}
\usepackage{tikz-cd}
\usepackage{float}
\usepackage{algorithm}
\usepackage{algpseudocode}
\usepackage{hyperref}
\usepackage[style=ieee,url=false,doi=false]{biblatex}

\ifodd 1

\newcommand{\revOld}[1]{} 
\newcommand{\com}[1]{\textbf{\color{blue}  (COMMENT: #1)}}
\else

\newcommand{\revOld}[1]{}
\newcommand{\com}[1]{}
\fi

\definecolor{red}{RGB}{153,0,0}
\definecolor{green}{RGB}{0,153,0}			
\definecolor{darkred}{RGB}{90,0,0}
\definecolor{darkgreen}{RGB}{0,90,0}
\definecolor{darkblue}{RGB}{0,0,90}			

\definecolor{econ-red}{HTML}{E3120B}
\definecolor{econ-red60}{HTML}{F6423C}
\definecolor{econ-prim-chicago20}{HTML}{141F52}
\definecolor{econ-prim-chicago30}{HTML}{1F2E7A}
\definecolor{econ-prim-chicago45}{HTML}{2E45B8}
\definecolor{econ-prim-chicago55}{HTML}{475ED1}
\definecolor{econ-prim-chicago90}{HTML}{D6DBF5}
\definecolor{econ-prim-chicago95}{HTML}{EBEDFA}
\definecolor{econ-sec-hongkong45}{HTML}{1DC9A4}
\definecolor{econ-sec-hongkong55}{HTML}{36E2BD}
\definecolor{econ-sec-hongkong90}{HTML}{D2F9F0}
\definecolor{econ-sec-hongkong95}{HTML}{E9FCF8}
\definecolor{econ-sec-tokyo45}{HTML}{C91D42}
\definecolor{econ-sec-tokyo55}{HTML}{E2365B}
\definecolor{econ-sec-tokyo90}{HTML}{F9D2DB}
\definecolor{econ-sec-tokyo95}{HTML}{FCE9ED}
\definecolor{econ-ter-singapore55}{HTML}{F97A1F}
\definecolor{econ-ter-singapore65}{HTML}{FB9851}
\definecolor{econ-ter-singapore75}{HTML}{FCB583}
\definecolor{econ-ter-singapore90}{HTML}{FEE1CD}
\definecolor{econ-ter-newyork55}{HTML}{F9C31F}
\definecolor{econ-ter-newyork65}{HTML}{FBD051}
\definecolor{econ-ter-newyork75}{HTML}{FCDE83}
\definecolor{econ-ter-newyork90}{HTML}{FEF2CD}
\definecolor{econ-greyscale-london5}{HTML}{0D0D0D}
\definecolor{econ-greyscale-london10}{HTML}{1A1A1A}
\definecolor{econ-greyscale-london20}{HTML}{333333}
\definecolor{econ-greyscale-london35}{HTML}{595959}
\definecolor{econ-greyscale-london70}{HTML}{B3B3B3}
\definecolor{econ-greyscale-london85}{HTML}{D9D9D9}
\definecolor{econ-greyscale-london95}{HTML}{F2F2F2}
\definecolor{econ-greyscale-london100}{HTML}{FFFFFF}
\definecolor{econ-canvas-la85}{HTML}{E1DFD0}
\definecolor{econ-canvas-la90}{HTML}{EBE9E0}
\definecolor{econ-canvas-la95}{HTML}{F5F4EF}
\definecolor{econ-canvas-paris85}{HTML}{D0E1E1}
\definecolor{econ-canvas-paris90}{HTML}{E0EBEB}
\definecolor{econ-canvas-paris95}{HTML}{EFF5F5}

\hypersetup{colorlinks,urlcolor=darkred,linkcolor=darkred,citecolor=green}

\newfloat{algorithm}{t}{lop}

\theoremstyle{plain}
\newtheorem{Theorem}{Theorem}

\newtheorem{Lemma}[Theorem]{Lemma}
\newtheorem{Assumption}[Theorem]{Assumption}

\newboolean{draft}
\newcommand{\isdraft}[2]{\ifthenelse{\boolean{draft}}{#1}{#2}}

\isdraft{\usepackage{setspace}}{}                              
\isdraft{\usepackage{paralist}}{}                              
\isdraft{}{}

\allowdisplaybreaks
\newcounter{MYtempeqncnt}

\newcommand{\mypar}[1]{{\bf #1.}}

\begin{document}
\title{The Rate-Distortion-Perception-Classification Tradeoff: Joint
Source Coding and Modulation via Inverse-Domain GANs}
\author{Junli Fang$^\star$,\, Jo\~ao F. C. Mota$^{\star}$,\, Baoshan Lu,\, Weicheng Zhang,\, Xuemin Hong
\thanks{
Junli Fang, Weicheng Zhang, and Xuemin Hong are with the School of
Informatics, Xiamen University, Xiamen, China.
(e-mail: \{junlifang,zhangweicheng\}@stu.xmu.edu.cn, xuemin.hong@xmu.edu.cn)

Baoshan Lu is with the School of Electronics and Information Engineering,
Guangxi Normal University, Guilin, China.
(e-mail: baoshanlu@gxnu.edu.cn)

Jo\~{a}o F. C. Mota is with the School of Engineering \& Physical Sciences,
Heriot-Watt University, Edinburgh EH14 4AS, UK. 
(e-mail: j.mota@hw.ac.uk)

Work supported in part by the National Natural Science Foundation of China
under Grant 62077040, by Guangxi Natural Science Foundation (Grant
2024GXNSFBA010309), 
and by UK's EPSRC New
Investigator Award (EP/T026111/1).

$^\star$equal contribution

Paper accepted in IEEE Transactions on Signal Processing, 2024.
}}

\maketitle
\begin{abstract}
  The joint source-channel coding (JSCC) framework leverages deep learning to
  learn from data the best codes for source and channel coding. 
  When the output signal, rather than being binary, is directly mapped onto the
  IQ domain (complex-valued), we call the resulting framework joint source
  coding and modulation (JSCM).
  We consider a JSCM scenario and show the existence of a
  strict tradeoff between channel rate, distortion, perception, and
  classification accuracy, a tradeoff that we name RDPC. We then propose two
  image compression methods to navigate that tradeoff: the RDPCO algorithm
  which, under simple assumptions, directly solves the optimization problem
  characterizing the tradeoff, and an algorithm based on an inverse-domain
  generative adversarial network (ID-GAN), which is more general and achieves
  extreme compression. Simulation results corroborate the theoretical findings,
  showing that both algorithms exhibit the RDPC tradeoff. 
  They also demonstrate
  that the proposed ID-GAN algorithm effectively balances image distortion,
  perception, and classification accuracy, and significantly outperforms
  traditional separation-based methods and recent deep JSCM architectures in
  terms of one or more of these metrics.
\end{abstract}

 
\section{Introduction}

Traditional communication systems follow the celebrated source-channel coding
theorem by Shannon~\cite{cover1999elements}, which states that source coding
and channel coding can be designed separately without loss of optimality.
Source coding removes redundant information from a signal, for example, by
representing it in a different domain and zeroing out small coefficients.
Channel coding, on the other hand, adds to the resulting compressed signal
additional information, error-correcting codes, to make its transmission via a
noisy channel more robust. Such a modular design, while optimal for memoryless
ergodic channels with codes of infinite block length, becomes unsuitable for
extreme scenarios, e.g., when bandwidth is highly limited or the channel varies
rapidly. An example is underwater acoustic communication, in which multipath
interference and noise are so large that the performance of separate
source-channel coding schemes sharply drops below a certain signal-to-noise
ratio (SNR), a phenomenon known as \textit{the cliff effect}~\cite{deep-jscc}.
Traditional techniques like fast adaptive modulation and channel coding rarely
work in such an environment, especially for long source bit sequences like
images. 

\mypar{Joint source coding and modulation} 
The above problem can be addressed by jointly designing the source coding, channel
coding,  and modulation schemes, a framework we call \textit{joint source
coding and modulation} (JSCM). JSCM directly maps signals to the IQ domain and
generalizes \textit{joint source and channel coding}
(JSCC)~\cite{fresia2010joint}, in which the output signal is binary rather than
complex-valued.

In the context of large signals (like images), and under extreme compression
requirements (as in underwater communication), the selection of the features to
be compressed strikes tradeoffs between different metrics: for
example, optimizing for image reconstruction may reduce the perceptual quality
of the reconstructed image or decrease the accuracy of a subsequent image
classification algorithm. 
The study of such type of tradeoffs started with the
seminal work in~\cite{blau2018perception}, which modified the rate-distortion
framework~\cite{cover1999elements} to study the tradeoff between perception and
distortion metrics of image restoration algorithms.
In this paper, we extend the study of such tradeoffs to a JSCM scenario. This
requires considering not only vector signals, and thus the possibility to
reduce their dimension, but also various metrics, including rate, distortion,
perception, and classification performance. 

\subsection{The RDPC function and problem statement}
\label{subsec:rdpcfunction-probstatement}

We introduce the rate-distortion-perception-classification (RDPC)
function in a JSCM scenario. To define it, we consider an
$n$-dimensional source signal $\bm{X} \in \mathbb{R}^n$ that can be drawn from
one of $L$ classes:
\begin{align}
  \label{eq:SourceModel}
  \bm{X} \vert H_l \sim p_{\bm{X}\vert H_l}\,,\qquad l=1,\ldots, L\,,
\end{align}
where $H_l$ represents the hypothesis that $\bm{X}$ is
drawn from class $l$, which occurs with probability $p_{l}:=\mathbb{P}(H_l)$.
The communication process is modeled as a Markov chain
\begin{equation}
  \label{eq:channeldiagram}
  \begin{tikzcd}[sep=normal]
    H_l
    \arrow{r}{p_{\bm{X}\vert H_l}\,\,\,}
    &
    \bm{X}
    \arrow{r}{p_{\bm{Y}\vert\bm{X}}}
    &
    \bm{Y}
    \arrow{r}{p_{\widehat{\bm{Y}}\vert\bm{Y}}}
    &
    \widehat{\bm{Y}}
    \arrow{r}{p_{\widehat{\bm{X}}\vert\widehat{\bm{Y}}}}
    &
    \widehat{\bm{X}}\,,
  \end{tikzcd}
\end{equation}
where $\bm{X}$, $\widehat{\bm{X}} \in \mathbb{R}^n$ represent the source and
reconstructed signals, and $\bm{Y}$, $\widehat{\bm{Y}} \in \mathbb{R}^m$, with
$m < n$, represent the transmitted and received signals. The distribution
$p_{\bm{Y}\vert\bm{X}}$ (resp.\ $p_{\widehat{\bm{Y}}\vert\bm{Y}}$ and
$p_{\widehat{\bm{X}}\vert\widehat{\bm{Y}}}$) characterizes the encoder (resp.\
channel and decoder).
We assume the channel adds zero-mean Gaussian noise to $\bm{Y}$, i.e.,
$p_{\widehat{\bm{Y}}\vert\bm{Y}}(\widehat{\bm{y}}\vert \bm{y}) =
\mathcal{N}\big(\bm{y}\,, \bm{\Sigma}\big)$, where $\bm{\Sigma}$ is an $m\times
m$ diagonal matrix whose $i$th diagonal entry $\bm{\Sigma}_{ii} > 0$
represents the \textit{noise power} of channel $i$.
Motivated by recent work on task-aware image
compression~\cite{Hua22-SemanticsGuidedContrastiveJointSourceChannelCodingForImageTransmission,Lei23-ProgressiveDeepImageCompressionWithHybridContextsOfImageClassificationAndReconstruction,Hua24-ClassificationDrivenDiscreteNeuralRepresentationLearningForSemanticCommunications}, the purpose of the communication channel 
in~\eqref{eq:channeldiagram} is to transmit images that can be used across
different tasks, each of which may have different requirements in terms of
image fidelity, perception, or classification. Considering classification as
a task justifies the assumption in~\eqref{eq:SourceModel} that $\bm{X}$ always
belongs to a given class.

Our goal is to design the encoder-decoder pair $(p_{\bm{Y}\vert\bm{X}},
p_{\widehat{\bm{X}}\vert\widehat{\bm{Y}}})$ and
the noise power $\bm{\Sigma}$ so that the channel rate is minimized
while satisfying three constraints: 
\begin{align}
  \label{eq:problem-general}
  R (D,P,C) 
  = 
  \begin{array}[t]{cl}
    \underset{p_{\bm{Y}\vert\bm{X}},\,
    p_{\widehat{\bm{X}}\vert\widehat{\bm{Y}}},\,\bm{\Sigma}}
    {\min} 
    & 
    \sum_{i=1}^{m}\log \big(1+\frac{1}{\bm{\Sigma}_{ii}}\big)
    \\ 
    \text{s.t.} 
    & 
    \mathbb{E}\big[\Delta(\bm{X}, \widehat{\bm{X}})\big] \leq D
    \\[0.1cm]
    & 
    d(p_{\bm{X}},\, p_{\widehat{\bm{X}}}) \leq P
    \\[0.1cm]
    &
    \mathbb{E}\big[\epsilon_{c_0}(\bm{X}, \widehat{\bm{X}})\big] \leq C\,.
  \end{array}
\end{align}
The first constraint bounds below $D \geq 0$ the expected distortion between
$\bm{X}$ and $\widehat{\bm{X}}$, as measured by $\Delta\,:\, \mathbb{R}^n
\times \mathbb{R}^n \to \mathbb{R}_+$.\footnote{We assume $\Delta(\bm{x},\,
  \bm{y}) = 0$ if and only if $\bm{x} = \bm{y}$.} The second constraint
  enforces a minimal perception quality on $\widehat{\bm{X}}$ by  bounding
  below $P\geq0$ the distance between the probability distributions
  $p_{\bm{X}}$ of $\bm{X}$ and $p_{\widehat{\bm{X}}}$ of $\widehat{\bm{X}}$, as
  measured by $d\,:\, \mathcal{P}_{\bm{X}} \times \mathcal{P}_{\bm{X}} \to
  \mathbb{R}_+$.\footnote{$\mathcal{P}_{\bm{X}}$ is the set of probability
    measures on the measurable space where $\bm{X}$ is defined, e.g.,
  $\mathbb{R}^n$ and the $d$-Cartesian product of Borel $\sigma$-algebras. We assume
$d(p, q) = 0$ if and only if $p = q$.}   And the third constraint bounds below
$C\geq 0$ the expected classification error achieved by an arbitrary classifier $c_0$,
as measured by $\epsilon_{c_0}\,:\,\mathbb{R}^n \times \mathbb{R}^n \to
\mathbb{R}_+$. 
We implicitly assume $\bm{\Sigma}_{ii} > 0$ and $\bm{\Sigma}_{ij}=0$ for $i\neq
j$.
The objective of~\eqref{eq:problem-general} defines the channel rate assuming,
without loss of generality, that the encoder normalizes its output to have unit
power. The reason is to avoid spurious degrees of freedom when defining the
rate. In a more practical scenario, one may equivalently have an estimate of
the channel noise power and, instead, adjust the power of the output of the
encoder. Henceforth, whenever we mention \textit{rate}, we mean \textit{channel
rate} (not to be confused with source rate). We will
call~\eqref{eq:problem-general} the RDPC function.

\mypar{Problem statement}
Our goal is to characterize and solve the problem
in~\eqref{eq:problem-general}. Specifically, we aim to understand how the
different values of $D$, $P$, and $C$ affect the achievable rate $R(D, P, C)$.
We also aim to design an encoder $p_{\bm{Y}\vert\bm{X}}$, decoder
$p_{\widehat{\bm{X}}\vert\widehat{\bm{Y}}}$, and noise matrix $\bm{\Sigma}$
that solve~\eqref{eq:problem-general}.

\subsection{Our approach and contributions}
\label{subsec:approach-contributions}

As existing characterizations of tradeoffs between, for example, distortion and
perception~\cite{blau2018perception}, rate, distortion, and
perception~\cite{RDPtradeoff}, or classification error, distortion, and
perception~\cite{CDP}, we show the existence of a tradeoff between rate,
distortion, perception, and classification error. Our setup [cf.\
\eqref{eq:channeldiagram}] is more general than the ones
in~\cite{blau2018perception,RDPtradeoff,CDP}, as we consider vector signals
(not necessarily scalar) and their compression in terms of dimensionality. 
We also establish a strict tradeoff between all the above quantities, i.e.,
that the function $R(D, P, C)$ is \textit{strictly} convex in $D$, $P$, and
$C$.

It is difficult to solve~\eqref{eq:problem-general} in full generality. So,
under the assumption that the source signals are drawn from a Gaussian mixture
model (GMM) with two classes ($L=2$) and that the encoder and decoder are
linear maps, we design an
algorithm, RDPCO, that directly attempts to solve~\eqref{eq:problem-general}.
In addition, leveraging the capacity of generative adversarial networks (GANs) to
model probability distributions~\cite{goodfellow2014generative}, we also propose
to use inverse-domain GAN (ID-GAN)~\cite{IndomainGAN} to design an image
compression algorithm that achieves both extremely high compression rates and
good quality in terms of reconstruction, perception, and classification.
Compared to the original GAN~\cite{goodfellow2014generative}, 
ID-GAN~\cite{IndomainGAN} learns how to map not only a latent code to an image,
but also an image to a latent code. 
Despite several differences, experimental results show that algorithms RDPCO
and ID-GAN exhibit a similar behavior.
We summarize our contributions as follows:
\begin{itemize}
  \item 
    We show the existence of a strict tradeoff between rate, distortion,
    perception, and classification error in joint source coding and modulation (JSCM). 
  \item 
    We propose two algorithms to solve the JSCM tradeoff problem: 
    a simple algorithm (RDPCO) that directly solves the tradeoff problem but applies
    only under restrictive assumptions, and another based
    on inverse-domain GAN (ID-GAN)~\cite{IndomainGAN} which can
    transmit images under extreme compression rates, handling low-capacity
    channels and preserving semantic information, perception quality, and
    reconstruction fidelity. In particular, we port techniques
    from~\cite{IndomainGAN}, originally applied to image editing, to a JSCM scenario.

  \item We upper bound the optimal value in~\eqref{eq:problem-general} when the
    input signal is a GMM and the encoder and decoder are linear. To achieve
    this, we derive a new bound on the Wasserstein-$1$ distance between GMMs in
    terms of their parameters. See Lemma~\ref{lem:Wasserstein}.

  \item 
    Simulation results show that RDPCO and ID-GAN exhibit
    the same behavior and reveal further insights about the RDPC problem.
    In addition, the proposed ID-GAN algorithm
    achieves a better RDPC tradeoff than a traditional method with source coding and modulation
    designed separately (JPEG+LDPC+BPSK)
    and than AE+GAN~\cite{agustsson2019generative}, a recent deep algorithm
    (modified to a JSCM scenario). It also achieves much better
    perception and classification accuracy than D-JSCC~\cite{deep-jscc}, at the
    cost of a slight increase in distortion.
  
\end{itemize}

\subsection{Organization}

We overview related work in Section~\ref{sec:relatedwork} and characterize the
tradeoff problem~\eqref{eq:problem-general} in Section~\ref{sec:tradeoff}.
Section~\ref{sec:RDPCO} analyzes the RDPC tradeoff under GMM source signals and
linear encoders/decoders, and proposes an algorithm to achieve the
optimal tradeoff.
Section~\ref{sec:idgan} develops the ID-GAN algorithm. The 
performance of both methods is then assessed in Section~\ref{sec:experiments}, and 
Section~\ref{sec:conclusions} concludes the paper.

\section{Related Work}
\label{sec:relatedwork}

We now review prior work on JSCM and then describe existing analyses
of tradeoffs in image-based compression.

\subsection{Joint source coding and modulation (JSCM)}
\label{subsec:relatedwork-jscc}

JSCM schemes outperform classical source-channel separation methods. 
Prior work on JSCM methods, also called JSCC even when they output
complex-valued signals, can be divided into two categories according to the
type of channel: basic channel transmission, in which the channel is simple
like a Gaussian or Rayleigh channel, and advanced channel transmission, in
which more realistic models for the channels are adopted, and the 
emphasis is on optimizing the transmission aspect of the system.
  
\mypar{JSCM for basic channels}
Methods in this category typically focus on designing neural networks
that optimize the compression performance of the JSCC/JSCM system, while
neglecting aspects of transmission optimization, such as radio resource allocation.
For example, \cite{deep-jscc} proposed a deep joint
source-channel coding (D-JSCC) algorithm based on an autoencoder and showed that besides outputting
images with quality superior to separation-based schemes, the algorithm
exhibits graceful performance degradation in low SNR.
Techniques vary according to the domain of the
data, e.g., text, image, video, or multimodal data.
For example, the JSCM system designed in~\cite{Farsad18-DJSCC_Text} used a
recurrent neural network for transmitting text. Also focusing on text
transmission, \cite{Xie21-DJSCC_Semantic} proposed a semantic communication
system (DeepSC) based on a transformer and, to evaluate performance, also a
novel metric to measure sentence similarity. DeepSC was extended
in~\cite{Weng21-DJSCC_Speech} for speech transmission. 
JSCM has also been applied to the transmission of multimodal data.
For instance, \cite{Wan23-CooperativeMM} proposed a cooperative scheme to
transmit audio, video, and sensor data
from multiple end devices to a central server. And concentrating on text and
images, \cite{Zhang22-SemanticMM} designed a coarse-to-fine multitask semantic
model using an attention mechanism. 
The theory and algorithms we derive in this paper fall under this
category, as we consider Gaussian channels.

\mypar{JSCM for advanced channels}
Methods in this category adopt more realistic channel, like the erasure
channel~\cite{DeepJSCC-l}, feedback channel with channel state information
(CSI)~\cite{DeepJSCC-f,Jialong23-DJSCC_CSI}, and the waveform (OFDM, etc.) or
multi-user channels. They focus on optimizing transmission.
For example, \cite{deep-edge} designed retrieval-oriented image
compression schemes, \cite{DeepJSCC-f} used channel feedback to improve the
quality of transmission, and~\cite{DeepJSCC-l} considered an adaptive bandwidth
to transmit information progressively under an erasure channel. 
Furthermore, \cite{Jialong23-DJSCC_CSI} designed an end-to-end approach for
D-JSCC~\cite{deep-jscc} with channel state information (CSI) feedback. The main
idea was to apply
a non-linear transform network to compress both the data and the CSI.
Finally, \cite{Yang21-DJSCC_OFDM} designed a scheme for orthogonal frequency
division multiplexing (OFDM) transmission that directly maps the source images
onto complex-valued baseband samples. 

\subsection{GAN-based compression}
Most algorithms for image transmission are based on
autoencoders~\cite{rumelhart1986learning}, e.g.,
\cite{deep-jscc,deep-edge,DeepJSCC-l,DeepJSCC-f,Jialong21-DSCC_Attention}. 
Autoencoders, however, compress signals only up to moderate compression ratios.
Although they achieve high-quality reconstruction, this is at the cost of
communication efficiency. Extreme compression has been achieved instead by
using GANs~\cite{goodfellow2014generative}, which are
generative models that learn, without supervision, both a
low-dimensional representation of the data and its
distribution~\cite{karras2019style}. This gives them the potential to achieve
extreme compression without undermining image perception quality. For example,
\cite{agustsson2019generative} proposed an autoencoder-GAN (AE+GAN) image
compression system in which the encoder and decoder are trained simultaneously.
The resulting method can achieve extremely low bitrates. One of the algorithms we
propose, ID-GAN, requires less training (as encoder and decoder are
trained separately), but attains a performance similar to or better than
AE+GAN~\cite{agustsson2019generative}.

Also related to our work, \cite{Erdemir23-Generative_JSCC} proposed
two algorithms, inverse-JSCC and generative-JSCC, to reconstruct images passed
through a fixed channel with a high compression ratio. 
The inverse-JSCC algorithm
views image reconstruction as an inverse problem and uses a powerful GAN model,
StyleGAN-2~\cite{Karras20-StyleGAN2}, as a regularizer together with a
distortion loss that aligns with human perception,
LPIPS~\cite{Zhang18-TheUnreasonableEffectivenessOfDeepFeatures}.  It is thus an
unsupervised method. Generative-JSCC transforms inverse-JSCC into a supervised
method by learning the parameters of an encoder/decoder pair while keeping the
parameters of StyleGAN-2 fixed. This work differs from ours in several ways.
First, we consider not only distortion and perception metrics, but also
classification accuracy and channel rate. In particular,
the experiments in~\cite{Erdemir23-Generative_JSCC} do not consider
any classification task. We also characterize the tradeoff
between all these metrics. Second,  our metric for perception, the
Wasserstein-1 distance between the input and output distributions, differs from
the LPIPS metric. Third, we train both the encoder and the decoder
adversarially, while~\cite{Erdemir23-Generative_JSCC} uses a pre-trained GAN
for the decoder.
Finally, while training StyleGAN-2 in~\cite{Erdemir23-Generative_JSCC} (on a
database of faces) requires tremendous computational resources, training our
ID-GAN can be done with less resources. 

\subsection{Tradeoff analyses}

The study of tradeoffs in lossy compression can be traced back  to
rate-distortion theory~\cite{cover1999elements}, which characterizes the 
rate-distortion function 
\begin{align}
  \label{eq:RD-function}
  R (D) = 
  \begin{array}[t]{cl}
    \underset{p_{\widehat{\bm{X}}\vert \bm{X}}}{\min} 
    & 
    I(\bm{X} , \widehat{\bm{X}})
    \\
    \text{s.t.}
    & 
    \mathbb{E}\big[\Delta( \bm{X} , \widehat{\bm{X}})\big] \leq D\,,
  \end{array}
\end{align}
where $I(\bm{X},\widehat{\bm{X}})$ is the mutual information between 
$\bm{X}$ and its reconstruction $\widehat{\bm{X}}$. The $R(D)$ function has
a closed-form expression under some simple source distributions and distortion
metrics. Recent work has gone beyond using reconstruction metrics, e.g., the
mean squared error (MSE), to assess image quality, considering also perception
and semantic metrics. 

\mypar{The PD tradeoff}
For example, \cite{blau2018perception} studied the perception-distortion (PD)
tradeoff by replacing the objective in~\eqref{eq:RD-function} with a divergence
metric $d(p_{\bm{X}}, p_{\widehat{\bm{X}}})$ [cf.\ \eqref{eq:problem-general}].
Assuming that the input signal follows a Rademacher distribution, 
they proved the existence of a tradeoff between the best achievable divergence
and the allowable distortion $D$.

\mypar{The RDP tradeoff}
Building on~\cite{blau2018perception}, \cite{RDPtradeoff}
studied the rate-distortion-perception (RDP) tradeoff. The problem they
analyzed was a variation of~\eqref{eq:problem-general}, without the last
constraint (on classification error) and with $I(\bm{X},\widehat{\bm{X}})$ in
the objective, instead of the rate. Assuming a Bernoulli input, they showed that
in lossy image compression, the higher the perception quality of the output
images, the lower the achievable rate. Although insightful, the analysis
in~\cite{RDPtradeoff} is not applicable to our scenario, as it considers only
scalar signals, thus ignoring the possibility of compressing them, and also
skips the quantization step. 
The work in~\cite{yan2021perceptual} further improved on~\cite{RDPtradeoff} and
showed that, for a fixed bit rate, imposing a perfect perception constraint
doubles the lowest achievable MSE. It further proposed a training
framework to achieve the lowest MSE distortion under a perfect perception
constraint at a given bit rate.

\mypar{The CDP tradeoff}
The work in~\cite{CDP} analyzed instead the
classification-distortion-perception (CDP) tradeoff, i.e.,
a modification of problem~\eqref{eq:problem-general} in which
$\mathbb{E}[\epsilon_{c_0}(\bm{X}, \widehat{\bm{X}})]$ is minimized subject to the first two
constraints (the rate is ignored). Assuming an
input signal that is drawn from a
Gaussian mixture model with two classes, they showed the existence of
a tradeoff. 
Our setup is more general, as we do not require the input to be Gaussian nor to
be drawn from just two classes.

\mypar{Our approach}
In all the above work, the signals are assumed scalar, which is
not suitable to study compression in terms of dimensionality reduction. By
contrast, in~\eqref{eq:problem-general}, we consider vector signals and
minimize the channel rate subject to constraints on distortion, perception, and
classification error. Furthermore, we show the existence of a
strict tradeoff, rather than just a simple tradeoff (as
in~\cite{blau2018perception,RDPtradeoff,CDP}) between rate and all the
constraints of~\eqref{eq:problem-general}.

\section{The RDPC Tradeoff}
\label{sec:tradeoff}

We now establish the existence of an inherent tradeoff in solving
problem~\eqref{eq:problem-general}. Recall our multiclass signal model
in~\eqref{eq:SourceModel} and the channel model in~\eqref{eq:channeldiagram}.
Recall also that we assume a Gaussian channel
$p_{\widehat{\bm{Y}}\vert\bm{Y}}(\widehat{\bm{y}}\vert \bm{y}) =
\mathcal{N}\big(\bm{y}\,, \bm{\Sigma}\big)$, where $\bm{\Sigma}_{ii} > 0$ is
the noise power of channel $i$, and $\bm{\Sigma}_{ij} = 0$ for $i\neq j$.

We assume a deterministic classifier $c_0\,:\,\mathbb{R}^n \to \{1, \ldots,
L\}$ which, for $l=1,\ldots, L$, decides $c_0(\widehat{\bm{X}}) = l$
whenever $\widehat{\bm{X}}$ belongs to a fixed region $\mathcal{R}_l \subset
\mathbb{R}^n$. Assuming $\epsilon_{c_0}$ is the $0$-$1$ loss, the expected
classification error is then
\begin{align}
  \mathbb{E}\Big[\epsilon_{c_0}(\bm{X},\, \widehat{\bm{X}})\Big]
  &=
  \mathbb{P}\big(\text{class}(\bm{X}) \neq c_0\big(\widehat{\bm{X}}\big)\big)
  \notag
  \\
  &=
  \sum_{i < j}
  \mathbb{P}\Big(c_0\big(\widehat{\bm{X}}\big) = i\,\big\vert\, H_j\Big)
  \cdot
  p_j
  \notag
  \\
  &=
  \sum_{i < j}
  p_j\cdot
  \int_{\mathcal{R}_i}\dif\,p_{\widehat{\bm{X}}\vert H_j}\,,
  \label{eq:probclasserror}
\end{align}
where $p_j := \mathbb{P}(H_l)$ is the probability of $\bm{X}$
being drawn from class $j$.
Our main result is as follows.
\begin{Theorem}
  \label{thm:convexity}
  Let $\bm{X}$ be a multiclass model as in~\eqref{eq:SourceModel}. Consider
  the communication scheme in~\eqref{eq:channeldiagram} and the associated RDPC
  problem in~\eqref{eq:problem-general}. Assume 
  the classifier $c_0$ is deterministic and that the perception function
  $d(\cdot, \cdot)$ is convex in its second argument. Then, the function $R(D,
  P, C)$ is strictly convex, and it is non-increasing in each argument.
\end{Theorem}
\begin{proof}
  See Appendix~\ref{app:convexity}.
\end{proof}
Theorem~\ref{thm:convexity} is generic and applies to any distortion metric
$\Delta$, perception metric $d$, and classifier $c_0$. The main assumption is
that the perception metric $d(\cdot,\cdot)$ is convex in the second argument,
which holds for a variety of divergences, e.g., $f$-divergence (including total
variation, Kullback-Leibler, and Hellinger distance) and R\'enyi
divergence~\cite{Csiszar04-InformationTheoryAndStatistics,Erven14-RenyiKL}. The
same assumption was used in~\cite{blau2018perception,RDPtradeoff,CDP}. The
theorem says that if we optimize the channel for the smallest possible rate,
the encoding-decoding system cannot achieve arbitrarily small distortion,
perception error, and classification error. These metrics are in conflict and we
need to strike a tradeoff between them. This behavior will be observed in
practice when we design algorithms to (approximately) solve the RDPC problem.
Note that while prior work~\cite{blau2018perception,RDPtradeoff,CDP} shows the
existence of a tradeoff by proving that a certain function is convex in each
argument, we establish a \textit{strict} tradeoff by proving that $R(D,P,C)$ is
\textit{strictly} convex in each argument.

As solving the RDPC problem in~\eqref{eq:problem-general} in full generality
i.e., non-parametrically, is difficult, in the next two sections we propose two
algorithms that approximately solve that problem under different assumptions.
As we will see in the experiments in Section~\ref{sec:experiments}, both
algorithms exhibit the tradeoff behavior stipulated by
Theorem~\ref{thm:convexity}. 

\section{RDPC Tradeoff under GMM signals and linear encoder and decoder}
\label{sec:RDPCO}

To make problem~\eqref{eq:problem-general} more tractable, in this section we
assume that the source signal $\bm{X}$ in~\eqref{eq:SourceModel} is a Gaussian
mixture model (GMM) drawn from two classes and that the encoder and decoder are
linear. This will enable us to approximate~\eqref{eq:problem-general} with a
problem whose optimal cost function upper bounds the optimal cost
of~\eqref{eq:problem-general} (Section~\ref{subsec:rdpco-problem-formulation}).
We then develop an 
algorithm, RDPCO, to solve the resulting problem (Section~\ref{subsec:rdpco-algorithm}).
More formally, we make the following assumptions.
\begin{Assumption}
  \label{ass:linearprob}
  In~\eqref{eq:SourceModel}-\eqref{eq:problem-general}, we assume:
  \begin{enumerate}
    \item The source $\bm{X} \in\mathbb{R}^n$ is drawn from a two-class
      GMM:
      \begin{subequations}
        \label{eq:GMMModel}
        \begin{align}
          \bm{X} \vert H_0 &\sim \mathcal{N} (\bm{0}_n,\, \bm{I}_n)
          \label{eq:GMMModel-H0}
          \\
          \bm{X} \vert H_1 &\sim \mathcal{N} (\bm{c}_n,\, \bm{I}_n)\,,
          \label{eq:GMMModel-H1}
        \end{align}
      \end{subequations}
      where $\bm{0}_n$ is the all-zeros vector in $\mathbb{R}^n$,
      $\bm{I}_n$ the identity matrix, and $\bm{c}_n \in \mathbb{R}^n$ a fixed vector.
      That is, we set $L=2$ in~\eqref{eq:SourceModel} and assume $\bm{X}\vert H_l$ is
      Gaussian, $l = 1, 2$.

    \item The encoder $e\,:\, \mathbb{R}^n \to \mathbb{R}^m$ and decoder $d\,:\,
      \mathbb{R}^m \to \mathbb{R}^n$ are linear and deterministic, i.e., they are implemented by
      full-rank matrices $\bm{E} \in \mathbb{R}^{m \times n}$ and $\bm{D}\in \mathbb{R}^{n
      \times m}$. 

    \item We use the mean-squared error (MSE) as a metric for distortion, i.e.,
      $\Delta(\bm{X},\, \widehat{\bm{X}}) = \|\bm{X} - \widehat{\bm{X}}\|_2^2$,
      and the Wasserstein-1 
      distance\footnote{\label{ft:wassersteindef}The
        Wasserstein-$p$ distance between two probability measures $p_{\bm{X}},
        p_{\bm{Y}}$ in
        $\mathbb{R}^n$ is 
        $
        W_p(p_{\bm{X}}, p_{\bm{Y}}) 
        = 
        \big(
        \underset{\gamma \in \Pi(p_{\bm{X}}, p_{\bm{Y}})}{\inf}
        \mathbb{E}_{(\bm{X}, \bm{Y}) \sim \gamma}\big[\|\bm{X} -
        \bm{Y}\|_2^p\big]\big)^{1/p}
        $, where $\Pi\big(p_{\bm{X}}, p_{\bm{Y}}\big)$ is the set
        of all joint distributions with marginals $p_{\bm{X}}$ and
        $p_{\bm{Y}}$, and $1\leq p \leq +\infty$.} 
      $W_1(p_{\bm{X}},\, p_{\widehat{\bm{X}}})$
      as a metric for perception, where $p_{\bm{X}}$ and $p_{\widehat{\bm{X}}}$
      are the distributions of $\bm{X}$ and $\widehat{\bm{X}}$.

    \item The classifier $c_0$ is an optimal Bayes classifier. Specifically,
      given an observation $\widehat{\bm{x}}$ of $\widehat{\bm{X}}$, it decides
      $H_1$ if $\mathbb{P}(H_1 \vert \widehat{\bm{x}}) \geq \mathbb{P}(H_0 \vert
      \widehat{\bm{x}})$, and $H_0$ otherwise.
  \end{enumerate}
\end{Assumption}

Assumptions 1) and 2) imply that the reconstructed signal $\widehat{\bm{X}}$ is
also a GMM. To see this, first note that the output signal
is $\widehat{\bm{X}} = \bm{D}\big(\bm{E}\bm{X} + \bm{N}\big)$, where $\bm{N}
\sim \mathcal{N}(\bm{0}_m, \bm{\Sigma})$. Since the sum of two Gaussian random
variables is also Gaussian, we obtain
\begin{subequations}
  \label{eq:outputsignals}
  \begin{align}
    \label{eq:outputsignals1}
    \widehat{\bm{X}}\,\vert\, H_0 &\,\,\sim\,\, \mathcal{N}
    \Big (\bm{0}_n\,,\,\bm{D}\big (\bm{E}\bm{E}^\top+\bm{\Sigma}\big)\bm{D}^\top\Big),
    \\
    \widehat{\bm{X}}\,\vert\, H_1 &\,\,\sim\,\, \mathcal{N}
    \Big (\bm{D}\bm{E}\bm{c}_n\,,\,\bm{D}\big (\bm{E}\bm{E}^\top+\bm{\Sigma}\big)\bm{D}^\top\Big)\,.
    \label{eq:outputsignals2}
  \end{align}
\end{subequations}
Note that $\bm{D} \in \mathbb{R}^{n \times m}$ has more rows than columns
($n > m$), making the covariance matrix $\widehat{\bm{\Sigma}}:=
\bm{D}(\bm{E}\bm{E}^\top + \bm{\Sigma})\bm{D}^\top$ 
in~\eqref{eq:outputsignals} rank-deficient, and thus both distributions
in~\eqref{eq:outputsignals} degenerate. Henceforth,
$\widehat{\bm{\Sigma}}^{-1}$ will thus refer to the generalized inverse of
$\widehat{\bm{\Sigma}}$. Specifically, let $\widehat{\bm{\Sigma}} =
\bm{Q}\bm{\Lambda}\bm{Q}^\top$ be an eigenvalue decomposition of
$\widehat{\bm{\Sigma}}$, with $\bm{\Lambda} = \text{Diag}(\lambda_1, \ldots,
\lambda_n)$ being a diagonal matrix of eigenvalues. Define $\bm{\Lambda}^{-1}$
as the diagonal matrix with diagonal entries $1/\lambda_i$ if $\lambda_i > 0$,
and $0$ otherwise. Then, $\widehat{\bm{\Sigma}}^{-1} :=
\bm{Q}\bm{\Lambda}^{-1}\bm{Q}^\top$. Similarly, the generalized determinant of
$|\widehat{\bm{\Sigma}}|$ is the product of the positive entries of
$\bm{\Lambda}$.

\subsection{Problem formulation}
\label{subsec:rdpco-problem-formulation}

Under Assumption~\ref{ass:linearprob}, problem \eqref{eq:problem-general}
becomes
\begin{align}
  \label{eq:problem-linearcase}
  R (D,P,C) 
  = 
  \begin{array}[t]{cl}
    \underset{\bm{E},\bm{D},\bm{\Sigma}}
    {\min} 
    & 
    \sum_{i=1}^{m}\log \big(1+\frac{1}{\bm{\Sigma}_{ii}}\big)
    \\ 
    \text{s.t.} 
    & 
    \mathbb{E}\big[\big\|\bm{X} - \widehat{\bm{X}}\big\|_2^2\big] \leq D
    \\[0.1cm]
    & 
    W_1(p_{\bm{X}},\, p_{\widehat{\bm{X}}}) \leq P
    \\[0.1cm]
    &
    \mathbb{E}\Big[\epsilon_{c_0}(\bm{X}, \widehat{\bm{X}})\Big] \leq C\,,
  \end{array}
\end{align}
where we omitted the dependence of $\widehat{\bm{X}}$ on $\bm{E}$ and $\bm{D}$
for simplicity. Despite the simplifications made under
Assumption~\ref{ass:linearprob}, problem~\eqref{eq:problem-linearcase} is still
challenging, and we will solve instead an approximation by relaxing its last
two constraints. Before doing so, we analyze each constraint in detail.

\mypar{Distortion constraint}
To derive an expression for the first constraint
in~\eqref{eq:problem-linearcase}, we
first condition the expected values:
\begin{multline}
  \label{eq:distortion-step1}
  \mathbb{E} \Big[\big\|\widehat{\bm{X}} - \bm{X}\big\|_2^2 \Big]
  =
  \mathbb{E}\, \Big[\big\|\widehat{\bm{X}} - \bm{X}\big\|_2^2
  \, \big\vert\, H_0\Big]\cdot p_0
  \\
  +
  \mathbb{E}\, \Big[\big\|\widehat{\bm{X}} - \bm{X}\big\|_2^2
  \, \big\vert\, H_1\Big]\cdot p_1\,,
\end{multline}
where $p_l = \mathbb{P}(H_l)$, $l=0,1$.
Notice that for $l =0,1$,
\begin{multline}
  \label{eq:distortion-step2}
  \mathbb{E}\Big[\big\|\widehat{\bm{X}} - \bm{X}\big\|_2^2 \, \big\vert\, H_l\Big]
  =
  \mathbb{E}\, \Big[\big\|\widehat{\bm{X}}\big\|_2^2 \, \big\vert\, H_l\Big]
  \\
  -2\,\mathbb{E}\, \Big[\widehat{\bm{X}}^\top \bm{X} \, \big\vert\, H_l\Big]
  +
  \mathbb{E}\, \Big[\big\|\bm{X}\big\|_2^2 \, \big\vert\, H_l\Big]\,.
\end{multline}
Under hypothesis $H_0$, the last term is simply a constant: 
\begin{align*}
  \mathbb{E}\, \Big[\big\|\bm{X}\big\|_2^2 \, \big\vert\, H_0\Big]
  &=
  \mathbb{E}\, \Big[\text{tr}\big (\bm{X}\bm{X}^\top\big)\, \big\vert\, H_0\Big]
  =
  \text{tr}\Big (\mathbb{E}\big[\bm{X}\bm{X}^\top\,\vert\,H_0\big]\Big)
  \notag
  \\
  &=
  \text{tr} (\bm{I}_n)
  =
  n\,,
\end{align*}
where we used the linearity of the trace $\text{tr} (\cdot)$ in the second
equality, and~\eqref{eq:GMMModel-H0} in the third equality. Similarly, under
$H_1$,
\begin{align*}
  \mathbb{E}\, \Big[\big\|\bm{X}\big\|_2^2 \, \big\vert\, H_1\Big]
  &=
  \text{tr}\Big (\mathbb{E}\big[\bm{X}\bm{X}^\top\,\vert\,H_1\big]\Big)
  =
  \text{tr} (\bm{I}_n + \bm{c}_n \bm{c}_n^\top)
  \notag
  \\
  &=
  n + \big\|\bm{c}_n\big\|_2^2\,,
\end{align*}
due to
\eqref{eq:GMMModel-H1}. Similar
reasoning applies to the first term of~\eqref{eq:distortion-step2}:
\begin{align*}
  \mathbb{E}\, \Big[\big\|\widehat{\bm{X}}\big\|_2^2 \, \big\vert\, H_0\Big]
  &=
  \text{tr}\big (\widehat{\bm{\Sigma}}\big)
  \\
  \mathbb{E}\, \Big[\big\|\widehat{\bm{X}}\big\|_2^2 \, \big\vert\, H_1\Big]
  &=
  \text{tr}\big (\widehat{\bm{\Sigma}}\big)
  +
  \bm{c}_n^\top \bm{E}^\top \bm{D}^\top \bm{D} \bm{E} \bm{c}_n\,,
\end{align*}
where $\widehat{\bm{\Sigma}} := \bm{D}\big
(\bm{E}\bm{E}^\top+\bm{\Sigma}\big)\bm{D}^\top$.
Finally,  the second term of the right-hand side
of~\eqref{eq:distortion-step2} can be rewritten for $l=0, 1$ as
\begin{align}
  \mathbb{E}\, \Big[\widehat{\bm{X}}^\top \bm{X} \, \big\vert\, H_l\Big]
  &=
  \mathbb{E}\, \Big[\bm{X}^\top\bm{D}\big (\bm{E}\bm{X} + \bm{N}\big) \, \big\vert\, H_l\Big]
  \notag
  \\
  &=
  \mathbb{E}\big[\text{tr}\big (\bm{E}\bm{X}\bm{X}^\top\bm{D}\big)\, \big\vert\, H_l\big]
  +
  \mathbb{E}\Big[\bm{X}^\top\bm{D}\bm{N}\, \big\vert\, H_l\Big]
  \notag
  \\
  &=
  \text{tr}\Big
(\bm{E}\,\mathbb{E}\big[\bm{X}\bm{X}^\top\,\vert\,H_l\big]\bm{D}\Big)\,,
  \label{eq:distortion-step7}
\end{align}
where we used 
$\text{tr} (\bm{A}\bm{B}) = \text{tr} (\bm{B}\bm{A})$ (since the dimensions
allow) in the first equality and the independence between $\bm{X}$ and $\bm{N}$
in the last equality.
Plugging~\eqref{eq:distortion-step2}-\eqref{eq:distortion-step7}
into~\eqref{eq:distortion-step1},
\begin{multline}
  \mathbb{E}\, \Big[\big\|\widehat{\bm{X}} - \bm{X}\big\|_2^2 \Big]  
  =
  \bigg[
    \text{tr}\big (\widehat{\bm{\Sigma}}\big)
    -
    2\,\text{tr}\big (\bm{E}\bm{D}\big)
    +
    n
  \bigg] p_0
  +
  \bigg[
    \text{tr}\big (\widehat{\bm{\Sigma}}\big)
    \\
    +
    \bm{c}_n^\top \bm{E}^\top \bm{D}^\top \bm{D} \bm{E} \bm{c}_n
    -
    2\text{tr}\Big (\bm{E} (\bm{I}_n + \bm{c}_n \bm{c}_n^\top)\bm{D}\Big)
    +
    n
    +
    \|\bm{c}_n\|_2^2
  \bigg] p_1\,.
  \label{eq:finalexpressiondistortion}
\end{multline}

\mypar{Perception constraint}
We now consider the perception constraint in~\eqref{eq:problem-linearcase},
which upper bounds the Wasserstein-$1$ distance $W_1(p_{\bm{X}},\,
p_{\widehat{\bm{X}}})$ by $P$. Both $p_{\bm{X}}$ and
$p_{\widehat{\bm{X}}}$ are Gaussian mixture models for which,
to the best of our knowledge, there is no closed-form expression for their
Wasserstein-$p$ distance. 
There is, however, a closed-form
expression for the Wasserstein-$2$ distance between
Gaussian distributions. Specifically, 
let $\bm{X} \sim p_{\bm{X}} = \mathcal{N} (\bm{\mu_X},\,
\bm{\Sigma_X})$ and $\bm{Y} \sim p_{\bm{Y}} = \mathcal{N}
(\bm{\mu_Y},\, \bm{\Sigma_Y})$ be two Gaussian random vectors with means
$\bm{\mu_X},\, \bm{\mu_Y} \in \mathbb{R}^n$ and positive semidefinite
covariance matrices $\bm{\Sigma_X},\,\bm{\Sigma_Y}\succeq \bm{0}_{n\times n}$.
It can be shown that the squared Wasserstein-$2$ distance between them
is~\cite{Olkin82-distance,Dowson82-Frechet} 
\begin{align*}
  \|\bm{\mu_X} - \bm{\mu_Y}\|_2^2 + 
  \text{tr}
  \bigg (
    \bm{\Sigma_X} + \bm{\Sigma_Y} 
    -
    2\Big (\bm{\Sigma_Y}^{\frac{1}{2}}\bm{\Sigma_X}\bm{\Sigma_Y}^{\frac{1}{2}}\Big)^{\frac{1}{2}}
  \bigg)\,.
\end{align*}
In the case where $\bm{\Sigma_X}$ and $\bm{\Sigma_Y}$ commute, i.e.,
$\bm{\Sigma_X}\bm{\Sigma_Y}=\bm{\Sigma_Y}\bm{\Sigma_X}$, the expression
simplifies to
\begin{align}
  \label{eq:Wasserstein-simple}
  W_2^2 (p_{\bm{X}},\, p_{\bm{Y}}) 
  = 
  \big\|\bm{\mu_X} - \bm{\mu_Y}\big\|_2^2 + 
  \big\|\bm{\Sigma_X}^{\frac{1}{2}} - \bm{\Sigma_Y}^{\frac{1}{2}}\big\|_F^2\,,
\end{align}
where $\|\cdot\|_F$ is the Frobenius norm.

Our objective is thus to upper bound $W_1(p_{\bm{X}},\,
p_{\widehat{\bm{X}}})$ as a function of
$W_2(p_{\bm{X}},\, p_{\widehat{\bm{X}}}\, \vert\, H_0)$
and 
$W_2(p_{\bm{X}},\, p_{\widehat{\bm{X}}}\, \vert\, H_1)$,
which we define as in footnote~\ref{ft:wassersteindef} [or, in dual form, as
in~\eqref{eq:WassersteinDistance} below] with expected values conditioned on
$H_0$ or $H_1$. We have the following result.
\begin{Lemma}
  \label{lem:Wasserstein}
  Let $p_{\bm{X}}$ (resp.\ $p_{\widehat{\bm{X}}}$) be a
  GMM modeled as~\eqref{eq:GMMModel} [resp.\
  \eqref{eq:outputsignals}], in which the probability of hypothesis $H_0$ is
  $p_0$ and of hypothesis $H_1$ is $p_1 = 1 - p_0$. Then, 
  \begin{align}
    W_1(p_{\bm{X}},\, p_{\widehat{\bm{X}}})
    \leq
    \big\|
    \widehat{\bm{\Sigma}}^{\frac{1}{2}}
    - 
    \bm{I_n}
    \big\|_F
    +
    \big\|\bm{D}\bm{E}\bm{c_n}-\bm{c_n}\big\|_2\cdot p_1\,.
    \label{eq:boundWasserstein}
  \end{align}
\end{Lemma}
\begin{proof}
  See Appendix~\ref{app:wasserstein}.
\end{proof}

To enforce the second constraint in~\eqref{eq:problem-linearcase}, we will
thus impose the right-hand side of~\eqref{eq:boundWasserstein} to be
bounded by $P$.

\mypar{Classification constraint}
We now address the last constraint of~\eqref{eq:problem-linearcase}. As in
Assumption~\ref{ass:linearprob}.4), we assume a Bayes classifier, which achieves
a minimal probability of error. Such a probability, however, does not have a
closed-form expression, but is upper bounded by
the Bhattacharyya bound~\cite{Duda01-PatternClassification}. For a two-class
GMM $\bm{X} \sim p_0\,\mathcal{N}(\bm{\mu}_0,\, \bm{\Sigma}_0) +
p_1\, \mathcal{N}(\bm{\mu}_1,\, \bm{\Sigma}_1)$, the bound is
\begin{align}
  &\mathbb{P}\big(\text{error}^\star\big)
  \leq
  \sqrt{p_0\, p_1}
  \int_{\mathbb{R}^n}
  \sqrt{
    p_{\bm{X}\vert H_0} (\bm{x})
    \,
    p_{\bm{X}\vert H_1} (\bm{x})
  }\,\dif\bm{x}
  \notag
  \\
  &=
  \sqrt{p_0\, p_1}
  \,
  \exp\bigg[
  -\frac{1}{8}\big (\bm{\mu}_1-\bm{\mu}_0\big)^\top
  \bigg[\frac{\bm{\Sigma}_0+\bm{\Sigma}_1}{2}\bigg]^{-1} 
  \!\!\!\!
  (\bm{\mu}_1-\bm{\mu}_0)
  \notag
  \\
  &\qquad\qquad
  - 
  \frac{1}{2}\log
  \frac
  {\big|\big (\bm{\Sigma}_0+\bm{\Sigma}_1\big)/2\big|}
  {\sqrt{\big|\bm{\Sigma}_0\big|\big|\bm{\Sigma}_1\big|}}
  \bigg]\,,
  \label{eq:bhattacharyya-general}
\end{align}
where $|\cdot|$ is the determinant of a matrix, and $\text{error}^\star$ is the
classification error achieved by a Bayes classifier. We
apply~\eqref{eq:bhattacharyya-general} to $\bm{X}$ and $\widehat{\bm{X}}$,
whose models are in~\eqref{eq:GMMModel} and~\eqref{eq:outputsignals}. Thus,
$\bm{\mu}_0 = \bm{0}_n$, $\bm{\mu}_1 = \bm{D}\bm{E}\bm{c}_n$, and
$\bm{\Sigma}_0 = \bm{\Sigma}_1 = \bm{D} (\bm{E}\bm{E}^\top +
\bm{\Sigma})\bm{D}^\top$. Hence,
\begin{align}
  \mathbb{E}\Big[\epsilon_{c_0}(\bm{X},\, \widehat{\bm{X}})\Big]
  &=
  \mathbb{P}\big(\text{class}(\widehat{\bm{X}}) \neq c_0\big(\widehat{\bm{X}}\big)\big)
  \notag
  \\
  &\leq
  \sqrt{p_0 p_1}
  \exp\bigg [\!\! -\frac{1}{8}\bm{c}_n^\top\bm{E}^\top\bm{D}^\top
    \widehat{\bm{\Sigma}}^{-1}
    \bm{D}\bm{E}\bm{c}_n
  \bigg].
  \label{eq:bhattacharyya-our}
\end{align}
So, in~\eqref{eq:problem-linearcase}, rather than bounding
$\mathbb{E}\big[\epsilon_{c_0}(\bm{X},\, \widehat{\bm{X}})\big] \leq C$, we
impose instead that the right-hand side of~\eqref{eq:bhattacharyya-our} is
upper bounded by $C$, which is equivalent to 
\begin{align}
  \bm{c}_n^\top\bm{E}^\top\bm{D}^\top 
    \widehat{\bm{\Sigma}}^{-1}
  \bm{D}\bm{E}\bm{c}_n \geq -8
  \log\frac{C}{\sqrt{p_0p_1}}\,.
  \label{eq:classification-constraint-final}
\end{align}
This defines a nonconvex set over $\bm{E}$, $\bm{D}$, and
$\bm{\Sigma}$ (via $\widehat{\bm{\Sigma}}$).

\mypar{Bound on RDPC}
Instead of solving~\eqref{eq:problem-linearcase}, we will
aim to solve a problem that upper bounds its optimal value:
\begin{align}
  R(D, P, C)
  \leq
  \begin{array}[t]{cl}
    \underset{\bm{E},\,\bm{\Sigma},\, \bm{D}}
    {\min} 
    & 
    \sum_{i=1}^{m}\log \big(1+\frac{1}{\bm{\Sigma}_{ii}}\big)
    \\ 
    \text{s.t.} 
    & 
    \eqref{eq:finalexpressiondistortion} \leq D
    \\[0.1cm]
    & 
    \eqref{eq:boundWasserstein} \leq P
    \\[0.1cm]
    &
    \eqref{eq:classification-constraint-final}\,,
  \end{array}
  \label{eq:boundRDPC}
\end{align}
where~\eqref{eq:finalexpressiondistortion} and~\eqref{eq:boundWasserstein}
refer to the right-hand side of the respective equations. While the first
constraint is exact, the second and third constraints are more stringent
versions of the original constraints in~\eqref{eq:problem-linearcase}.
The resulting problem, however, is still nonconvex and will require
approximation techniques.

\subsection{RDPCO: Heuristic algorithm for RDPC optimization}
\label{subsec:rdpco-algorithm}

Solving~\eqref{eq:boundRDPC} is difficult, as it is nonconvex and has an
infinite number of solutions. Indeed, $\bm{E}$ and $\bm{D}$
appear in the constraints of~\eqref{eq:boundRDPC} always as the product
$\bm{D}\bm{E}$. Thus, if $(\bm{E}^\star, \bm{\Sigma}^\star, \bm{D}^\star)$ is a
solution of~\eqref{eq:boundRDPC} so is $(\bm{E}^\star\bm{M}, \bm{\Sigma}^\star,
\bm{D}^\star \bm{M}^{-1})$ for any invertible matrix $\bm{M}$. This means
there are too many degrees of freedom. 
We will leverage this to first design the output covariance matrix
$\widehat{\bm{\Sigma}}$, and then alternatively find the encoder-decoder pair
$(\bm{E}, \bm{D})$, via intuitive principles, and the rate matrix
$\bm{\Sigma}$, via a barrier-type method applied to~\eqref{eq:boundRDPC}. 
 
\mypar{Design of $\widehat{\bm{\Sigma}}$}
While the original signals in~\eqref{eq:GMMModel} have non-degenerate
distributions, the decoded signals in~\eqref{eq:outputsignals} have degenerate
distributions. Specifically, assuming that $\bm{E} \in \mathbb{R}^{m \times n}$ and
$\bm{D} \in \mathbb{R}^{n \times m}$ have full rank and that $\text{range}(\bm{E})
\cap \text{null}(\bm{D}) = \emptyset$, the output signals
in~\eqref{eq:outputsignals} live in an $m$-dimensional subspace. If the fixed
vector $\bm{c}_n$, which represents the distance between $\bm{X}\vert H_0$ and
$\bm{X}\vert H_1$, is orthogonal to that subspace (equivalently
$\bm{D}\bm{E}\bm{c}_n = \bm{0}_n$), then $\widehat{\bm{X}}\vert H_0$ and
$\widehat{\bm{X}}\vert H_1$ become indistinguishable. In this case,
classification is impossible and perception is also undermined
[note that the second term in~\eqref{eq:boundWasserstein} requires
$\|\bm{D}\bm{E}\bm{c}_n - \bm{c}_n\|_2$ to be small].

To avoid this, we first generate the (degenerate)
covariance matrix $\widehat{\bm{\Sigma}} := \bm{D}(\bm{E}\bm{E}^\top +
\bm{\Sigma})\bm{D}^\top$ by guaranteeing that the distance between $\bm{X}\vert
H_0$ and $\bm{X}\vert H_1$ is preserved after transmitting these signals
through the channel. We achieve this by guaranteeing that $\bm{c}_n$ is an
eigenvector of $\widehat{\bm{\Sigma}}$ associated to eigenvalue $1$, while the
remaining eigenvectors are associated to eigenvalues of smaller magnitude.
Specifically, we set $\widehat{\bm{\Sigma}} = \bm{Q}\bm{\Lambda}\bm{Q}^\top$,
where the first column of $\bm{Q}$ is $\bm{c}_n$ and the remaining ones
are the output of Gram-Schmidt orthogonalization. Also, $\bm{\Lambda} =
\text{Diag}(1, \lambda_2, \ldots, \lambda_{m}, 0, \ldots, 0)$, with $\lambda_i$
being drawn uniformly at random from $[0,1]$, for $i=2, \ldots, m$.

Once $\widehat{\bm{\Sigma}}$ is fixed, we alternate between computing the 
encoder-decoder pair $(\bm{E}, \bm{D})$ and the rate matrix $\bm{\Sigma}$.

\mypar{Finding $(\bm{E}, \bm{D})$}
With $\widehat{\bm{\Sigma}}$ fixed and assuming that, at iteration $k$,
$\bm{\Sigma} = \bm{\Sigma_{k-1}}$ is also fixed, we seek a factorization
$\widehat{\bm{\Sigma}} = \bm{D}\bm{E}\bm{E}^\top\bm{D}^\top +
\bm{D}\bm{\Sigma_{k-1}}\bm{D^\top}$. We do so via an intuitive process that
leads to a unique factorization. Specifically, we design $\bm{E}$ and
$\bm{D}$ such that $\bm{D}\bm{E}\bm{E}^\top\bm{D}^\top $ is as close to the
identity matrix as possible (to preserve signals passing through the channel),
while $\bm{D}\bm{\Sigma_{k-1}}\bm{D}^\top$ is as small as possible (to mitigate
the effects of noise). Also, we ensure the principal direction
$\bm{c}_n$ is preserved: $\bm{D}\bm{E}\bm{c}_n \simeq \bm{c_n}$. These requirements,
weighted equally, can be cast as an optimization problem:
\begin{align}
  \label{eq:finding-ed-1}
  \begin{array}[t]{cl}
    \underset{\bm{E},\bm{D}}{\min} 
    & 
    \frac{1}{2}\Big\|\bm{I}_n-\bm{D}\bm{E}\bm{E}^\top\bm{D}^\top\Big\|_F^2
    +
    \frac{1}{2}\Big\|\bm{D}\bm{\Sigma_{k-1}}\bm{D}^\top\Big\|^2_F
    \\
    &\qquad\qquad\qquad\qquad\qquad
    +
    \frac{1}{2}\big\|\bm{c_n}-\bm{D}\bm{E}\bm{c_n}\big\|^2_2
    \\[0.2cm]
    \text{s.t.} 
    & 
    \widehat{\bm{\Sigma}} = \bm{D}\bm{E}\bm{E}^\top\bm{D}^\top\ +
    \bm{D}\bm{\Sigma_{k-1}}\bm{D}^\top\,,
  \end{array}
\end{align}
which, eliminating the constraint, can be written as
\begin{multline}
  \label{eq:finding-ed-2}
    \underset{\bm{E},\bm{D}}{\min}
    \,\,\,
    \frac{1}{2}\Big\|\bm{I}_n-\widehat{\bm{\Sigma}} + \bm{D}\bm{\Sigma_{k-1}}\bm{D}^\top\Big\|_F^2
    +
    \frac{1}{2}\Big\|\bm{D}\bm{\Sigma_{k-1}}\bm{D}^\top\Big\|^2_F
    \\
    +
    \frac{1}{2}\big\|\bm{c_n}-\bm{D}\bm{E}\bm{c_n}\big\|^2_2\,.
\end{multline}
We apply gradient descent to~\eqref{eq:finding-ed-2}
in order to find $(\bm{E_k}, \bm{D_k})$.
It can be shown that the partial derivatives of the objective $g(\bm{E},
\bm{D})$ of~\eqref{eq:finding-ed-2}
are
\begin{subequations}
    \label{eq:rdpco-derivatives}
    \begin{align}
      \dpd{g(\bm{E}, \bm{D})}{\bm{E}}
  &=
  \bm{D}^\top (\bm{D}\bm{E}\bm{c_n}-\bm{c_n})\bm{c_n}^\top
  \label{eq:rdpco-derivativeE}
  \\
  \dpd{g(\bm{E}, \bm{D})}{\bm{D}}
  &=
  4\bm{D}\bm{\Sigma_k}\bm{D}^\top \bm{D}\bm{\Sigma_{k-1}}
  +
  2(\bm{I_n} - \widehat{\bm{\Sigma}})\bm{D}\bm{\Sigma_{k-1}}
  \notag
  \\
  &
  \qquad\qquad\qquad\,\,+
  (\bm{D}\bm{E}\bm{c_n} - \bm{c_n})\bm{c_n}^\top \bm{E}^\top\,.
  \label{eq:rdpco-derivativeD}
    \end{align}
\end{subequations}

\begin{figure*}[t]
  \normalsize
  \setcounter{MYtempeqncnt}{\value{equation}}
  \setcounter{equation}{23}  
  \begin{subequations}
    \label{subeq:derivativeofrate}
    \begin{align}
      \nabla_{\bm{\sigma}} h_p(\bm{\Sigma})
        &=
        \bigg[
          \text{diag}
          \bigg(
            \bm{D}^\top
            \Big(
              \bm{I_m}
              -2
              \big(\bm{D}\bm{E}\bm{E}^\top \bm{D}^\top +
              \bm{D}\bm{\Sigma}\bm{D}^\top\big)^{-1}
            \Big)
            \bm{D}
          \bigg)
        \bigg]
        \bigg/
        \Big(
          P_k - \big\|\bm{\widehat{\Sigma}}^{\frac{1}{2}} - \bm{I_n}\big\|_F^2
        \Big)
        \label{eq:derivativeofrate-p}
        \\
        \nabla_{\bm{\sigma}}
        h_c(\bm{\Sigma})
        &=
        \frac{
          \text{diag}
          \bigg(
            \bm{D}^\top
            \big(
              \bm{D}\bm{E}\bm{E}^\top \bm{D}^\top + \bm{D}\bm{\Sigma}\bm{D}^\top
            \big)^{-1}
            \bm{D}\bm{E}\bm{c}\bm{c}^\top \bm{E}^\top \bm{D}^\top
            \big(
              \bm{D}\bm{E}\bm{E}^\top \bm{D}^\top + \bm{D}\bm{\Sigma}\bm{D}^\top
            \big)^{-1}
            \bm{D}
        \bigg)}
        {\bm{c_n}^\top \bm{E_k}^\top \bm{D_k}^\top 
          \bm{\widehat{\Sigma}}^{-1}
          \bm{D_k}\bm{E_k}\bm{c_n}
        + 8 \log\frac{C}{\sqrt{p_0 p_1}}}
        \label{eq:derivativeofrate-c}
      \end{align}
    \end{subequations}
    \setcounter{equation}{\value{MYtempeqncnt}}
    \hrulefill
    \vspace*{4pt}
\end{figure*}

\mypar{Finding $\bm{\Sigma}$}
Once the encoder-decoder pair is fixed at $(\bm{E_k} , \bm{D_k})$, we find the
diagonal rate matrix $\bm{\Sigma} := \text{Diag}(\bm{\sigma}) :=
\text{Diag}(\sigma_1, \ldots, \sigma_m)$ by applying a barrier
method~\cite{Boyd04-ConvexOptimization} to~\eqref{eq:boundRDPC}, i.e., we
solve a sequence of problems in $t$, each of which is
\begin{align}
  \label{eq:prob-find-sigma}
  \underset{\bm{\Sigma} = \text{Diag}(\bm{\sigma})}{\min}\,\,\,
  t\, h_r(\bm{\sigma})
  - 
  \lambda_D
  h_D(\bm{\Sigma})
  -
  \lambda_P
  h_P(\bm{\Sigma})
  -
  \lambda_C
  h_C(\bm{\Sigma})\,,
\end{align}
where $\lambda_D, \lambda_P, \lambda_C \geq 0$ are regularization parameters,
and
\begin{subequations}
  \label{subeq:functionsfindingrate}
  \begin{align}
    h_r(\bm{\sigma})
    &=
    \sum_{i=1}^{m}
    \log\Big(1 + \frac{1}{\sigma_i}\Big)
    \label{eq:functionfindingrate-r}
    \\
    h_d(\bm{\Sigma})
    &=
    \log
    \Big[
      D_k - \text{tr}\big(\bm{D_k}\bm{\Sigma}\bm{D_k}^\top\big)
    \Big]
    \label{eq:functionfindingrate-d}
    \\
    h_p(\bm{\Sigma})
    &=
    \log
    \Big[
      P_k - \big\|\bm{\widehat{\Sigma}_k}^{\frac{1}{2}} - \bm{I_n}\big\|_F^2
    \Big]
    \label{eq:functionfindingrate-p}
    \\
    h_c(\bm{\Sigma})
    &=
    \log
    \Big[
      \bm{c_n}^\top \bm{E_k}^\top \bm{D_k}^\top 
      \bm{\widehat{\Sigma}_k}^{-1}
      \bm{D_k}\bm{E_k}\bm{c_n}
      + 8 \log\frac{C}{\sqrt{p_0 p_1}}
    \Big]\,.
    \label{eq:functionfindingrate-c}
  \end{align}
\end{subequations}
where $\widehat{\bm{\Sigma}}_{\bm{k}} = \bm{D_k}\bm{E_k}\bm{E_k}^\top\bm{D_k}^\top +
\bm{D_k}\bm{\Sigma} \bm{D_k}^\top$.
In~\eqref{eq:functionfindingrate-d}, $D_k$ absorbs all the terms independent
from $\bm{\Sigma}$ when we set $\mathbb{E}[\big\|\widehat{\bm{X}} -
\bm{X}\big\|_2^2] \leq D$ in~\eqref{eq:finalexpressiondistortion} [including
$D$]. 
To obtain~\eqref{eq:functionfindingrate-p}, note that
imposing the right-hand side of~\eqref{eq:boundWasserstein} to be smaller than
$P$ is equivalent to $\big\|\widehat{\bm{\Sigma}}_{\bm{k}}^{\frac{1}{2}} -
\bm{I_n}\big\|_F^2 \leq \big(P - \|\bm{D_k}\bm{E_k}\bm{c_n} - \bm{c_n}\|_2\cdot
p_1\big)^2 =: P_k$. 
And $h_p(\bm{\Sigma})$ depends on $\bm{\Sigma}$ via
$\widehat{\bm{\Sigma}}_{\bm{k}}$.
Finally,
\eqref{eq:functionfindingrate-c} is the direct application of the log-barrier
function to~\eqref{eq:classification-constraint-final}. 

To solve each instance of~\eqref{eq:prob-find-sigma}, we apply again gradient
descent. While the gradients of $h_r$ in~\eqref{eq:functionfindingrate-r} and
$h_d$ in~\eqref{eq:functionfindingrate-d} can be computed directly, namely 
$\dif h_r(\bm\sigma)/ \dif \sigma_i = -1/(\sigma_i^2 + \sigma_i)$, for $i =
1, \ldots, m$, and $\nabla_{\bm{\sigma}} h_d(\bm{\Sigma}) =
-\text{diag}(\bm{D_k}^\top \bm{D_k})/[D_k - \text{tr}(\bm{D_k} \bm{\Sigma},
\bm{D_k}^\top)]$, where $\text{diag}(\cdot)$ extracts the diagonal entries of a
matrix into a vector, computing the gradients of $h_p$
in~\eqref{eq:functionfindingrate-p} and $h_c$
in~\eqref{eq:functionfindingrate-c} is more laborious. Their expressions are
shown in~\eqref{subeq:derivativeofrate} where, for simplicity, we omitted the
iteration index.

\begin{algorithm}
\caption{RDPCO algorithm}
\label{alg:rdpco}
\begin{algorithmic}[1]
  \small
  \algrenewcommand\algorithmicrequire{\textbf{Input:}}
  \Require  
  mean $\bm{c_n} \in\mathbb{R}^n$; probabilities $p_0, p_1 = 1
  - p_0$; bounds on distortion
  $D$, perception $P$, and classification $C$; initial barrier parameter
  $t_0$ and update parameter $\mu$; max. \# of iterations $K$; 
  stopping criteria parameter $\epsilon$; parameters $\lambda_D,
  \lambda_P, \lambda_C$.

  \algrenewcommand\algorithmicrequire{\textbf{Initialization:}}
  \Require  $\bm{\Sigma_0} = \bm{I}_m$

  \Statex 
  \vspace{-0.2cm}

  \algrenewcommand\algorithmicrequire{\textbf{Generate $\widehat{\bm{\Sigma}}$}}

  \Require 
    \State Set $\widetilde{\bm{Q}} = \begin{bmatrix} \bm{c_n} &\!\!\! \bm{R}
    \end{bmatrix}$, where $\bm{R} \in \mathbb{R}^{n \times n-1}$ has i.i.d.\
    $\mathcal{N}(0,1)$ entries
    \label{subalg:rdpco-GS1}

    \State Apply Gram-Schmidt orthogonalization to $\widetilde{\bm{Q}}$ to
    obtain $\bm{Q}$

    \State Generate $\lambda_i \in [0, 1]$, $i=2, \ldots, m$ randomly and build $\bm{\Lambda} =
    \text{Diag}(1, \lambda_2, \ldots, \lambda_{m}, 0, \ldots, 0) \in
    \mathbb{R}^{n\times n}$

    \State Set $\widehat{\bm{\Sigma}} = \bm{Q}\bm{\Lambda} \bm{Q}^\top$
    \label{subalg:rdpco-GS4}

  \Statex
  \vspace{-0.2cm}

  \algrenewcommand\algorithmicrequire{\textbf{Find $\bm{E}, \bm{D}, \bm{\Sigma}$}}

  \Require

  \For{$k = 1, \ldots, K$}
  \label{subalg:rdpco-outerloop}
  \State Find $(\bm{E_k}, \bm{D_k})$ via gradient descent applied to~\eqref{eq:finding-ed-2}
  [cf.\ \eqref{eq:rdpco-derivatives}]\label{subalg:rdpco-ED}

  \State Set $t = t_0$

  \For{$r = 1, \ldots, \lceil m/100 \rceil$}
    \label{subalg:rdpco-loopbarrier-beg}
    \State Find $\bm{\Sigma_r}$ via gradient descent applied to~\eqref{eq:prob-find-sigma} 
    \label{subalg:rdpco-sigma}

    \State $t \leftarrow \mu t$
    \label{subalg:rdpco-t}
  \EndFor
  \label{subalg:rdpco-loopbarrier-end}
  \State Set $\bm{\Sigma_k} = \bm{\Sigma_r}$

  \If{$\|(\bm{E_k}, \bm{D_k}, \bm{\Sigma_k}) - (\bm{E_{k-1}}, \bm{D_{k-1}},
  \bm{\Sigma_{k-1}})\|_F \leq \epsilon$}
  \label{subalg:rdpco-stoppingglobal}
  \State Stop
  \EndIf
  \EndFor
  \label{subalg:rdpco-end}
\end{algorithmic}
\end{algorithm}

\mypar{RDPCO algorithm}
We summarize all the above steps in Algorithm~\ref{alg:rdpco}, which we name
RDPCO for RDPC Optimization. 
Steps~\ref{subalg:rdpco-GS1}-\ref{subalg:rdpco-GS4} describe the procedure to
generate the covariance matrix $\widehat{\bm{\Sigma}} =
\bm{D}(\bm{E}\bm{E}^\top + \bm{\Sigma})\bm{D}^\top$, whose factors are then
computed in steps~\ref{subalg:rdpco-outerloop}-\ref{subalg:rdpco-end}. The
barrier method in
steps~\ref{subalg:rdpco-loopbarrier-beg}-\ref{subalg:rdpco-loopbarrier-end}
stops whenever the duality gap is below $0.01$ or the number of iterations
reaches $m/100$, both parameters determined experimentally. The remaining
parameters that we used in our experiments are reported in
Section~\ref{subsec:experiments-rdpco}.

\section{Solving RDPC with Inverse-Domain GAN}
\label{sec:idgan}

RDPCO attempts to solve the RDPC problem~\eqref{eq:problem-general} under 
restrictive assumptions [see Assumption~\ref{ass:linearprob}]. In this section, leveraging the
modeling power of neural networks, in particular
generative adversarial networks (GANs)~\cite{goodfellow2014generative}, we
propose an algorithm that works under more general assumptions.
In our channel
diagram~\eqref{eq:channeldiagram}, we will thus model the encoder
$p_{\bm{Y}\vert\bm{X}}$ with a neural network $e(\cdot\,;\bm{\theta}_e)\, :\,
\mathbb{R}^n \to \mathbb{R}^m$ parameterized by $\bm{\theta}_e$, and the
decoder $p_{\widehat{\bm{X}}\vert\widehat{\bm{Y}}}$ as neural network
$d(\cdot\,;\bm{\theta}_d)\, :\, \mathbb{R}^m \to \mathbb{R}^n$ parameterized by
$\bm{\theta}_d$. These networks will be trained as in ID-GAN~\cite{IndomainGAN}
which, however, was proposed for a task different from JSCC/JSCM. Specifically,
given an (adversarially-trained) image generator, the goal
in~\cite{IndomainGAN} was to train an encoder to obtain a
semantically-meaningful latent code for image editing. We adopt this process of
training the generator first, and then the encoder. 

\subsection{Proposed scheme}

Fig.~\ref{fig:ID-GAN} shows our framework based on ID-GAN. As
in~\cite{IndomainGAN}, we first train an image generator/decoder $d(\cdot\,;\,
\bm{\theta}_d)$ (Fig.~\ref{fig:ID-GAN}, top) adversarially against
discriminator $f_1$, which learns to distinguish a real signal from a randomly
generated one, $d(\bm{Z}\,;\, \bm{\theta}_d)$, where $\bm{Z} \sim
\mathcal{N}(\bm{0}_m,\,\bm{I}_m)$ is a vector of i.i.d.\ standard
Gaussians. This is the conventional GAN
setup~\cite{goodfellow2014generative,WassersteinGAN}. As the discriminator is a
particular case of a classifier, outputting just a binary signal, it is also
known as a \textit{critic}. Once the decoder is trained, we fix it and train
the encoder $e(\cdot\,;\,\bm{\theta}_e)$ together with its own critic $f_2$,
which again learns to distinguish real signals from randomly generated from
ones (Fig.~\ref{fig:ID-GAN}, bottom).
Comparing~\eqref{eq:channeldiagram} and Fig.~\ref{fig:ID-GAN} (bottom),
we see that $p_{\bm{Y}\vert\bm{X}}$ is implemented by $e(\cdot\, ;\,
\bm{\theta_e})$, $p_{\widehat{\bm{X}}\vert\widehat{\bm{Y}}}$ is implemented by $d(\cdot\, ;\,
\bm{\theta_d^\star})$, and the Gaussian channel noise
$p_{\widehat{\bm{Y}}\vert\bm{Y}}(\widehat{\bm{y}}\vert \bm{y}) =
\mathcal{N}(\bm{y}, \bm{\Sigma})$ has (diagonal) covariance matrix $\bm{\Sigma}=
\sigma_t^2 \bm{I}_m$, where $\sigma_t$ is a parameter we learn (or fix) during
training. If we normalize the output $\bm{Y}$ of the encoder to have unit
power, then the signal-to-noise ratio (SNR) is determined by $\sigma_t$ as
$\text{SNR}_t = -10 \log_{10}\, \sigma_t^2$, and the channel rate is 
\begin{align}
  \label{eq:rate}
  R = m\log_{2}\Big(1 + \frac{1}{\sigma_t^2}\Big)\,.
\end{align}
Next, we explain the training processes of
  the decoder and encoder in more detail.

\begin{figure}[t]
  \centering			
  \psscalebox{1.0}{	
  \begin{pspicture}(0,-0.3)(8.8,5.7)


    \def\scalebl{2.0}           
    \def\colorlightness{60}     

    \newcommand{\blockd}[1]{
      \psscalebox{\scalebl}{
        \pspolygon*[linearc=0.1,linecolor=#1!\colorlightness!white](-0.35,-0.2)(-0.35,0.2)(0.35,0.35)(0.35,-0.35)
        \pspolygon[linearc=0.1,linewidth=0.2pt,linecolor=#1](-0.35,-0.2)(-0.35,0.2)(0.35,0.35)(0.35,-0.35)
      }
    }
    \newcommand{\blocke}[1]{
      \psscalebox{\scalebl}{
        \pspolygon*[linearc=0.1,linecolor=#1!\colorlightness!white](-0.35,-0.35)(-0.35,0.35)(0.35,0.25)(0.35,-0.25)
        \pspolygon[linearc=0.1,linewidth=0.2pt,linecolor=#1](-0.35,-0.35)(-0.35,0.35)(0.35,0.25)(0.35,-0.25)
      }
    }

    \psset{arrows=->,linewidth=0.8pt,arrowinset=0,arrowsize=4.0pt,
    style=RoundCorners,gratioWh=1.25,strokeopacity=0.8,linecolor=black!55,fillstyle=none,nodesep=0.58cm}


    \psframe*[linecolor=darkblue!10,framearc=0](-0.1,3.6)(8.8,5.7)
    \rput[lt](0,5.6){\textbf{\textcolor{darkblue!55}{\textsc{Training of decoder}}}}


    \rput(2.0,4.4){\rnode{gaussian}{$\scriptstyle\bm{Z}\sim \mathcal{N}(\bm{0_m},\, \bm{I_m})$}}

    \rput(4.5,4.4){\rnode{decoder}{\blockd{darkblue}}}
    \rput(4.5,4.4){\textcolor{white}{$d(\cdot\, ;\, \bm{\theta_d})$}}
    \rput[b](4.5,5.1){\scriptsize \textsf{\textcolor{darkblue}{decoder}}}

    \rput(6.3,4.4){\rnode{tdecoderout}{\includegraphics[scale=0.55]{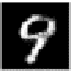}}}
    \rput(6.3,5.1){\rnode{gtimage}{\includegraphics[scale=0.55]{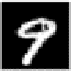}}}

    \rput(7.8,4.8){\rnode{f1}{\blocke{darkgreen}}}
    \rput(7.8,4.8){\textcolor{white}{$f_1(\!\cdot\, ;\, \bm{\theta_{f_1}}\!)$}}
    \rput[t](7.8,4.1){\scriptsize\textsf{\textcolor{darkgreen}{critic 1}}}

    \pnode(7.1,4.4){input1f1}
    \pnode(7.1,5.1){input2f1}
    \pnode(8.8,4.8){outputf1}

    \ncline[nodesepA=0cm]{gaussian}{decoder}
    \ncline[nodesepB=0cm]{decoder}{tdecoderout}
    \ncline[nodesep=0cm]{gtimage}{input2f1}
    \ncline[nodesep=0cm]{tdecoderout}{input1f1}
    \ncline[nodesepB=0cm]{f1}{outputf1}



    \psframe*[linecolor=red!10,framearc=0](-0.1,-0.3)(8.8,3.4)
    \rput[lt](0,3.3){\textbf{\textcolor{red!55}{\textsc{Training of encoder}}}}

    \rput[l](0.0,1.5){\rnode{inputencoder}{\includegraphics[scale=0.55]{figs/nine_gt.eps}}}

    \rput(2.0,1.5){\rnode{encoder}{\blocke{red}}}
    \rput(2.0,1.5){\textcolor{white}{$e(\cdot\, ;\, \bm{\theta_{e}})$}}
    \rput[t](2.0,0.8){\scriptsize \textsf{\textcolor{red}{encoder}}}

    \pscircleop[scale=0.75,linecolor=black!70,opsep=0.08](3.25,1.5){sum}
    \rput[l](2.87,0.1){\rnode{tnoise}{$\scriptstyle\bm{N}\sim\mathcal{N}$}$\scriptstyle(\bm{0_m},\, \sigma_t^2\bm{I}_m)$}

    \rput(4.5,1.5){\rnode{decoderfixed}{\blockd{darkblue}}}
    \rput(4.5,1.5){\textcolor{white}{$d(\cdot\, ;\, \bm{\theta_d^\star})$}}
    \rput[t](4.5,0.8){\scriptsize \textsf{\textcolor{darkblue}{decoder}}}

    \rput(6.3,1.5){\rnode{tencoderout}{\includegraphics[scale=0.55]{figs/nine_corr.eps}}}

    \rput(7.8,2.3){\rnode{f2}{\blocke{darkgreen}}}
    \rput(7.8,2.3){\textcolor{white}{$f_2(\!\cdot\, ;\, \bm{\theta_{f_2}}\!)$}}
    \rput[b](7.8,3.05){\scriptsize\textsf{\textcolor{darkgreen}{critic 2}}}

    \rput(7.8,0.7){\rnode{classifier}{\blocke{econ-ter-singapore55!85!black}}}
    \rput(7.8,0.7){\textcolor{white}{$c(\cdot)$}}
    \rput[t](7.8,-0.05){\scriptsize\textsf{\textcolor{econ-ter-singapore55!85!black}{classifier}}}

    \pnode(8.8,2.3){outputf2}
    \pnode(8.8,0.7){outputclassifier}

    \ncline[nodesepA=0cm]{inputencoder}{encoder}
    \ncline[nodesepB=0cm]{encoder}{sum}
    \ncline[nodesepA=0cm]{sum}{decoderfixed}
    \ncline[nodesepB=0cm,nodesepA=0cm]{tnoise}{sum}
    \ncline[nodesepB=0cm]{decoderfixed}{tencoderout}

    \ncangle[angleA=90,angleB=180,nodesepA=0cm,offsetB=-0.3cm]{tencoderout}{f2}
    \ncangle[angleA=270,angleB=180,nodesepA=0cm]{tencoderout}{classifier}
    \ncangle[angleA=90,angleB=180,nodesepA=0cm,offsetB=0.3cm]{inputencoder}{f2}

    \ncline[nodesepB=0cm]{f2}{outputf2}
    \ncline[nodesepB=0cm]{classifier}{outputclassifier}

    \ncline[linecolor=darkblue,linestyle=dotted,dotsep=0.5pt]{decoder}{decoderfixed}
    \Aput[0.1]{\small\textsf{\textit{\textcolor{darkblue}{fixed}}}}



    
\end{pspicture}
}		
\vspace{-0.5cm}
\caption{
  Proposed ID-GAN framework for solving the RDPC
  problem~\eqref{eq:problem-general}. The decoder is first trained
  adversarially with critic 1 in the first step (top). The decoder is then fixed and
  coupled with an encoder, which is in turn trained with critic 2 in order to
  preserve both reconstruction quality and classification accuracy (bottom). Critics 1
  and 2 have the same architecture.
}
\label{fig:ID-GAN}
\end{figure}

\subsection{Training the decoder}
\label{subsec:IDGAN-trainingdecoder}

To train the decoder $d(\cdot\, ;\, \bm{\theta}_d)$ and critic $f_1$
as in Fig.~\ref{fig:ID-GAN} (top), we use the Wasserstein GAN
(WGAN)~\cite{WassersteinGAN} framework. This consists of finding the parameters
$\bm{\theta}_d$ of the decoder that minimize the
Wasserstein-1 (or earth-mover) distance $W_1(p_r,
p_{\bm{\theta}_d})$ between the distribution $p_r$ of real
data and the distribution $p_{\bm{\theta}_d}$ of data generated by
$d(\bm{Z}\, ;\, \bm{\theta}_d)$, with $\bm{Z} \sim
\mathcal{N}(\bm{0}_m,\,\bm{I}_m)$. We consider the Kantarovich-Rubinstein dual form
of the Wasserstein-1 distance:
\begin{align}
  W_1(p_r,p_{\bm{\theta}_d})
  =
  \underset{\|f_1\|_L \leq 1}{\sup}
  \mathbb{E}_{\bm{X} \sim p_r}[f_1(\bm{X})]
  -
  \mathbb{E}_{\bm{X} \sim p_{\bm{\theta}_d}}[f_1(\bm{X})]\,,
  \label{eq:WassersteinDistance}
\end{align}
where the supremum is over the functions $f_1\,:\, \mathbb{R}^n \to \mathbb{R}$
that are $1$-Lipschitz continuous.  While the critic $f_1$
is found by maximizing the argument in~\eqref{eq:WassersteinDistance}, the
parameters of the decoder are found by minimizing the full Wasserstein distance
$W_1(p_r,p_{\bm{\theta}_d})$. This distance
enables overcoming the mode collapse observed in the original
GAN framework~\cite{goodfellow2014generative}, which used instead the Jensen-Shannon
divergence. To enforce Lipschitz-continuity of the critic $f_1$,
\cite{WassersteinGAN} proposed to limit its parameters to a small box around
the origin. The work in~\cite{WGAN-GP}, however, found that this technique
leads to instabilities in training (exploding/vanishing gradients) and
showed that a gradient penalty solves these problems. We thus adopt the loss
suggested in~\cite{WGAN-GP} for finding critic $f_1$:
\begin{multline}
  \label{eq:WGAN_LOSS}
    L_{f_1} 
    = 
    \mathbb{E}_{\bm{X}\sim p_{\bm{\theta}_d}}
    [f_1(\bm{X})]
    - 
    \mathbb{E}_{\bm{X}\sim p_{r}}
    [f_1(\bm{X})]
    \\
    + 
    \lambda_g\,
    \mathbb{E}_{\widetilde{\bm{X}}\sim p_{\bm{\theta}_d,r}}
    \Big[
    \big(\|\nabla_{\widetilde{\bm{x}}}f_1({\widetilde{\bm{X}}})\|_2 - 1\big)^2
    \Big]\,,
\end{multline}
where $\lambda_g \geq 0$, and $\widetilde{\bm{X}}$ is a point sampled uniformly
on the line joining a real data point $\bm{X} \sim p_r$ and a point
$\bm{Y} \sim p_{\bm{\theta}_d}$ generated by the decoder.
The third term in~\eqref{eq:WGAN_LOSS} eliminates the need to
constrain the critic to be Lipschitz-continuous [constraint
in~\eqref{eq:WassersteinDistance}]; see~\cite{WGAN-GP} for more details.
In turn, the parameters $\bm{\theta}_d$ of the decoder are found by
minimizing~\eqref{eq:WassersteinDistance} which, when $f_1$ is fixed, is
equivalent to minimizing
\begin{align}
  \label{eq:WGAN_LOSS_G}
    L_{\bm{\theta}_d} 
    = 
    -\mathbb{E}_{\bm{X}\sim p_{\bm{\theta_d}}}
    [
    f_1(\bm{X})
    ]
    =
    -\mathbb{E}_{\bm{Z}\sim \mathcal{N}(\bm{0}_m,\, \bm{I}_m)}
    \big[
    f_1\big(d(\bm{Z}\, ;\, \bm{\theta}_d)\big)
    \big].
\end{align}
During training, there are two nested loops: the outer loop updates
$\bm{\theta_d}$; and the inner loop, which runs for $n_{\text{critic}}$
iterations, updates the parameters of the critic such that 
the supremum in~\eqref{eq:WassersteinDistance} is reasonably well computed.
See~\cite{WassersteinGAN,WGAN-GP} for details.

\begin{algorithm}
  \caption{ID-GAN compression: training of the encoder}
  \label{alg:IDGANTrain}
  \begin{algorithmic}[1]
  \small
  \algrenewcommand\algorithmicrequire{\textbf{Input:}}
  \Require Training images/labels $\{(\bm{x}^{(t)},\bm{\ell}^{(t)})\}_{t=1}^T$, pretrained
  decoder $d(\cdot)$, pretrained classifier $c(\cdot)$, learning rate
  $\alpha$, momentum parameters $\beta_1$, $\beta_2$, batch size $S$, number of iterations of
  critic $n_{\text{critic}}$, and loss hyperparameters $\lambda_g$, $\lambda_d$, $\lambda_p$,
  $\lambda_c$
  \algrenewcommand\algorithmicrequire{\textbf{Initialization:}}
  \Require Set encoder $\bm{\theta}_e^{(1)}$ and critic 
  $\bm{\theta}_{f}^{(1)}$ parameters randomly;
  set channel noise standard deviation $\sigma_t^{(1)} > 0$ randomly 
   
  \Statex
  \algrenewcommand\algorithmicrequire{\textbf{In each epoch:}}
  \Require
  \State $\mathcal{S} = \text{randperm}(\{1,2, \ldots, T\})$
  \label{subalg:permute}
  \For{$j = 1, \ldots, \lceil T/S \rceil$}
  \State Select next $S$ batch indices $\mathcal{S}_j$ from $\mathcal{S}$
    \For{$k=1$ to $n_{\text{critic}}$}
    \label{subalg:criticfor}
    \State Generate (channel) noise $\bm{Z} \sim \mathcal{N}(\bm{0}_m, \bm{I}_m)$
      \State 
        $L_{f_2}^{(1)} 
          = 
          \frac{1}{S}\sum_{s \in \mathcal{S}_j} 
          f_2\Big(d\big(e(\bm{x}^{(s)}\, ;\, \bm{\theta}_e^{(j)}) +
            \sigma_t^{(j)}\bm{Z}\big)\,;\,
            \bm{\theta}_f^{(k)}\Big)$
      \State
        $
          L_{f_2}^{(2)}
          =
          \frac{1}{S}\sum_{s \in \mathcal{S}_j} 
          f_2\Big(\bm{x}^{(s)}\, ;\, \bm{\theta}_f^{(k)}\Big)
        $
      \For{$s \in \mathcal{S}_j$}
      \label{subalg:xtildefor}
        \State Draw $\epsilon \sim \mathcal{U}(0, 1)$ randomly
        \State 
        $\widetilde{\bm{x}}^{(s)} = (1-\epsilon)\cdot \bm{x}^{(s)} +
        \epsilon\cdot d\big(e\big(\bm{x}^{(s)}\, ;\, \bm{\theta}_e^{(j)}\big)
        + \sigma_t^{(j)}\bm{Z}\big)$
      \EndFor
      \label{subalg:xtildeend}

      \State 
        $
          L_{f_2}^{(3)}
          =
          \frac{1}{S}\sum_{s \in \mathcal{S}_j} 
          \Big(\big\|\nabla_{\widetilde{\bm{x}}} f_2(\widetilde{\bm{x}}^{(s)}) \big\|_2 -
          1\Big)^2
        $

      \State $L_{f_2} = L_{f_2}^{(1)} - L_{f_2}^{(2)} + \lambda_g L_{f_2}^{(3)}$

      \vspace{0.1cm}

      \State 
        $
        \bm{\theta}_{f}^{(k+1)} 
          =
          \text{Adam}\Big(\bm{\theta}_{f}^{(k)},\,\lambda_p L_{f_2},\,\alpha,\,\beta_1,\,\beta_2\Big)$
    \EndFor
    \label{subalg:criticend}
    \State $\bm{\theta}_f^{(j)} = \bm{\theta}_{f}^{(n_{\text{critic}})} $
    \For{$s \in \mathcal{S}_j$}
    \label{subalg:xhatfor}
      \State 
      $\widehat{\bm{x}}^{(s)} = d\Big(e\big(\bm{x}^{(s)}\, ;\,
      \bm{\theta}_e^{(j)}\big) + \sigma_t^{(j)}\bm{Z}\Big)$, w/
      $\bm{Z}\!\sim\!\mathcal{N}(\bm{0}_m, \bm{I}_m)$
    \EndFor
    \label{subalg:xhatend}

    \State 
    \vspace{-0.5cm}
    \begin{multline*}
      L_e 
      = 
      m\log_2\Big(1 + 1/{\sigma_t^{(j)}}^2\Big)
      +
      \frac{1}{S}\sum_{s \in \mathcal{S}_j}     
      \lambda_d \Big\|\bm{x}^{(s)} - \widehat{\bm{x}}^{(s)}\Big\|_2^2
      \\[-0.2cm]
      +
      \lambda_c \text{CE}\big(c(\widehat{\bm{x}}^{(s)}), \bm{\ell}^{(s)}\big)
      -
      \lambda_p f_2\Big(\widehat{\bm{x}}^{(s)}\,;\, \bm{\theta}_f^{(j)}\Big)
    \end{multline*}

    \State 
    \label{subalg:adamonencoderloss}
    $
    \Big(\bm{\theta}_{e}^{(j+1)}, \sigma_t^{(j + 1)}\Big) 
    =
    \text{Adam}\Big(\big(\bm{\theta}_{e}^{(j)},
    \sigma_t^{(j)}\big),\,L_{e},\,\alpha,\,\beta_1,\,\beta_2\Big)$
  \EndFor
  \end{algorithmic}
\end{algorithm}

\subsection{Training the encoder}
\label{subsec:IDGAN-trainingencoder}

After training the decoder, we fix its parameters to $\bm{\theta}_d^\star$ and
consider the scheme in Fig.~\ref{fig:ID-GAN} (bottom) to train the encoder
$e(\cdot\, ;\, \bm{\theta}_e)$. 
As the decoder, the encoder is also trained adversarially against a critic
$f_2$ to enhance perception quality, but also takes into account the
reconstruction quality and semantic meaning of the reconstruction. The former
is captured by an MSE loss between the original and reconstructed images, and
the latter by a cross-entropy loss between the image label and the output of a
pre-trained classifier $c(\cdot)$ applied to the reconstructed image. 

\mypar{Derivation of the loss for the encoder}
To motivate our loss for the encoder, we start from the RDPC
problem~\eqref{eq:problem-general}, considering $\Delta(\bm{x}, \bm{y}) =
\|\bm{x} - \bm{y}\|_2^2$ as the distortion metric, $d(p_{\bm{X}},
p_{\widehat{\bm{X}}}) = W_1(p_{\bm{X}}, p_{\widehat{\bm{X}}})$ as the
perception metric, and the cross-entropy $\epsilon_{c_0}(\bm{X},
\widehat{\bm{X}}) = \text{CE}(c(\widehat{\bm{X}}), \ell(\bm{X}))$ as the
classification loss, where $\ell(\bm{X})$ denotes the class of $\bm{X}$.
Then, there exist
constants $\lambda_d$, $\lambda_p$, and $\lambda_c$ (related to
$D$, $P$, and $C$),
such that~\eqref{eq:problem-general}  and
\begin{multline}
  \label{eq:connecting-idgan-rdpc}
    \underset{p_{\bm{Y}\vert\bm{X}},\,
    p_{\widehat{\bm{X}}\vert\widehat{\bm{Y}}},\,\bm{\Sigma}}
    {\min} 
    \,\,\,
    \sum_{i=1}^{m}\log \Big(1+\frac{1}{\bm{\Sigma}_{ii}}\Big) 
    +
    \mathbb{E}\Big[
    \lambda_d\big\|\bm{X}- \widehat{\bm{X}}\big\|_2^2
    \\
    +
    \lambda_c
    \text{CE}\big(c(\widehat{\bm{X}}),\, \ell\big(\bm{X}\big)\big)
    \Big] 
    +
    \lambda_p
    W_1\big(p_{\bm{X}},\, p_{\widehat{\bm{X}}}\big) 
\end{multline}
have the same solution. Problem~\eqref{eq:connecting-idgan-rdpc} is
non-parametric, i.e., the functions representing the encoder
$p_{\bm{Y}\vert\bm{X}}$ and the decoder $p_{\widehat{\bm{X}}\vert
\widehat{\bm{Y}}}$ have no structure. As mentioned, we assume they are
implemented by neural networks $e(\cdot\,;\bm{\theta_e})$ and
$d(\cdot\,;\bm{\theta}_d)$, respectively. The encoder, in particular,
normalizes its output signal to unit power. 
According to the channel model~\eqref{eq:channeldiagram}, the output of the
decoder is then $\widehat{\bm{X}} = d(e(\bm{X}) + \bm{N})$, where we omitted dependencies on
$\bm{\theta}_e$ and $\bm{\theta}_d$ for simplicity.
Under these assumptions and further assuming that the different channels have equal
variance [i.e., \eqref{eq:rate}], 
\eqref{eq:connecting-idgan-rdpc} becomes
\begin{multline}
  \label{eq:connecting-idgan-rdpc2}
  \underset{\bm{\theta}_e, \sigma_t}{\text{minimize}}
  \,\,\,
  m\log_2 \Big(1+\frac{1}{\sigma_t^2}\Big) 
  +
  \mathbb{E}\Big[\lambda_d \big\|\bm{X} - d(e(\bm{X}) + \sigma_t\bm{Z})\big\|_2^2
    \\
    + 
    \lambda_c 
    \text{CE}\Big(c\big(d(e(\bm{X}) + \sigma_t\bm{Z})\big),\, \ell\big(\bm{X}\big)\Big)
  \Big]
  +
  \lambda_p
  W_1\big(p_{\bm{X}},\, p_{\widehat{\bm{X}}}\big)\,,
\end{multline}
where, akin to the reparameterization
trick~\cite{Kingma14-AutoEncodingVariationalBayes}, we replaced $\bm{N}$ by
$\sigma_t \bm{Z}$, with $\bm{Z} \sim \mathcal{N}(\bm{0}_m, \bm{I}_m)$. This
makes the dependency of $\bm{N}$ on $\sigma_t$ explicit and enables computing
derivatives with respect to $\sigma_t$. Note that 
the expectation in~\eqref{eq:connecting-idgan-rdpc2} is with respect to $\bm{X}$ and $\bm{Z}$.
Note also that $\bm{\theta}_d$ is not included in the optimization variables
of~\eqref{eq:connecting-idgan-rdpc2}, as the decoder has already been trained.
Adopting the
approximations for $W_1$ described in
Section~\ref{subsec:IDGAN-trainingdecoder}, we find the parameters of the
encoder and channel noise level by solving
\begin{multline}
  \label{eq:loss-encoder}
  \underset{\bm{\theta}_e, \sigma_t}{\text{minimize}}
  \,\,\,
  m\log_2 \Big(1+\frac{1}{\sigma_t^2}\Big) 
  +
  \mathbb{E}\Big[\lambda_d \big\|\bm{X} - d(e(\bm{X}) + \sigma_t\bm{Z})\big\|_2^2
    \\
    + 
    \lambda_c 
    \text{CE}\Big(c\big(d(e(\bm{X}) + \sigma_t\bm{Z})\big),\, \ell\big(\bm{X}\big)\Big)
    \\
  -\lambda_p f_2(d(e(\bm{X})+\sigma_t\bm{Z}))
  \Big]
  \,.
\end{multline}
As remarked in
Section~\ref{subsec:rdpcfunction-probstatement}, we 
design the encoder and channel noise level $\sigma_t$ jointly. In
practice, one would instead estimate the channel noise level and design the
power of the encoder. The advantage of doing as we do is that the first
term of~\eqref{eq:loss-encoder} is
convex in $\sigma_t$.
The parameters $\bm{\theta}_{f_2}$ of critic $f_2$, in turn, are computed 
like in~\eqref{eq:WGAN_LOSS}, by minimizing the loss
\begin{multline}
  \label{eq:loss-encoder-critic}
    L_{f_2} 
    = 
    \lambda_p
    \bigg\{
    \mathbb{E}
    \Big[f_2\big(d(e(\bm{X})+\sigma_t\bm{Z})\big)
    - 
    f_2(\bm{X})
    \Big]
    \\
    + 
    \lambda_g\,
    \mathbb{E}
    \Big[
    \big(\big\|\nabla_{\widetilde{\bm{x}}}f_2\big(\widetilde{\bm{X}}\big)\big\|_2 - 1\big)^2
    \Big]\bigg\}\,,
\end{multline}
where $\widetilde{\bm{X}} = (1-\epsilon) \bm{X} + \epsilon\,
d(e(\bm{X})+\sigma_t\bm{Z})$ and
$\epsilon \sim \mathcal{U}(0,1)$ is uniformly distributed in $[0,1]$. The first
expectation is with respect to $\bm{X}$ and $\bm{Z}$, and the second with
respect to $\bm{X}$ and $\epsilon$.

\mypar{Training algorithm}
The complete training procedure of the encoder is shown in
Algorithm~\ref{alg:IDGANTrain}. Its inputs include
training images $\bm{x}^{(t)}$ and corresponding labels $\bm{\ell}^{(t)}$, and
a pretrained decoder $d(\cdot)$ and classifier $c(\cdot)$. After initializing
the parameters of the encoder, associated critic, and channel noise level, in
each epoch we randomly permute the indices of the training data
(step~\ref{subalg:permute}) and visit all the training data in batches of size
$S$. This takes $\lceil T/S \rceil$ iterations, where $T$ is the number of data
points. The loop in steps~\ref{subalg:criticfor}-\ref{subalg:criticend}
performs $n_{\text{critic}}$ iterations of Adam to minimize the critic loss
in~\eqref{eq:loss-encoder-critic} and thus to update the critic $f_2$
parameters $\bm{\theta}_f$ (where we omit the index $2$ for simplicity). This
corresponds to computing the supremum in the Wasserstein
distance~\eqref{eq:WassersteinDistance} between the real data $\bm{X}$ and the
reconstructed one $\widehat{\bm{X}} = d\big(e\big(\bm{X}\big) +
\bm{N}\big)$. The terms in~\eqref{eq:WGAN_LOSS} are computed
separately, with the last term requiring the creation of the intermediate
variables $\widetilde{\bm{x}}^{(s)}$ in
steps~\ref{subalg:xtildefor}-\ref{subalg:xtildeend}. As usual, expected values
were replaced by sample averages over the batch. After having updated the
parameters of critic $f_2$, we perform one iteration of Adam to minimize the
encoder loss in~\eqref{eq:loss-encoder} and thus to update the encoder
parameters $\bm{\theta}_e$ and channel noise level $\sigma_t$. This requires
passing each image in the batch through the encoder and decoder to create
$\widehat{\bm{x}}^{(t)}$, as in
steps~\ref{subalg:xhatfor}-\ref{subalg:xhatend}.
In step~\ref{subalg:adamonencoderloss}, the parameters of the encoder and
the channel noise level are updated simultaneously. In our
experiments, reported in Section~\ref{sec:experiments}, we run two versions of
Algorithm~\ref{alg:IDGANTrain}: one exactly as described, where the noise level
$\sigma_t$, and thus $\text{SNR}_t =-10\log_{10}\sigma_t^2$, is
learned during training; and another where $\text{SNR}_t$
is fixed to a predefined value.

\begin{figure*}
  \def\wdres{5.9cm}
  \def\hdres{5.25cm}
  \subfigure[]{
    \label{subfig:rdpco-varying-P-C-distortion}
    \begin{pspicture}(\wdres,\hdres)
      \rput[lb](0,0.35){\includegraphics[width=\wdres]{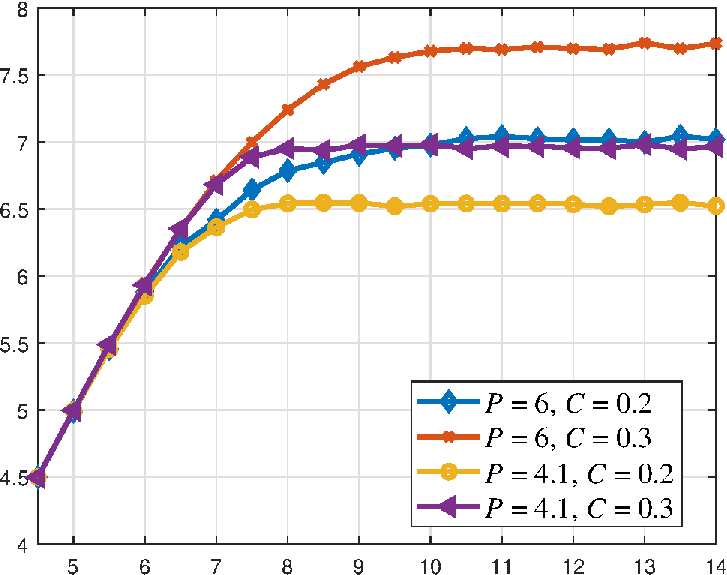}}
      \rput[lt](0,\hdres){\scriptsize\textbf{\sf Distortion}}
      \rput[b](2.9,0){\scriptsize\sf $D$}
    \end{pspicture}
  }
  \subfigure[]{
    \label{subfig:rdpco-varying-P-C-perception}
    \begin{pspicture}(\wdres,\hdres)
      \rput[lb](0,0.35){\includegraphics[width=\wdres]{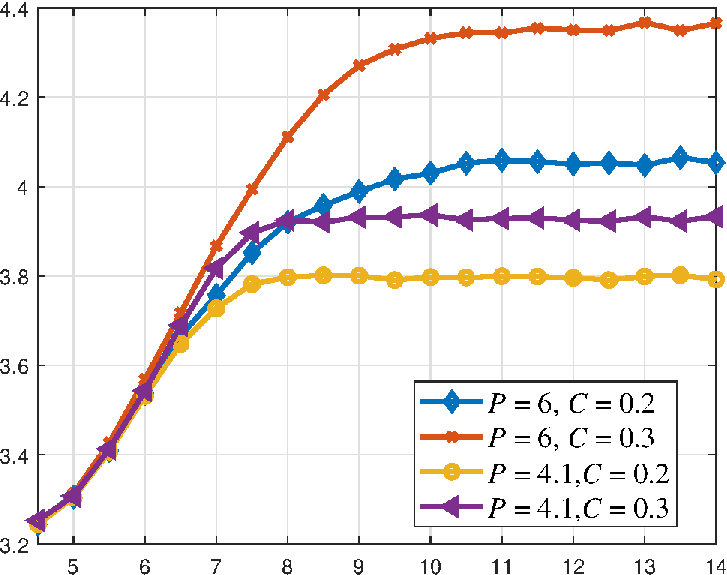}}
      \rput[lt](0,\hdres){\scriptsize\textbf{\sf Perception}}
      \rput[b](2.9,0){\scriptsize\sf $D$}
    \end{pspicture}
  }
  \subfigure[]{
    \label{subfig:rdpco-varying-P-C-classification}
    \begin{pspicture}(\wdres,\hdres)
      \rput[lb](0,0.35){\includegraphics[width=\wdres]{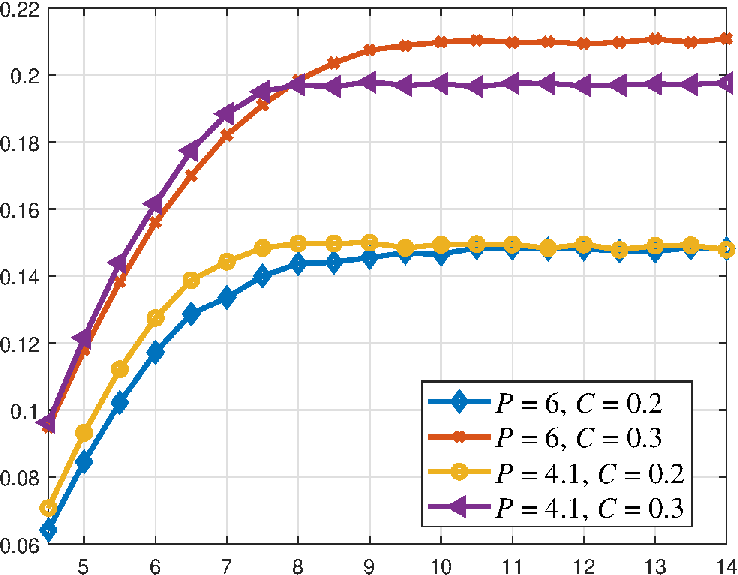}}
      \rput[lt](0,\hdres){\scriptsize\textbf{\sf Classification error}}
      \rput[b](2.9,0){\scriptsize\sf $D$}
    \end{pspicture}
  }
  \vspace{-0.4cm}
  \caption{
    Values of \text{(a)} distortion, \text{(b)} perception, and \text{(c)}
    classification error for RDPCO for varying distortion parameter $D$.
    These metrics are computed by the right-hand side of the expressions
    in~\eqref{eq:finalexpressiondistortion}, \eqref{eq:boundWasserstein},
    and~\eqref{eq:bhattacharyya-our}, respectively.
  }
  \label{fig:rdpco-varying-P-C}
\end{figure*}
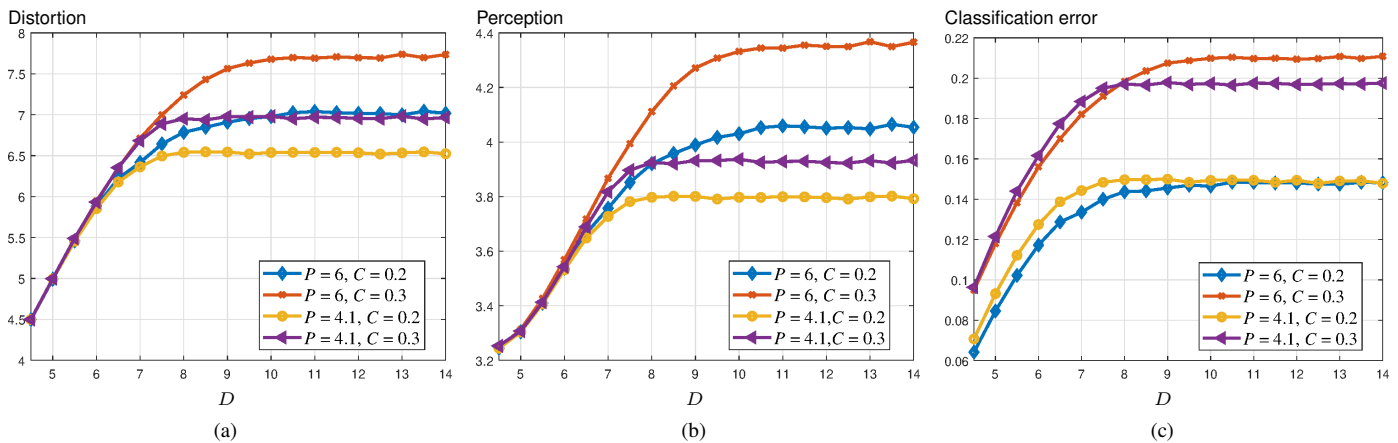

\begin{figure*}
  \def\wdres{5.9cm}
  \def\hdres{5.25cm}
  \subfigure[]{
    \label{subfig:rdpco-varying-m-distortion}
    \begin{pspicture}(\wdres,\hdres)
      \rput[lb](0,0.35){\includegraphics[width=\wdres]{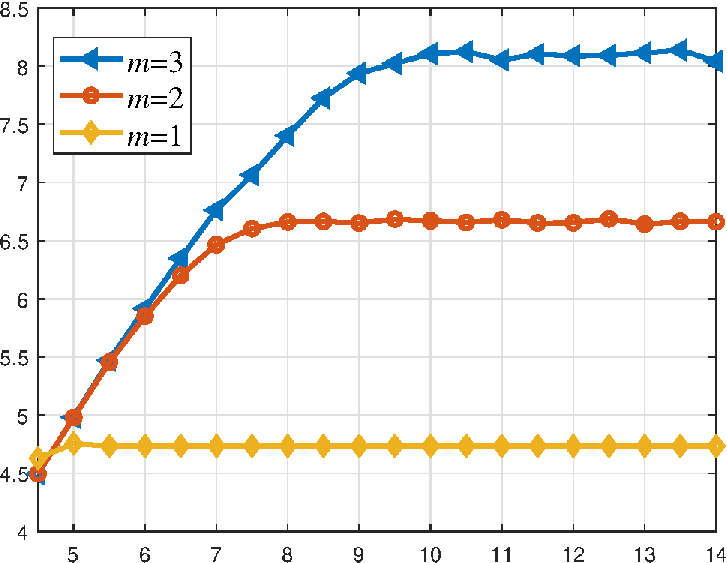}}
      \rput[lt](0,\hdres){\scriptsize\textbf{\sf Distortion}}
      \rput[b](2.9,0){\scriptsize\sf $D$}
    \end{pspicture}
  }
  \subfigure[]{
    \label{subfig:rdpco-varying-m-perception}
    \begin{pspicture}(\wdres,\hdres)
      \rput[lb](0,0.35){\includegraphics[width=\wdres]{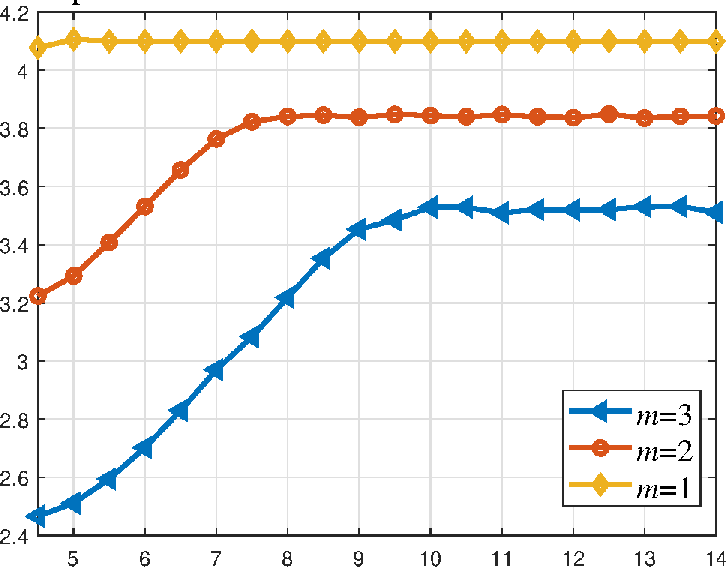}}
      \rput[lt](0,\hdres){\scriptsize\textbf{\sf Perception}}
      \rput[b](2.9,0){\scriptsize\sf $D$}
    \end{pspicture}
  }
  \subfigure[]{
    \label{subfig:rdpco-varying-m-classification}
    \begin{pspicture}(\wdres,\hdres)
      \rput[lb](0,0.35){\includegraphics[width=\wdres]{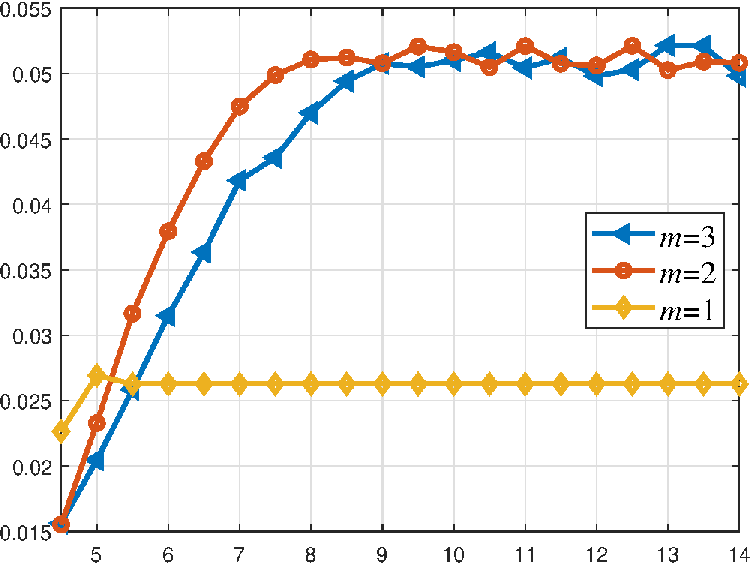}}
      \rput[lt](0,\hdres){\scriptsize\textbf{\sf Classification error}}
      \rput[b](2.9,0){\scriptsize\sf $D$}
    \end{pspicture}
  }
  \vspace{-0.4cm}
  \caption{
    Values of \text{(a)} distortion, \text{(b)} perception, and \text{(c)}
    classification error for RDPCO for varying latent dimension $m$
    and hyperparameters $(P, C) = (4.1, 0.1)$.
  }
  \label{fig:rdpco-varying-m}
\end{figure*}
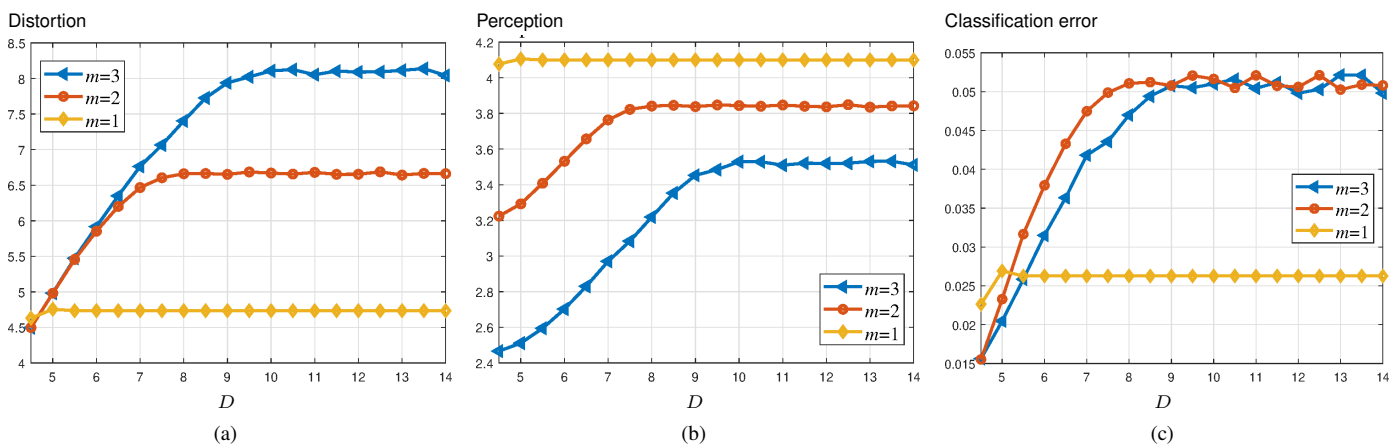

\section{Experimental Results}
\label{sec:experiments}

We now present our experiments to evaluate the performance of the proposed
algorithms, RDPCO (Algorithm~\ref{alg:rdpco}) and ID-GAN
(Algorithm~\ref{alg:IDGANTrain}). 
We start with RDPCO and then consider ID-GAN.

\subsection{RDPCO algorithm}
\label{subsec:experiments-rdpco}

Recall that RDPCO (Algorithm~\ref{alg:rdpco}) solves the
approximation~\eqref{eq:boundRDPC} of~\eqref{eq:problem-general}. Before
explaining the experiments, we describe how we set the parameters of the
algorithm. 

\mypar{Experimental setup}
In Algorithm~\ref{alg:rdpco}, we generate $\bm{c_n} \in
\mathbb{R}^n$ randomly with i.i.d.\ Gaussian entries with zero mean and
variance $4$. The classes are 
always equiprobable, i.e., $p_0 = p_1 = 1/2$. In the gradient
descent method in
step~\ref{subalg:rdpco-ED}, we employ a constant learning rate of $10^{-4}$ for
$20$k iterations. 
In the barrier method in
steps~\ref{subalg:rdpco-loopbarrier-beg}-\ref{subalg:rdpco-loopbarrier-end},
we initialize $t$ as $t_0 = 0.01$ and update it with a factor of $\mu = 2$. 
The parameter
$\epsilon$ in step~\ref{subalg:rdpco-stoppingglobal} is set to $10^{-5}$.
To balance the terms in~\eqref{eq:prob-find-sigma}, we set $\lambda_D =
1/\log(D)$,  $\lambda_P = 1/\log(P)$, and $\lambda_C = -1/\log(\sqrt{p_0\, p_1})$.
During the experiments, we vary $D$, $P$, $C$, and $R$ [which depends on the
latent dimension $m$ and noise level $\sigma_t$; cf.\ \eqref{eq:rate}].
To evaluate the performance of the algorithm, we visualize how two
metrics vary, e.g., rate and distortion, while the remaining parameters are
fixed.

\mypar{Metrics as a function of $\bm{D}$}
Here, we fix the input dimension to $n=5$ and the latent one to $m=2$.
Fig.~\ref{fig:rdpco-varying-P-C} shows how distortion, perception, and
classification error metrics vary with $D$ in~\eqref{eq:boundRDPC}.
These metrics are, respectively, the right-hand
side of~\eqref{eq:finalexpressiondistortion}, of
\eqref{eq:boundWasserstein}, and of~\eqref{eq:bhattacharyya-our}.
In
Figs.~\ref{subfig:rdpco-varying-P-C-distortion}-~\ref{subfig:rdpco-varying-P-C-perception},
we see that when $C$ (resp.\ $P$) is fixed, increasing $P$ (resp.\ $C$)
increases either the distortion or perception metrics. This behavior is
as expected according to Theorem~\ref{thm:convexity}. In
Fig.~\ref{subfig:rdpco-varying-P-C-classification}, we observe that for a fixed
$C$, modifying $P$ produces no significant effect on the classification error,
indicating that the classification constraint becomes active before the
perception one.

Fig.~\ref{fig:rdpco-varying-m} is similar to Fig.~\ref{fig:rdpco-varying-P-C},
but $P$ and $C$ are fixed to $4.1$ and $0.1$, respectively, while the latent
dimension $m$ varies. When $m=1$, all metrics are invariant to $D$. 
This is because, as seen in Fig.~\ref{subfig:rdpco-varying-m-perception}, the
perception constraint is active (it equals its limit of $4.1$), dominating the
two other constraints.
For $m=2,3$, their classification error behaves similarly, but
$m=3$ achieves better perception and worse distortion.

\begin{figure}
  \subfigure[]{
    \label{subfig:rate-distortion-RDPCO-lambdas}
    \def\hdres{7.1cm}
      \begin{pspicture}(\linewidth,\hdres)
        \rput[lb](-0.11,0.35){\includegraphics[width=8cm]{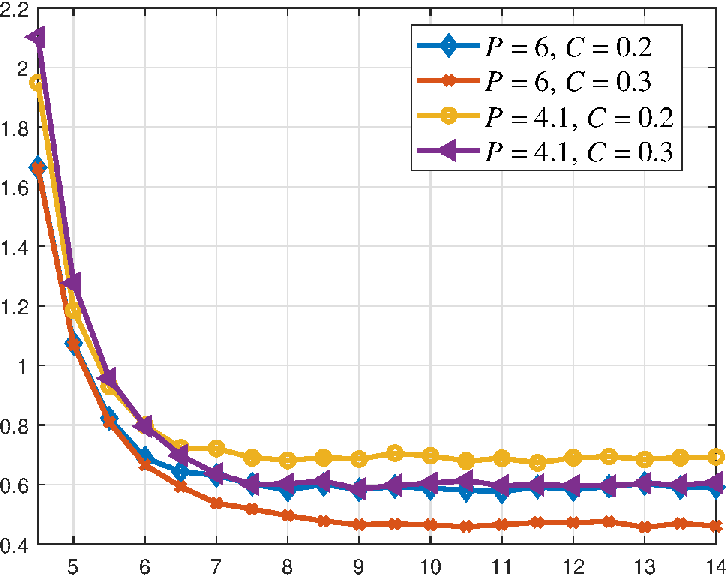}}
        \rput[lt](-0.11,\hdres){\small\textbf{\sf Rate}}
        \rput[b](4.3,0){\small\sf $D$}
      \end{pspicture}
  }
  \subfigure[]{
    \label{subfig:rate-distortion-RDPCO-m}
    \def\hdres{6.9cm}
      \begin{pspicture}(\linewidth,\hdres)
        \rput[lb](-0.11,0.35){\includegraphics[width=8cm]{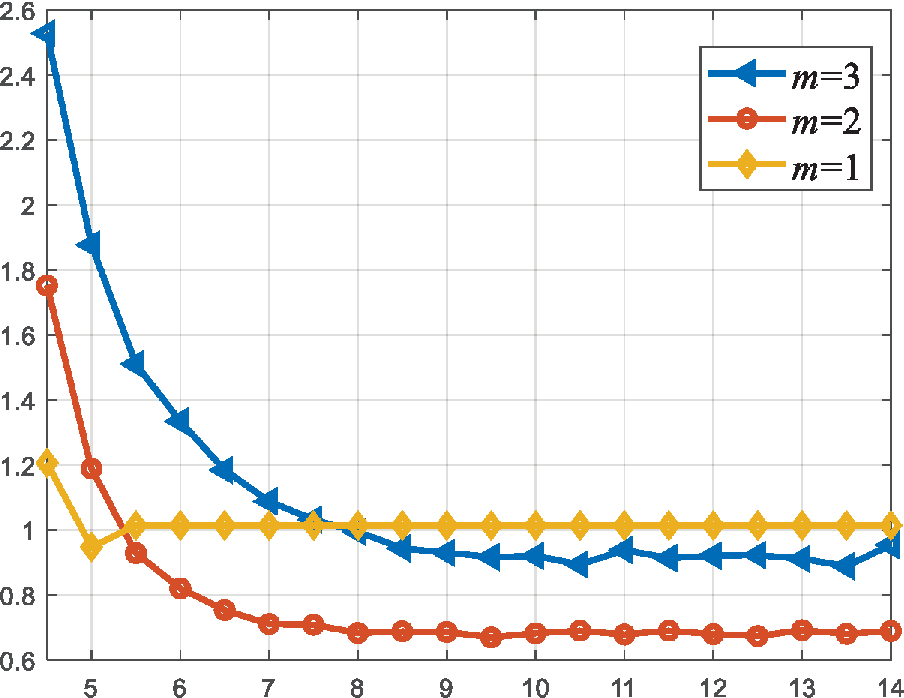}}
        \rput[lt](-0.11,\hdres){\small\textbf{\sf Rate}}
        \rput[b](4.3,0){\small\sf $D$}
      \end{pspicture}
  }
  \vspace{-0.2cm}
  \caption{
    Rate-distortion curves of RDPCO for
    \text{(a)}  varying $P$ and $C$, and \text{(b)} varying compressed
  dimension $m$ with hyperparameters $(P, C) = (4.1, 0.1)$.
  }
	\label{fig:rate-distortion-RDPCO}
\end{figure}
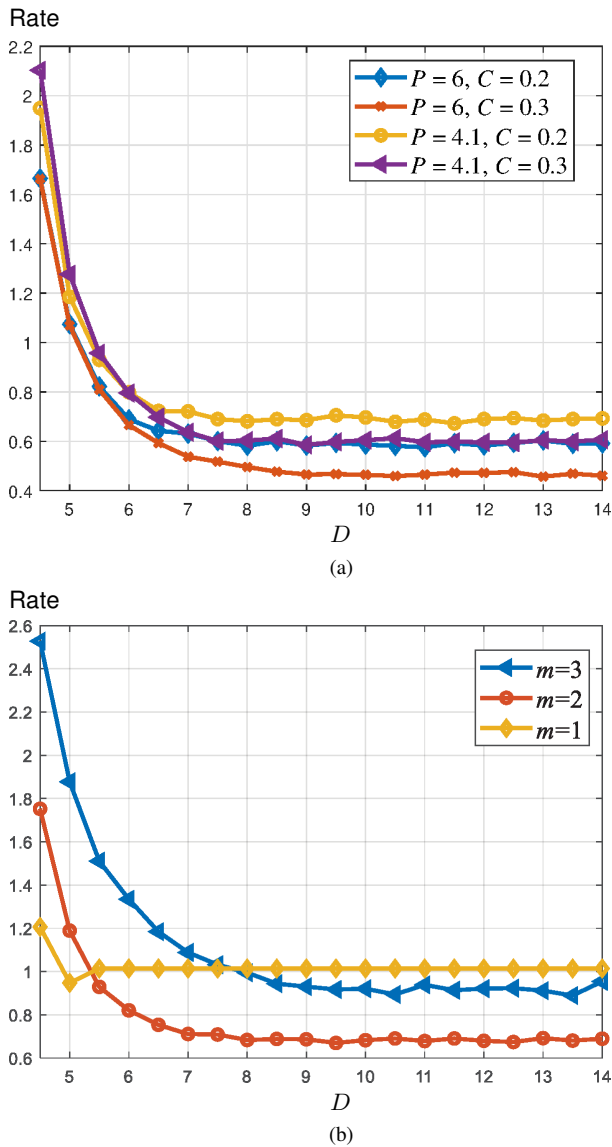

\mypar{Rate-distortion analysis}
In this experiment, as in Fig.~\ref{fig:rdpco-varying-P-C}, we vary both $P$
and $C$ in the constraints of~\eqref{eq:boundRDPC}.
Fig.~\ref{subfig:rate-distortion-RDPCO-lambdas} shows the resulting
rate-distortion curves. For a fixed $P$, decreasing $C$ increases the
rate; similarly, for a fixed $C$, decreasing $P$ increases the rate as well.
This validates the tradeoff established in Theorem~\ref{thm:convexity}.
Fig.~\ref{subfig:rate-distortion-RDPCO-m} shows the rate-distortion curves
under the same parameters as Fig.~\ref{fig:rdpco-varying-m}, i.e., $(P,
C)=(4.1, 0.1)$. We can see again that, for $m=1$, the rate is invariant to $D$,
since the perception constraint is the only active one.  For $m=2,3$, the
curves have the familiar tradeoff shape. In this case, $m=2$ yields a
rate-distortion curve better than $m=3$. The reason is that, as we saw in
Fig.~\ref{subfig:rdpco-varying-m-perception}, the perception constraint is the
most stringent of three constraints, and $m=2$ achieves a perception value
closer to the limit of $4.1$, leading to a better rate-distortion tradeoff in
Fig~\ref{subfig:rate-distortion-RDPCO-m}. 
These results corroborate the RDPC tradeoff we derived.
Indeed, they point to the existence of an optimal latent dimension $m$ that
minimizes the rate while satisfying the distortion, perception, and
classification constraints.

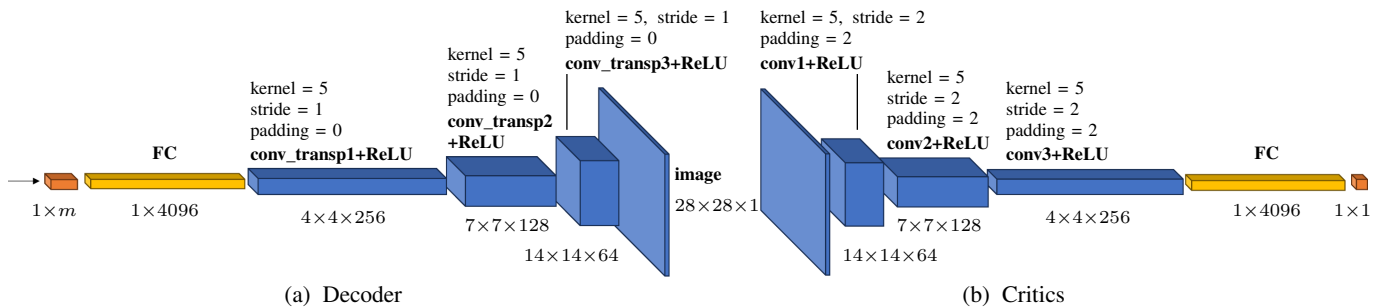
\begin{figure*}
  \centering
  \psscalebox{0.99}{	
    \begin{pspicture}(0.1,0)(18,4)

      \rput[l](0,1.7){\includegraphics[width=0.49\linewidth]{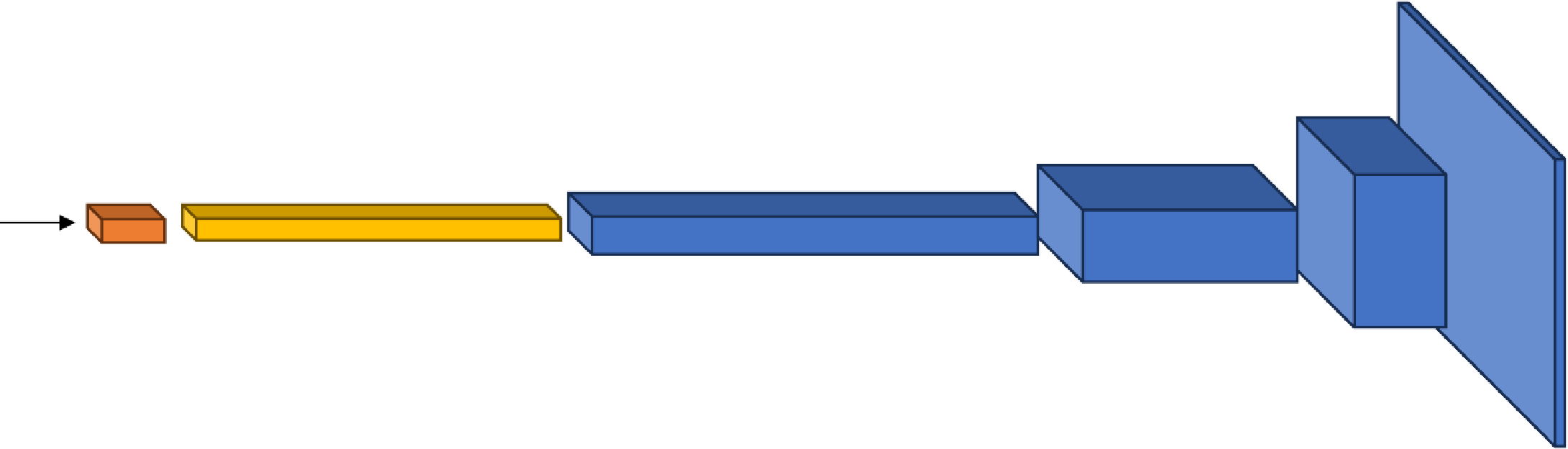}}
      \rput[l](10.1,1.7){\includegraphics[width=0.45\linewidth]{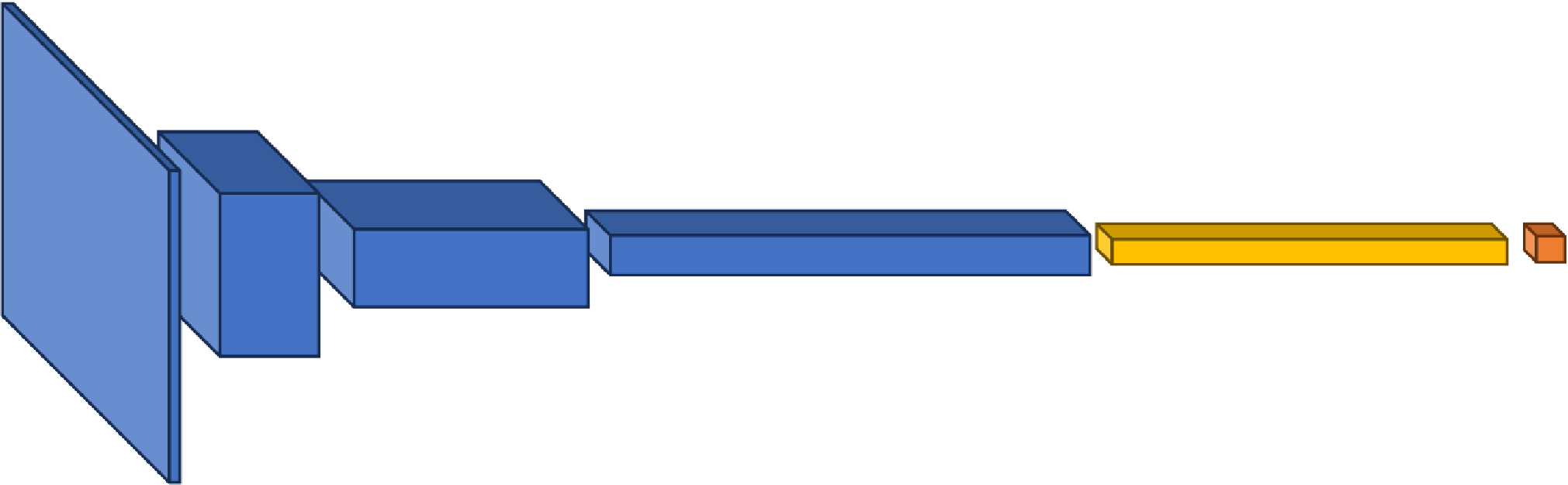}}

      \rput[tr](0.92,1.4){$\scriptstyle 1\times m$}

      \rput[t](2.1,1.4){$\scriptstyle 1\times 4096$}
      \rput[b](2.1,2.0){\scriptsize \textbf{FC}}

      \rput[t](4.5,1.3){$\scriptstyle 4\times 4 \times 256$}
      \rput[tl](3.25,3.05)
      {\parbox[t]{3cm}
        {\scriptsize kernel = 5 
          \\ 
          stride = 1 
          \\ 
          padding = 0 
          \\[0.05cm]
        \textbf{conv\_transp1+ReLU}}
      }

      \rput[t](6.7,1.2){$\scriptstyle 7\times 7 \times 128$}
      \rput[tl](5.9,3.5)
      {\parbox[t]{3cm}
        {\scriptsize kernel = 5 
          \\ 
          stride = 1 
          \\ 
          padding = 0 
          \\[0.05cm]
        \textbf{conv\_transp2\\+ReLU}}
      }

      \rput[tr](8.2,0.8){$\scriptstyle 14\times 14 \times 64$}
      \rput[tl](7.48,4.0)
      {\parbox[t]{3cm}
        {\scriptsize kernel = 5,\,
          stride = 1 
          \\ 
          padding = 0 
          \\[0.05cm]
        \textbf{conv\_transp3+ReLU}}
      }
      \psline[linewidth=0.5pt](7.5,3.15)(7.5,2.4)

      \rput[tl](8.95,1.9)
      {\parbox[t]{3cm}
        {{\scriptsize \textbf{image}}
          \\
          $\scriptstyle 28\times 28\times 1$
        }
      }


      \rput[tl](11.2,0.8){$\scriptstyle 14\times 14 \times 64$}
      \rput[tl](10.1,4.0)
      {\parbox[t]{3cm}
        {\scriptsize kernel = 5,\,
          stride = 2 
          \\ 
          padding = 2 
          \\[0.05cm]
        \textbf{conv1+ReLU}}
      }
      \psline[linewidth=0.5pt](11.4,3.15)(11.4,2.4)

      \rput[t](12.5,1.2){$\scriptstyle 7\times 7 \times 128$}
      \rput[tl](11.8,3.2)
      {\parbox[t]{3cm}
        {\scriptsize kernel = 5 
          \\ 
          stride = 2 
          \\ 
          padding = 2 
          \\[0.05cm]
        \textbf{conv2+ReLU}}
      }

      \rput[t](14.5,1.3){$\scriptstyle 4\times 4 \times 256$}
      \rput[tl](13.4,3.05)
      {\parbox[t]{3cm}
        {\scriptsize kernel = 5 
          \\ 
          stride = 2 
          \\ 
          padding = 2 
          \\[0.05cm]
        \textbf{conv3+ReLU}}
      }

      \rput[t](16.9,1.4){$\scriptstyle 1\times 4096$}
      \rput[b](16.9,2.0){\scriptsize \textbf{FC}}

      \rput[tr](18.3,1.4){$\scriptstyle 1\times 1$}


      \rput[t](4.5,0.3){\small \text{(a)}\, Decoder}
      \rput[t](13.5,0.3){\small \text{(b)}\, Critics}
    \end{pspicture}
  }		
  \vspace{-0.2cm}
  \caption
  {Architectures of the decoder $d(\cdot\, ;\, \bm{\theta}_d)$ and of the
    critics $f_1$ and $f_2$ in ID-GAN [cf.\ Fig.~\ref{fig:ID-GAN}]. \textbf{FC}
    stands for fully connected layer, \textbf{conv} for convolutional layer,
    and \textbf{conv\_transp} for transposed convolutional layer. We indicate
    the dimensions of
    the layer as well as the size of the kernels, stride, and padding.}
  \label{fig:arch-decoder-critics} 
\end{figure*}

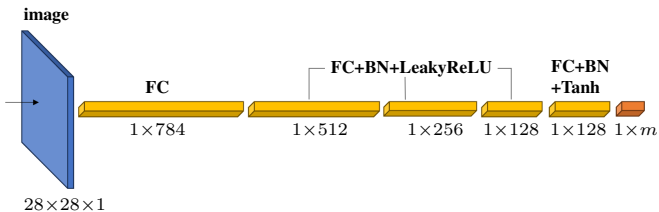
\begin{figure}
  \psscalebox{0.96}{	
    \begin{pspicture}(0,0.5)(9,3.1)

      \rput[l](0,1.7){\includegraphics[width=\linewidth]{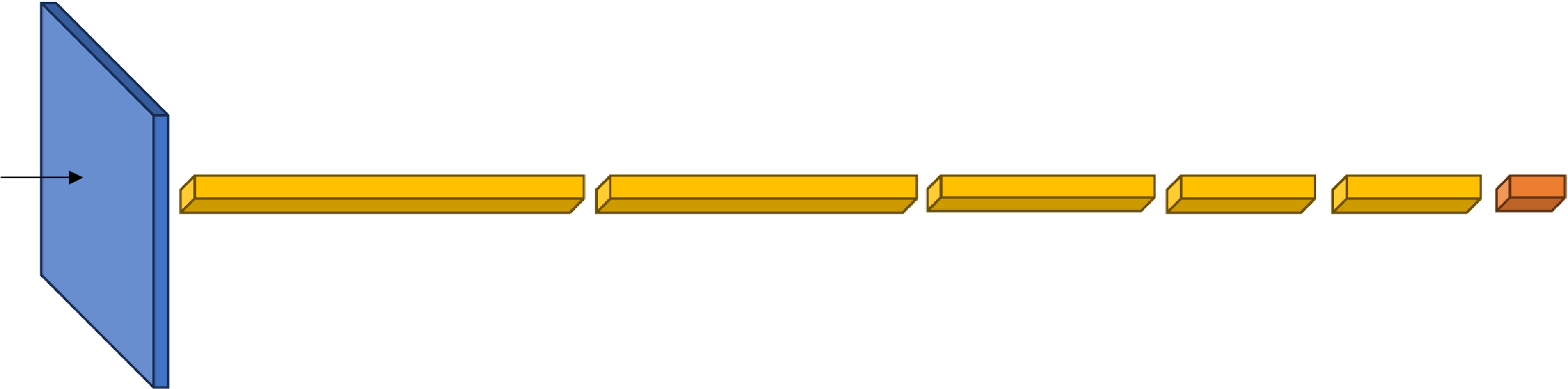}}


      \rput[tl](0.25,0.45){$\scriptstyle 28\times 28 \times 1$}
      \rput[tl](0.25,3.1){\scriptsize \textbf{image}}

      \rput[b](2.1,1.92){\scriptsize \textbf{FC}}
      \rput[t](2.1,1.5){$\scriptstyle 1\times 784$}

      \rput[tl](4.5,2.37){\rnode{LR}{\scriptsize \textbf{FC+BN+LeakyReLU}}}
      \pnode(4.4,1.65){L2}
      \pnode(5.5,1.65){L3}
      \pnode(7.0,1.65){L4}

      \psset{linewidth=0.4pt,linecolor=black!50}
      \ncangle[angleA=90,angleB=180,offsetA=0.2cm,nodesepA=0.18cm,nodesepB=0.08cm,armB=0.1cm]{L2}{LR}
      \ncline[offsetB=0.063cm,nodesepA=0.18cm,nodesepB=0.02cm]{L3}{LR}
      \ncangle[angleA=90,angleB=0,nodesepA=0.18cm,nodesepB=0.08cm,armB=0.12cm]{L4}{LR}

      \rput[t](4.35,1.5){$\scriptstyle 1\times 512$}

      \rput[t](5.95,1.5){$\scriptstyle 1\times 256$}

      \rput[t](7.0,1.5){$\scriptstyle 1\times 128$}

      \rput[b](8.55,1.92){\scriptsize \textbf{\parbox{2cm}{FC+BN\\+Tanh}}}
      \rput[t](7.93,1.5){$\scriptstyle 1\times 128$}

      \rput[tl](8.42,1.5){$\scriptstyle 1\times m$}


    \end{pspicture}
  }		
  \vspace{-0.3cm}
  \caption{
    Architecture of the encoder $e(\cdot\, ;\, \bm{\theta}_e)$, consisting of
    fully connected layers of indicated dimensions. Layers 2-4 contain a batch
    normalization (BN) layer and use LeakyReLU as activation. Layer 5 uses
    $\tanh$ instead.
  }
  \label{fig:arch-encoder}
\end{figure}

\begin{figure*}
  \def\wdres{5.8cm}
  \def\hdres{5.2cm}
  \subfigure[]
  {\label{subfig:rate-mse} 
      \begin{pspicture}(\wdres,\hdres)
        \rput[lb](0,0.35){\includegraphics[width=\wdres]{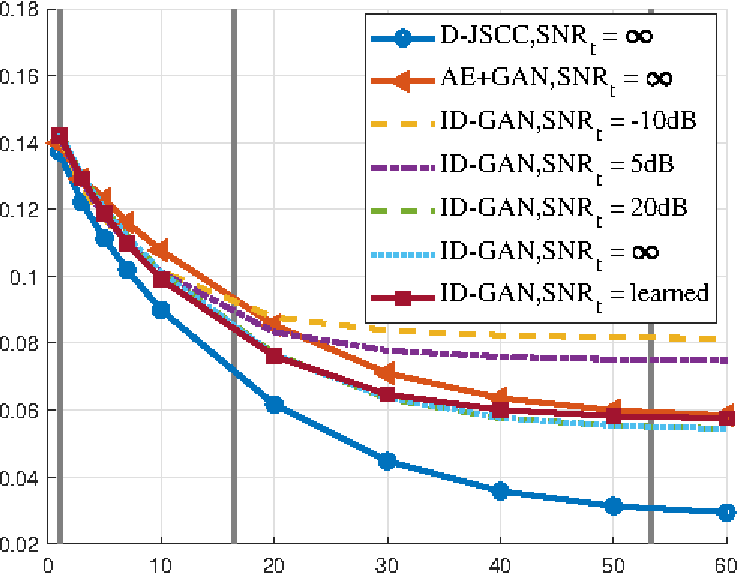}}
        \rput[lt](0,\hdres){\scriptsize\textbf{\sf Mean-squared error (MSE)}}
        \rput[b](2.9,0){\scriptsize\sf Rate}
        \rput[l](0.58,0.9){\scriptsize \textcolor{black!55}{\textbf{-10 \!dB}}}
        \rput[l](1.95,0.9){\scriptsize \textcolor{black!55}{\textbf{5 \!dB}}}
        \rput[r](5.04,1.2){\scriptsize \textcolor{black!55}{\textbf{20 \!dB}}}
      \end{pspicture}
  }
  \subfigure[]
  {\label{subfig:rate-fid}
      \begin{pspicture}(\wdres,\hdres)
        \rput[lb](0,0.35){\includegraphics[width=5.65cm]{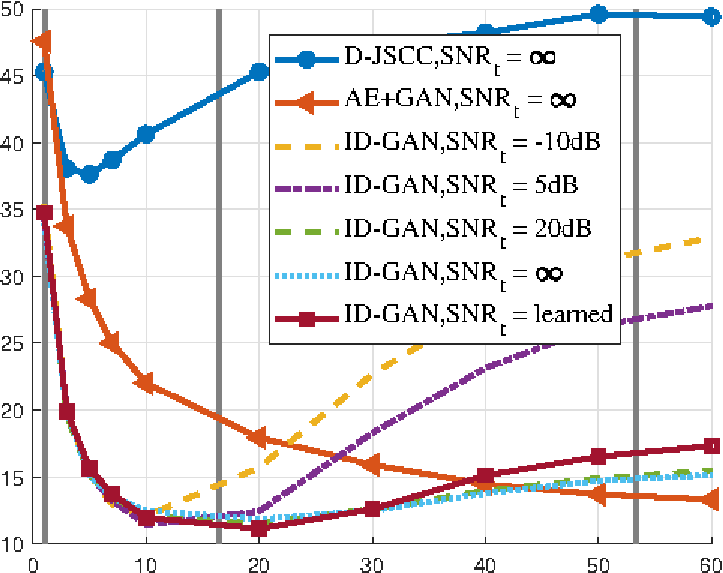}}
        \rput[lt](0,\hdres){\scriptsize\textbf{\sf Fréchet inception distance (FID)}}
        \rput[b](2.9,0){\scriptsize\sf Rate}
        \rput[l](0.4,0.7){\scriptsize \textcolor{black!55}{\textbf{-10 \!dB}}}
        \rput[l](1.81,1.85){\scriptsize \textcolor{black!55}{\textbf{5 \!dB}}}
        \rput[r](4.88,1.85){\scriptsize \textcolor{black!55}{\textbf{20 \!dB}}}
      \end{pspicture}
  }
  \subfigure[]
  {\label{subfig:rate-acc}
      \begin{pspicture}(\wdres,\hdres)
        \rput[lb](0,0.35){\includegraphics[width=5.65cm]{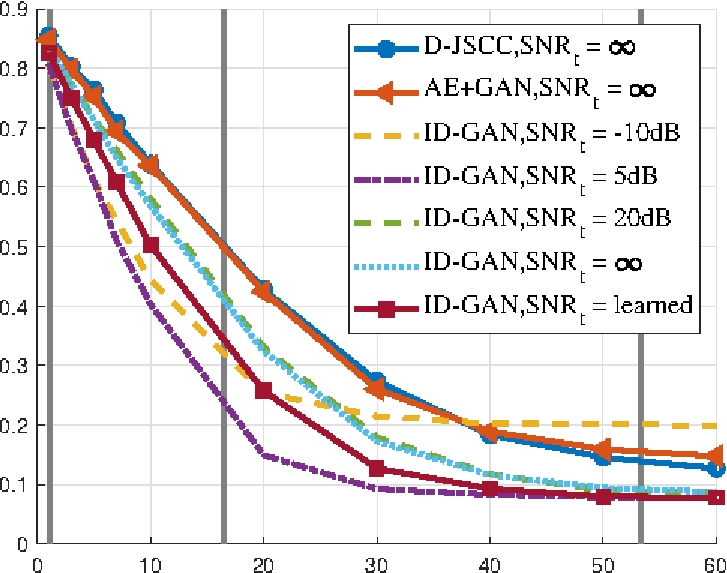}}
        \rput[lt](0,\hdres){\scriptsize\textbf{\sf Classification error}}
        \rput[b](2.9,0){\scriptsize\sf Rate}
        \rput[l](0.52,0.85){\scriptsize \textcolor{black!55}{\textbf{-10 dB}}}
        \rput[l](1.85,0.85){\scriptsize \textcolor{black!55}{\textbf{5 dB}}}
        \rput[r](4.89,1.76){\scriptsize \textcolor{black!55}{\textbf{20 \!dB}}}
      \end{pspicture}
  }
  \vspace{-0.4cm}
  \caption{
    Different metrics versus rate~\eqref{eq:rate} for the proposed ID-GAN, 
    D-JSCC~\cite{deep-jscc}, and AE+GAN~\cite{agustsson2019generative}.
    In ID-GAN, we set $(\lambda_g, \lambda_d, \lambda_p,
    \lambda_c) = (1, 10^3, 1, 10^3)$ and, during training, we either learned $\text{SNR}_t :=
    -10\log_{10}\sigma_t^2$ or fixed it to
    $-10$ dB, $5$ dB, $20$ dB (vertical lines depicting the corresponding
    rates), or $\infty$ (no noise). For all the metrics, the lower the better. \text{(a)} mean-squared error (MSE), \text{(b)} Fr\'echet inception distance (FID), and \text{(c)} classification error.
  }
  \label{fig:comparison-mse-fid-acc}
\end{figure*}
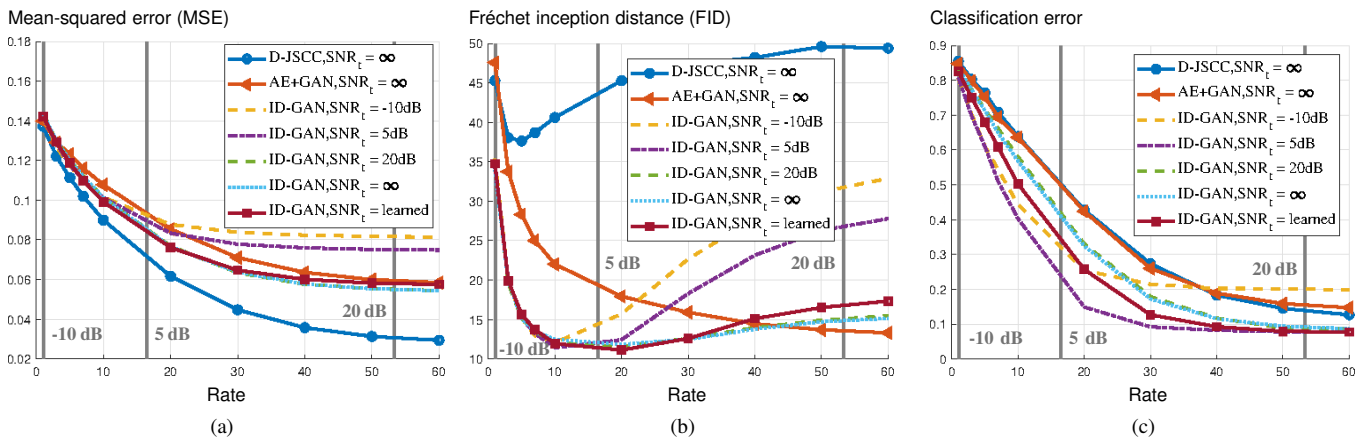

\begin{figure}
  \def\wdres{6.3cm}
  \def\hdres{5.9cm}
  \def\windfig{1.73cm}
  \begin{pspicture}(0,-1.8)(\linewidth,\hdres)
    \rput[lb](1.2,3.85){\includegraphics[width=1.59cm,height=1.65cm]{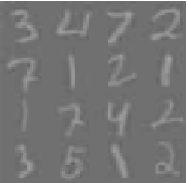}}
    \rput[lb](1.2,1.94){\includegraphics[width=1.59cm,height=1.65cm]{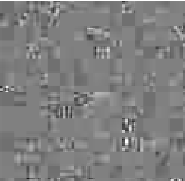}}
    \rput[lb](1.2,-0.02){\includegraphics[width=1.59cm,height=1.71cm]{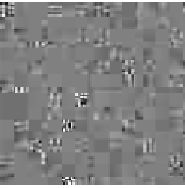}}
    \rput[lb](3.0,0){\includegraphics[width=\windfig]{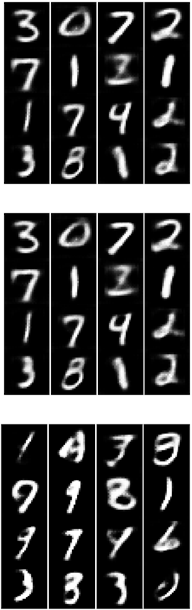}}
    \rput[lb](5.0,0){\includegraphics[width=\windfig]{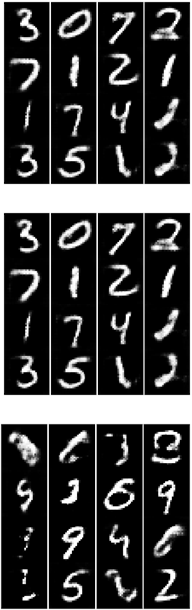}}
    \rput[lb](7.0,0){\includegraphics[width=1.675cm]{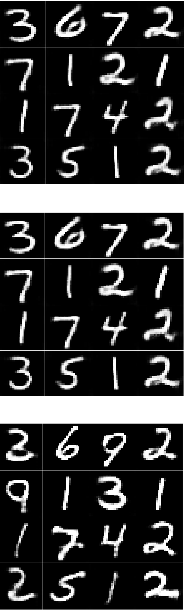}}

    \rput[lt](1.2,\hdres){\small Traditional}
    \rput[lt](3.0,\hdres){\small D-JSCC~\cite{deep-jscc}}
    \rput[lt](5.0,\hdres){\small AE+GAN~\cite{agustsson2019generative}}
    \rput[t](7.85,\hdres){\small \textcolor{red}{\textbf{ID-GAN}}}

    \rput[lt](0,\hdres){\textbf{Rate}}
    \rput[l](0,4.8){$10^{3}$}
    \rput[l](0,2.9){$10^{2}$}
    \rput[l](0,0.9){$10^{1}$}

    \rput[l](0,-1.0){\textbf{\textcolor{red}{Input}}}
    \rput[lt](1.2,-0.2){\includegraphics[width=1.59cm]{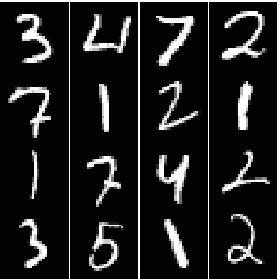}}

  \end{pspicture}
  \vspace{-0.5cm}
  \caption{Example images reconstructed by a traditional method
    (JPEG+LDPC+BPSK), D-JSCC, AE+GAN
    training, and proposed ID-GAN with learned $\text{SNR}_t$ and parameters $(\lambda_g, \lambda_d, \lambda_p,
\lambda_c) = (1, 500, 1, 10^3)$.}
  \label{JSCM_results}
\end{figure}
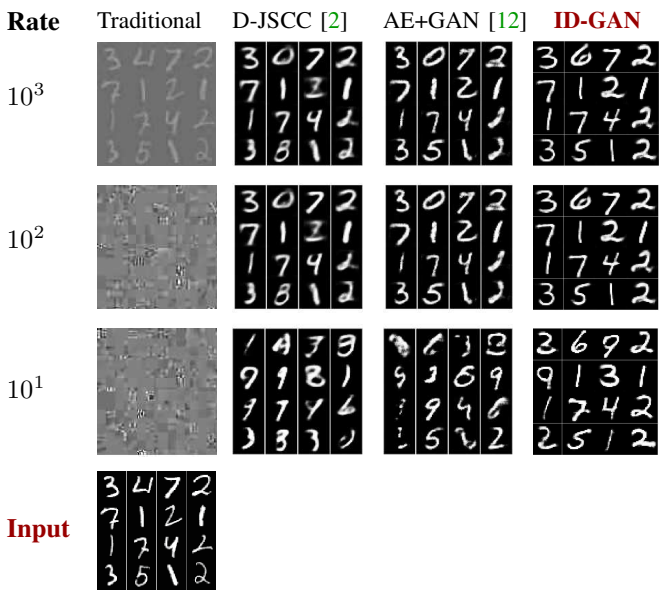

\begin{figure*}
  \subfigure[]{
    \label{subfig:rate-distortion-IDGAN-lambdas}
    \psscalebox{1.0}{	
      \begin{pspicture}(0,-2)(8.2,7.8)

        \rput[b](2.2,7.8){$(1, 1, 1)$}
        \rput[b](4.6,7.8){$(1, 1, 100)$}
        \rput[b](7.1,7.8){$(1, 100, 1)$}
        \rput[br](8.0,0.0){\includegraphics[height=7.7cm]{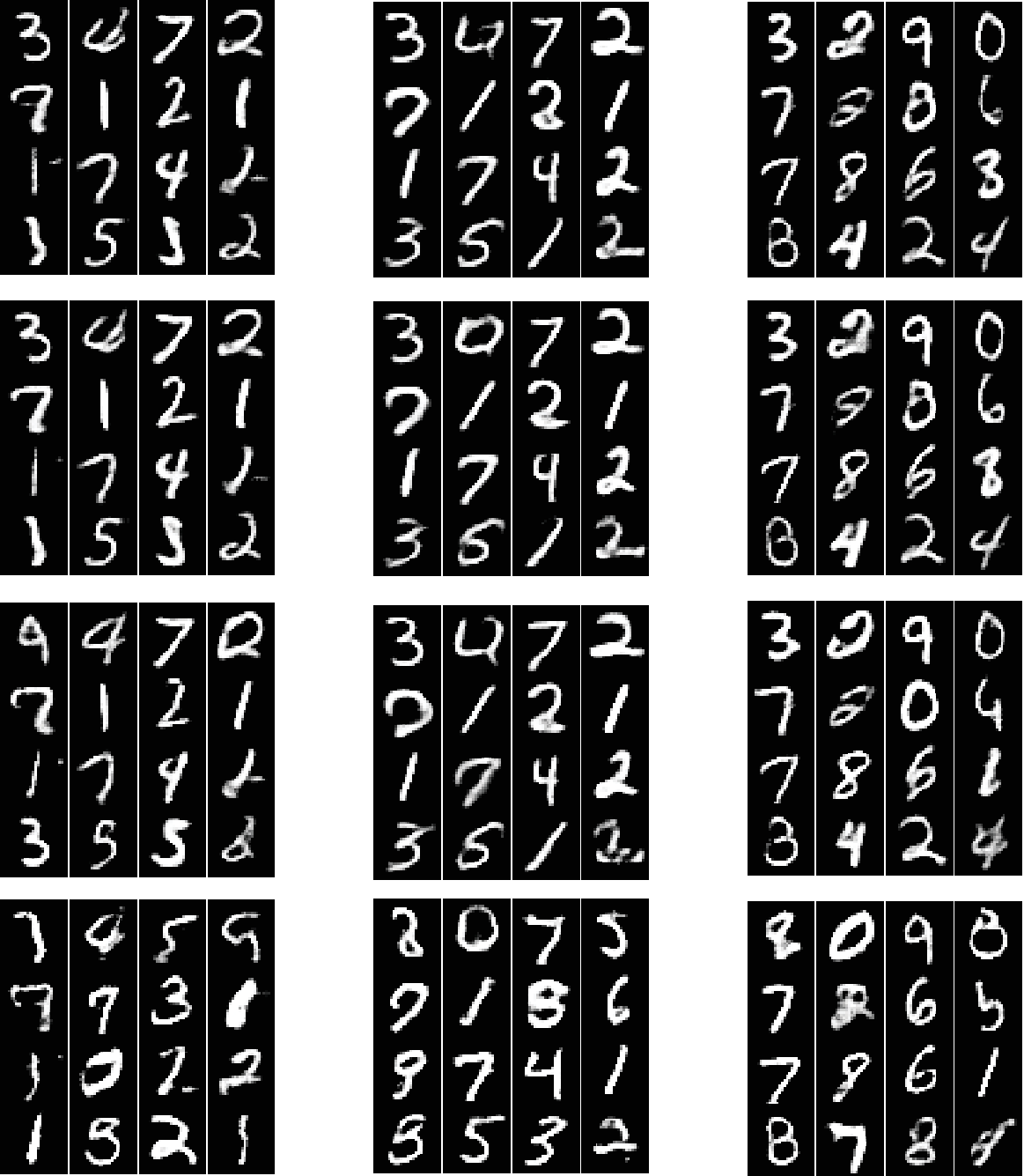}}

        \rput[l](0,8.0){\textbf{Rate}}
        \rput[l](0,6.70){$120$}
        \rput[l](0,4.85){$90$}
        \rput[l](0,2.90){$60$}
        \rput[l](0,0.95){$30$}

        \rput[l](0,-1.0){\textbf{\textcolor{red}{Input}}}
        \rput[lt](1.3,-0.2){\includegraphics[width=1.8cm]{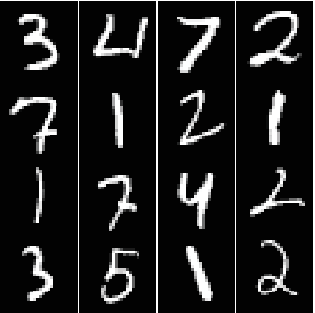}}

      \end{pspicture}
    }		
  }
  \hfill
  \subfigure[]{
    \label{subfig:rate-distortion-IDGAN-m}
    \psscalebox{1.0}{	
      \begin{pspicture}(0,-2)(8.2,7.8)

        \rput[b](2.45,7.8){$m = 2$}
        \rput[b](4.85,7.8){$m = 8$}
        \rput[b](7.1,7.8){$m = 64$}
        \rput[br](8.0,0.0){\includegraphics[height=7.7cm]{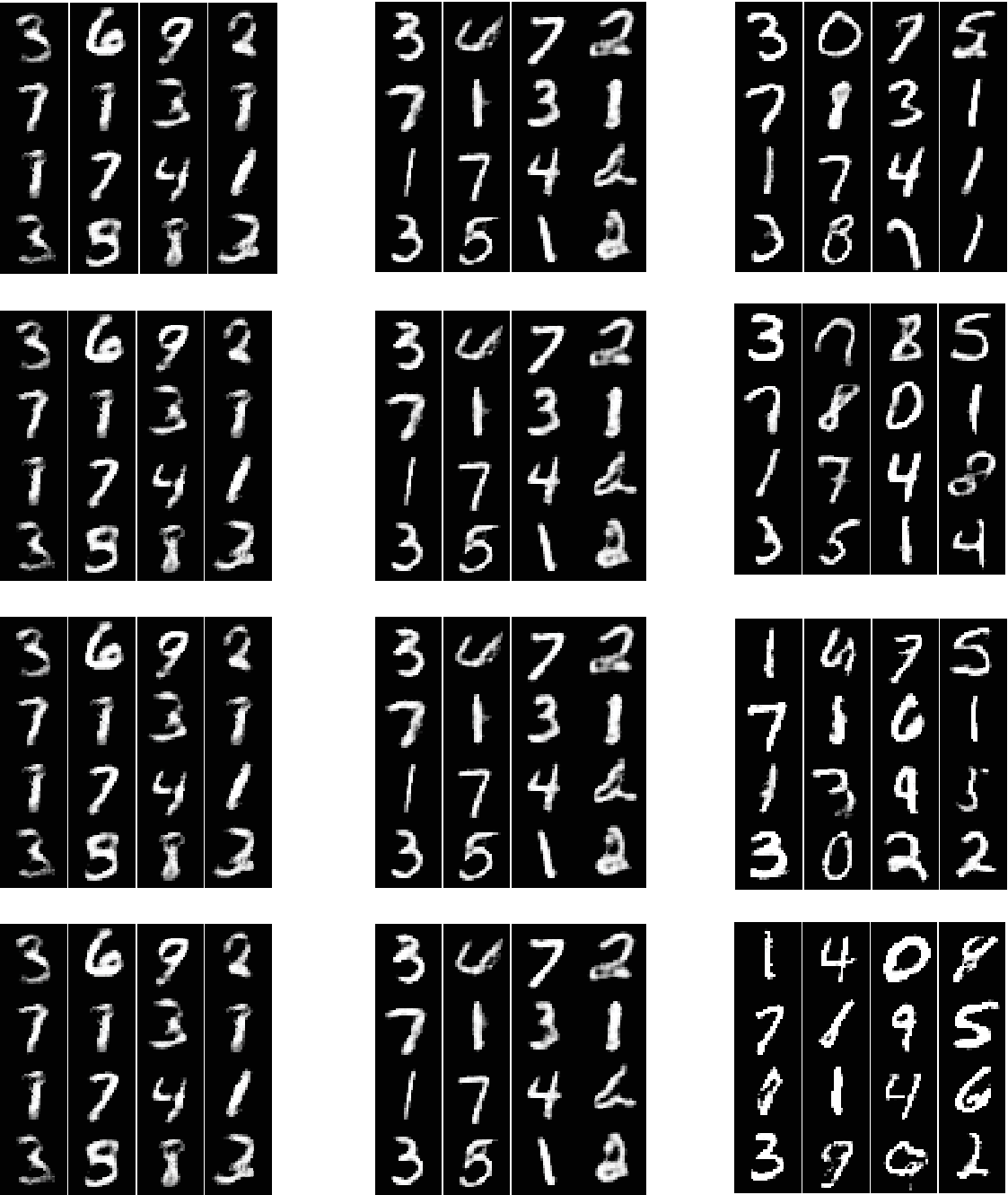}}

        \rput[lt](0,8.0){\textbf{Rate}}
        \rput[l](0,6.70){$120$}
        \rput[l](0,4.85){$90$}
        \rput[l](0,2.90){$60$}
        \rput[l](0,0.95){$30$}

        \rput[l](0,-1.0){\textbf{\textcolor{red}{Input}}}
        \rput[lt](1.49,-0.2){\includegraphics[width=1.78cm]{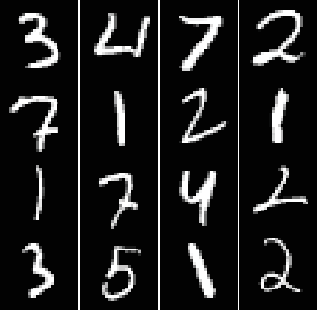}}

      \end{pspicture}
    }		
  }
  \caption{
    \text{(a)} Corresponding behavior of ID-GAN on MNIST images. The triples on
    top of each column represent hyperparameters $(\lambda_d, \lambda_p,
    \lambda_c)$. We fixed $\text{SNR}_t = \infty$. 
  \text{(b)} corresponding behavior of ID-GAN on
  MNIST images with fixed $\text{SNR}_t = \infty$.
  }
	\label{fig:rate-distortion-IDGAN}
\end{figure*}
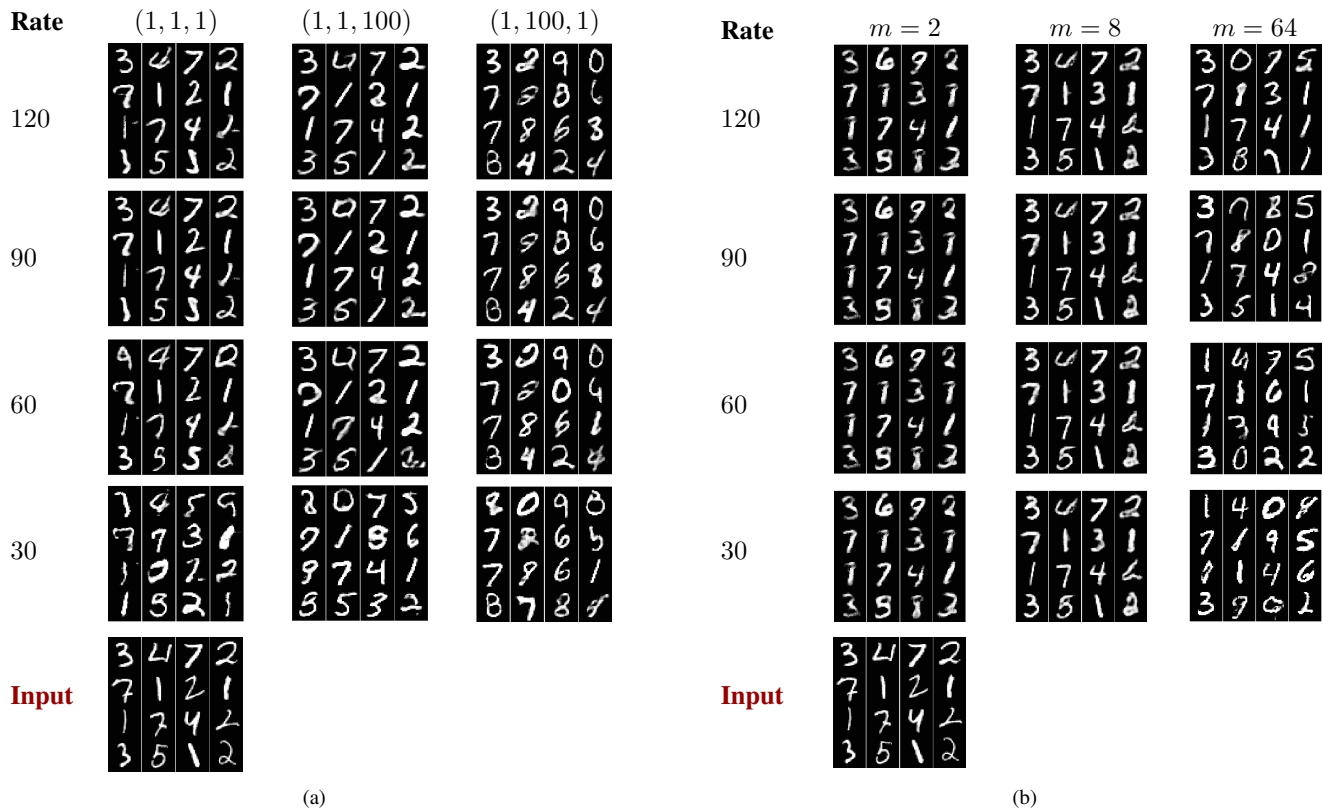

\begin{figure}[t]
  \def\hdres{2.9cm}
  \def\windfig{2.0cm}
  \begin{pspicture}(0,0)(\linewidth,\hdres)

    \rput[lb](0.0,0.8){\includegraphics[width=\windfig]{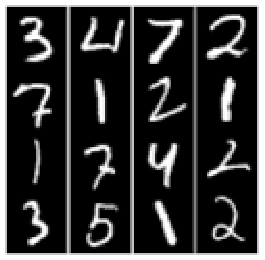}}
    \rput[lt](0.05,0.7){\textbf{\textcolor{red}{Input}}}
    \rput[lb](2.3,0.8){\includegraphics[width=\windfig]{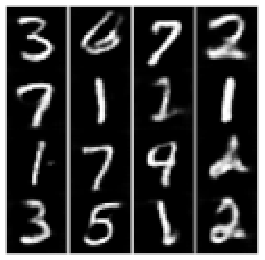}}
    \rput[lt](2.35,0.7){\small MSE loss}
    \rput[lt](2.35,0.3){\small (ID-GAN)}
    \rput[lb](4.6,0.8){\includegraphics[width=\windfig]{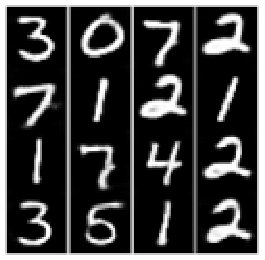}}
    \rput[lt](4.65,0.7){\small MSE \!+\! CE loss}
    \rput[lt](4.65,0.3){\small (ID-GAN)}
    \rput[lb](6.9,0.8){\includegraphics[width=\windfig]{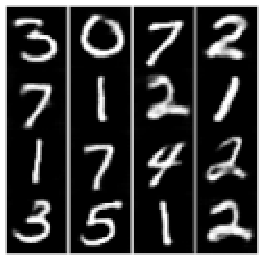}}
    \rput[lt](6.95,0.7){\textbf{\textcolor{red}{ID-GAN}}}

  \end{pspicture}
  \vspace{-0.5cm}
  \caption{
    Visualization of the results of the ablation study. In the second column,
    only the MSE term of~\eqref{eq:loss-encoder} was considered during
    training. In the third, both the MSE and the cross-entropy terms were
    considered. And in the fourth, the full
    loss~\eqref{eq:loss-encoder} was considered.
  }
  \label{fig:ablation}
\end{figure}

\subsection{ID-GAN algorithm}
\label{subsec:experiments-IDGAN}

We now consider ID-GAN (Algorithm~\ref{alg:IDGANTrain}) applied to the
popular MNIST dataset~\cite{LeCun98}, which has $60$k training images and $10$k
test images of size $28\times 28$ depicting digits from $0$ to $9$ ($10$
classes). Before describing the experimental setup in detail, we explain how
all the functions in the ID-GAN framework in Fig.~\ref{fig:ID-GAN} were
implemented.

\mypar{Network architectures}
Fig.~\ref{fig:arch-decoder-critics} shows the architectures of the decoder
$d(\cdot\, ;\, \bm{\theta}_d)$ [Fig.~\ref{fig:arch-decoder-critics}\text{(a)}]
and of the critics $f_1$ and $f_2$
[Fig.~\ref{fig:arch-decoder-critics}\text{(b)}]. The latter have the same
architecture, but they are initialized independently, with different seeds.
The decoder in Fig.~\ref{fig:arch-decoder-critics}\text{(a)} increases the
dimensions of the data by first using a fully connected network, whose output
is reshaped to a $4\times 4\times 256$ tensor, and then by upsampling
along the channel. 
The upsampling is performed via transposed convolutions with ReLU activations.
The architecture of the critics [Fig.~\ref{fig:arch-decoder-critics}\text{(b)}]
is symmetric to that of the decoder. The input image is compressed via a
convolutional network, whose output in the last layer is mapped to a
probability vector with a sigmoid function. The architecture of the encoder, on
the other hand, downsamples the input image using only fully connected layers, as
shown in Fig.~\ref{fig:arch-encoder}. 
The network architecture of the classifier [cf.\ Fig.~\ref{fig:ID-GAN}, bottom]
is similar to the one of the encoder, except that 
the last layer is mapped to a 10-dimensional vector (coinciding with the number
of classes of MNIST) and a sigmoid is applied to each entry. We train the
classifier beforehand and fix it when training the encoder.

\mypar{Experimental setup}
In Algorithm~\ref{alg:IDGANTrain}, we used a learning rate $\alpha = 10^{-5}$
and acceleration parameters $\beta_1 = 0.5$ and $\beta_2 = 0.9$ for Adam, a
batch size of $S = 50$, and $n_{\text{critic}} = 5$ iterations for the inner loop
of the critic. The loss hyperparameters and training noise level $\text{SNR}_t$
will be reported for each experiment. Also, we ran the algorithm for just $5$
epochs. The reason for such a small number is that, as described in
Section~\ref{subsec:IDGAN-trainingencoder}, the decoder has already been trained
when we run Algorithm~\ref{alg:IDGANTrain}. In fact, the decoder was also
trained with few epochs, $8$.

\mypar{Algorithms}
We compare the proposed ID-GAN algorithm against D-JSCC~\cite{deep-jscc}, the
parallel autoencoder-GAN (AE+GAN)~\cite{agustsson2019generative}, and a traditional approach in which
source coding, channel coding, and modulation are designed separately: 
JPEG and Huffman codes for source coding, $3/4$ LDPC codes for channel coding,
and BPSK for modulation~\cite{Gallager62-LDPC}. 

\mypar{Metrics and comparison}
For comparison metrics, we selected the mean squared error (MSE), the Fréchet
inception distance (FID)~\cite{heusel2017gans}, and the classification error
$(1/|\mathcal{V}|)\sum_{v\in \mathcal{V}}\mathbbm{1}_{c_0(\widehat{\bm{x}}^{(v)}) \neq
\bm{\ell}^{(v)}}$, where $\mathcal{V}$ is the test set, and
$\mathbbm{1}$ the $0$-$1$ loss.
FID captures perception quality by measuring the similarity between
distributions of real and generated images. The smaller the FID, the closer
the distributions. In summary, the smaller all the metrics, the better the
performance. Comparing the performance of a JSCM system against a traditional system,
however, is challenging. For example, \cite{deep-jscc} proposed to use the ratio
of bandwidth compression. Yet, in a traditional system, it is not obvious how
to accurately determine the ratio between the size of images and the
corresponding vectors in IQ-domain. 
Instead, we will use rate as defined in~\eqref{eq:rate}, i.e., channel
rate, which quantifies the amount of information that a symbol can convey
through a channel. In particular, the
dimension $m$ of the latent variable $\bm{y}$ corresponds to the number of
constellations in a traditional system.
Thus, to compare the different algorithms in a fair way, we fix during testing the channel
rate~\eqref{eq:rate}, making the results
independent from the latent dimension $m$. 

\mypar{Results}
Fig.\ \ref{fig:comparison-mse-fid-acc} shows how the above metrics
vary with the rate~\eqref{eq:rate}, measured in bits/image, for D-JSCC~\cite{deep-jscc},
AE+GAN~\cite{agustsson2019generative}, and the proposed ID-GAN. 
The curves represent the average values over the $10$k test images of
MNIST. For ID-GAN, we
fixed the hyperparameters as
$(\lambda_g, \lambda_d, \lambda_p, \lambda_c) = (1, 10^3, 1, 10^3)$ and
considered four different \textit{fixed} training noise levels, $\text{SNR}_t
\in \{-10, 5, 20, \infty\}$ \!dB [$\text{SNR}_t = \infty$ corresponds to no noise], and
also a value of $\text{SNR}_t$ that is \textit{learned} during training.
The vertical lines in Figs.\
\ref{subfig:rate-mse}-\ref{subfig:rate-acc} indicate the rates corresponding
to $\text{SNR}_t = -10$ dB, $5$ dB and $20$ dB.

Fig.~\ref{subfig:rate-mse} indicates that image distortion, measured by
the MSE, decreases with the rate for all the algorithms. D-JSCC outperforms
all the other methods, as it was designed to minimize image distortion (MSE).
Indeed, according to the RDPC tradeoff (Theorem~\ref{thm:convexity}), if
the perception and classification metrics are ignored, an algorithm can achieve a
smaller distortion for a given rate. When ID-GAN is trained with fixed
$\text{SNR}_t$, we observe that the larger the value of $\text{SNR}_t$,
the lower the achieved MSE for all rates. As we need to take into account other metrics (FID,
classification error), this indicates the difficulty of manually selecting
$\text{SNR}_t$. The version of ID-GAN with learned $\text{SNR}_t$ found an
optimal value of $\text{SNR}_t = 16.5$ dB during training, and its MSE performance follows closely
the performance of ID-GAN versions with $\text{SNR}_t = 20$ dB and
$\text{SNR}_t = \infty$ (the lines of latter practically coincide).

In Fig.~\ref{subfig:rate-fid}, however, 
D-JSCC is the algorithm with the worst FID performance, with an optimal point around
a rate of 5 bits/image. For AE+GAN, the FID metric decreases monotonically,
eventually surpassing ID-GAN trained with $\text{SNR}_t = 20$ dB and $\infty$. Indeed,
these versions of ID-GAN were the best overall in terms of FID performance,
again closely followed by ID-GAN with learned $\text{SNR}_t$.
When ID-GAN is trained with $\text{SNR}_t = -10$ dB or $5$ dB, FID reaches a
minimum around a rate of 10 bits/image and then it increases.
When $\text{SNR}_t$ is learned during training, 
its FID value decreases until a rate of 20 bits/image, and then it increases
gradually. This implies that when $\text{SNR}_t$ is learned, the perception
quality of ID-GAN is stable for a certain range of the rate.
The reason why FID values increase for larger rates in ID-GAN is likely
due to the mismatch between the noise levels during testing and training:
indeed, the smaller the training noise level, the earlier the FID curve starts
to increase. This happens only for FID likely because 
the decoder, which is largely responsible for the perception quality of the
output, is fixed during the training of the encoder. This decoupling may make
perception quality more sensitive to discrepancies between training/testing
setups.

Fig.~\ref{subfig:rate-acc} shows how the classification error of the algorithms
varies with the rate. All the versions of ID-GAN [except for $\text{SNR}_t =
-10$ dB, for rate $> 40$ bits/image] outperform both D-JSCC and AE+GAN, which
have almost overlapping lines. Within ID-GAN versions with fixed
$\text{SNR}_t$, we see that $\text{SNR}_t = 5$ dB yields the best
classification error across all the rates. Decreasing it to $-10$ dB or
increasing it to $20$ dB or $\infty$ results in a worse classification error.
The downside of such a good performance in terms of classification error is the
poor performance in terms of distortion [Fig.~\ref{subfig:rate-mse}] and
perception [Fig.~\ref{subfig:rate-fid}] for large rates. ID-GAN with learned
$\text{SNR}_t$, on the other hand, achieves a good performance across all
metrics [Figs.~\ref{subfig:rate-mse}-\ref{subfig:rate-acc}], effectively
balancing rate and all the metrics. 
We conclude that ID-GAN achieves a
tradeoff better than AE+GAN, outperforming it in all metrics for rates
$\lesssim 40$ bits/image. It also achieves perception and classification
performance better than D-JSCC, at the cost of worse distortion.

\mypar{Visual illustration}
Fig.~\ref{JSCM_results} shows concrete image examples reconstructed by all
the algorithms. 
In this experiment, the latent dimension $m$ in
D-JSCC~\cite{deep-jscc}, AE+GAN~\cite{agustsson2019generative}, and ID-GAN was
fixed to $8$ and the respective rate was then computed via~\eqref{eq:rate}.
The parameters of ID-GAN were set to $(\lambda_g,\lambda_d,\lambda_p,
\lambda_c) = (1, 500, 1, 10^3)$, and $\text{SNR}_{t}$ was learned.
The first column depicts the images
reconstructed by a traditional JPEG+LDPC+BPSK system. Even when the rate is
high, the reconstructed images fail to match the input one, likely due to
quantization and compression artifacts in
JPEG. As the rate decreases, we observe a \textit{cliff effect}, with
the image quality degrading abruptly.
D-JSCC~\cite{deep-jscc} in the second column, on the other hand, is based on an autoencoder which,
given the high bandwidth ratio, outputs blurry images at all rates. As in
Fig.~\ref{fig:comparison-mse-fid-acc}, even though D-JSCC outperforms the other
algorithms in MSE, it has a relatively poor perception quality (FID) and
classification error.
In the third column, the images reconstructed by
AE+GAN~\cite{agustsson2019generative} have poor perceptual quality for a low
rate, i.e., $10$. We noticed that the quality of its output images
collapses when its training parameters $\lambda_d$, $\lambda_p$, and
$\lambda_c $ are not carefully set, revealing the difficulty in balancing
different loss terms.
The proposed ID-GAN scheme shown in the fourth column, in contrast, not only
preserves semantic information in images, but also reconstructs them with high
perceptual quality. Compared to D-JSCC and AE+GAN, ID-GAN generates images with
various styles, including different angles and thickness, even at extremely
high bandwidth compression ratios. 

\mypar{Rate-distortion analysis}
To illustrate how ID-GAN and RDPCO behave similarly,
Fig.~\ref{fig:rate-distortion-IDGAN} adapts to MNIST the setup of
Fig.~\ref{fig:rate-distortion-RDPCO}, which shows rate-distortion curves of
RDPCO with varying $(P, C)$ or $m$. 
We fixed the training noise level of ID-GAN to $\text{SNR}_t = \infty$ so that
we can visualize the effect of modifying its hyperparameters. 
Fig.~\ref{subfig:rate-distortion-IDGAN-lambdas} shows, for different rates, a
behavior similar to RDPCO in
Fig.~\ref{subfig:rate-distortion-RDPCO-lambdas} when we modify its
hyperparameters $(\lambda_d, \lambda_p, \lambda_c)$ accordingly. 
Note that these hyperparameters apply only during training; during testing, we
inject different levels of channel noise, obtaining different rates. The first column
of Fig.~\ref{subfig:rate-distortion-IDGAN-lambdas} weighs all the metrics
equally. The second column imposes a more stringent requirement on classification
performance by setting $\lambda_c$ two orders of magnitude larger than
$\lambda_d$ and $\lambda_p$ [akin to decreasing $C$ in~\eqref{eq:boundRDPC}].
While classification accuracy is maintained, we observe loss of details, like
digit orientation and line thickness. The third column imposes a larger weight
on perception. At low rates, even though semantic information is altered, our
algorithm still generates meaningful digits. 

Fig.~\ref{subfig:rate-distortion-IDGAN-m}, akin to
Fig.~\ref{subfig:rate-distortion-RDPCO-m}, shows how different values of $m$
affect the output images of ID-GAN for different rates. Images in the first
column, for $m=2$, are blurry and discontinuous, indicating that they may
disregard the perception and classification losses. However, when $m=8$, the
output images seem to be more suitable for transmission at all rates,
preserving perception, but with a few semantic mistakes, e.g., a digit $2$
becoming a $3$. When $m=64$, the algorithm requires higher rates to reconstruct
the images accurately. 

\mypar{Ablation study}
We performed a small ablation study to assess whether all the terms in the loss
associated to the encoder~\eqref{eq:loss-encoder} are necessary for good
performance. To do so, we fixed the noise level at $\text{SNR}_t = 30$ dB,
corresponding to a rate of $50$, which is large enough to enable seeing
visual differences in the reconstructed images. 
First, we trained the encoder with just the MSE loss, i.e., $(\lambda_d,
\lambda_p, \lambda_c) = (1, 0, 0)$ in~\eqref{eq:loss-encoder}. A set of reconstructed digits from the test
set is shown in the second column of Fig.~\ref{fig:ablation}. The digits are
blurry, as the encoder learns just to compress pixel-level information and
disregards any semantic information. 
Then, we trained the encoder with the MSE and classification loss, i.e., $(\lambda_d,
\lambda_p, \lambda_c) = (1, 0, 40)$. As shown in the third column of
Fig.~\ref{fig:ablation}, the algorithm preserved semantic
information better, e.g., correctly depicting a $4$ in the $3$rd row and column, but the
digits exhibit poor diversity. For example, all the 2's look similar.
Finally, we trained the encoder using the full loss, with $(\lambda_d,
\lambda_p, \lambda_c) = (1, 1, 40)$. The reconstructed digits in the last column
of Fig.~\ref{fig:ablation} are not only sharp, of the correct class (except for
the $4$ in row 1, column 2) but
also exhibit more diversity, thus looking more realistic.

\section{Conclusions}
\label{sec:conclusions}

We formulated and analyzed the tradeoff between rate, distortion, perception,
and rate (RDPC) in a joint source coding and modulation (JSCM) framework. We
showed the existence of a tradeoff and proposed two algorithms to achieve it.
One algorithm is heuristic and was designed under simplifying
assumptions to minimize an upper bound on the RDPC function; the other was
based on inverse-domain GAN (ID-GAN) and works under a general scenario.
Experimental results showed that ID-GAN achieves a better tradeoff than
a traditional method, in which source coding and modulation are designed
separately, and a tradeoff better or similar to recent deep joint
source-channel coding schemes. Experiments revealed that improving perception
quality and classification accuracy require higher rates, and also showed the
existence of an optimal compressed/latent dimension that minimizes rate while
satisfying constraints on distortion, perception, and classification.

\section*{Acknowledgments} 
We thank the two reviewers for their insightful suggestions, which
significantly improved the quality of the paper. 

\printbibliography

\clearpage

\appendices

\newtheorem*{Theorem*}{Theorem}                
\newtheorem*{Lemma*}{Lemma}                    

\section{Proof of Theorem~\ref{thm:convexity}}
\label{app:convexity}

For convenience, we restate Theorem~\ref{thm:convexity}:
\begin{Theorem*}
  Let $\bm{X}$ be a multiclass model as in~\eqref{eq:SourceModel}. Consider
  the communication scheme in~\eqref{eq:channeldiagram} and the associated RDPC
  problem in~\eqref{eq:problem-general}. Assume 
  the classifier $c_0$ is deterministic and that the perception function
  $d(\cdot, \cdot)$ is convex in its second argument. Then, the function $R(D,
  P, C)$ is strictly convex, and it is non-increasing in each argument.
\end{Theorem*}
\begin{proof}
  If we increase either $D$, $P$, or $C$ in right-hand side
  of~\eqref{eq:problem-general}, the constraint set of the optimization problem 
  is enlarged or remains the same. This means that $R(D, P, C)$ is
  non-increasing with any of these variables.

  To show strict convexity, we take arbitrary pairs $(D_1, P_1, C_1)
  \geq 0$ and $(D_2, P_2, C_2) \geq 0$ and, for any $0 < \alpha < 1$, show
  that
  \begin{multline}
    \label{eq:proof-defconv}
    (1-\alpha)R(D_1, P_1, C_1) + \alpha R(D_2, P_2, C_2)
    \\
    >
    R\big(
    (1-\alpha)D_1 + \alpha D_2,\,
    (1-\alpha)P_1 + \alpha P_2,\,
    (1-\alpha)C_1 + \alpha C_2
    \big). 
  \end{multline}
  To do so, we define, for $j = 1, 2$, 
  \begin{align}
    \Big(
    p_{\bm{Y}\vert \bm{X}}^{(j)},\,
    p_{\widehat{\bm{X}}\vert \widehat{\bm{Y}}}^{(j)},\,
    \bm{\Sigma}^{(j)}
    \Big)
    &:= 
    \!\!\!
    \begin{array}[t]{cl}
      \underset{p_{\bm{Y}\vert\bm{X}},\,
      p_{\widehat{\bm{X}}\vert\widehat{\bm{Y}}},\,\bm{\Sigma}}
      {\text{argmin}} 
    & 
    \sum_{i=1}^{m}\log \big(1+\frac{1}{\bm{\Sigma}_{ii}}\big)
    \\ 
    \text{s.t.} 
    & 
    \mathbb{E}\big[\Delta(\bm{X}, \widehat{\bm{X}})\big] \leq D_j
    \\[0.1cm]
    & 
    d(p_{\bm{X}},\, p_{\widehat{\bm{X}}}) \leq P_j
    \\[0.1cm]
    &
    \mathbb{E}\big[\epsilon_{c_0}(\bm{X}, \widehat{\bm{X}})\big] \leq C_j,
    \end{array}
    \label{eq:proof-def-rdpcj}
    \end{align}
    and denote by $\widehat{\bm{X}}^{(j)}$ the output
    of~\eqref{eq:channeldiagram} with the parameters computed in~\eqref{eq:proof-def-rdpcj}. 
    Using the strict convexity of the function $x \mapsto \log(1 + 1/x)$ for $x
    > 0$, the
    left-hand side of~\eqref{eq:proof-defconv} equals
    \begin{align}
      &(1 -\alpha)R(D_1, P_1, C_1) + \alpha R(D_2, P_2, C_2)
      \notag
      \\
      &=
      \sum_{i}
      \bigg[
      (1-\alpha)
      \log\Big(1 + \frac{1}{\bm{\Sigma}_{ii}^{(1)}}\Big)
      +
      \alpha
      \log\Big(1 + \frac{1}{\bm{\Sigma}_{ii}^{(2)}}\Big)
      \bigg]
      \notag
      \\
      &>
      \sum_{i}
      \log\bigg(
        1 + \frac{1}{(1-\alpha)\bm{\Sigma}_{ii}^{(1)} + \alpha\bm{\Sigma}_{ii}^{(2)}}
      \bigg)
      \label{eq:proof-main1}
      \\
      &\geq
      \!\!
      \begin{array}[t]{cl}
        \underset{p_{\bm{Y}\vert\bm{X}},\,
        p_{\widehat{\bm{X}}\vert\widehat{\bm{Y}}},\,\bm{\Sigma}}
        {\min} 
          & 
          \sum_{i=1}^{m}\log \big(1+\frac{1}{\bm{\Sigma}_{ii}}\big)
          \\ 
          \text{s.t.} 
          & 
          \mathbb{E}\big[\Delta(\bm{X}, \widehat{\bm{X}})\big] \leq (1-\alpha)D_1 +
          \alpha D_2
          \\
          & 
          d(p_{\bm{X}},\, p_{\widehat{\bm{X}}}) \leq (1-\alpha)P_1 + \alpha P_2
          \\
          &
          \mathbb{E}\big[\epsilon_{c_0}(\bm{X}, \widehat{\bm{X}})\big] \leq (1-\alpha)C_1 + \alpha C_2
      \end{array}
      \label{eq:proof-main2}
      \\
      &=
      R\big(
        (1-\alpha)D_1 + \alpha D_2,\,
        (1-\alpha)P_1 + \alpha P_2,\,
      \notag
      \\
      &\qquad\qquad\qquad\qquad\qquad\qquad\qquad
        (1-\alpha)C_1 + \alpha C_2
      \big)\,. 
      \label{eq:proof-main3}
    \end{align}
    Step~\eqref{eq:proof-main2} to~\eqref{eq:proof-main3} follows from
    the definition of $R(D, P, C)$ in~\eqref{eq:problem-general}. 
    The rest of the proof will consist of showing that the step
    from~\eqref{eq:proof-main1} to~\eqref{eq:proof-main2} holds. Indeed, this
    will follow if we show that the triple 
    \begin{multline}
      \Big(
        (1-\alpha)p_{\bm{Y}\vert\bm{X}}^{(1)}+\alpha
        p_{\bm{Y}\vert\bm{X}}^{(2)}\,,\,\,
        (1-\alpha)p_{\widehat{\bm{X}}\vert\widehat{\bm{Y}}}^{(1)}
        +\alpha p_{\widehat{\bm{X}}\vert\widehat{\bm{Y}}}^{(2)}\,,\,\,
        \\
        (1-\alpha)\bm{\Sigma}^{(1)}+\alpha \bm{\Sigma}^{(2)}
      \Big)
      \label{eq:proof-channelmix}
    \end{multline}
    satisfies the constraints of the optimization problem
    in~\eqref{eq:proof-main2}.  

    First, notice that~\eqref{eq:proof-channelmix} defines valid parameters for
    the communication process in~\eqref{eq:channeldiagram}. Specifically,
    because convex combinations of probability distributions are also
    probability distributions, $(1-\alpha)p_{\bm{Y}\vert\bm{X}}^{(1)}+\alpha
    p_{\bm{Y}\vert\bm{X}}^{(2)}$  and $(1-\alpha)p_{\widehat{\bm{X}}\vert\widehat{\bm{Y}}}^{(1)}
    +\alpha p_{\widehat{\bm{X}}\vert\widehat{\bm{Y}}}^{(2)}$ characterize valid
    encoding and decoding processes. If $\bm{\Sigma}^{(1)}$ and
    $\bm{\Sigma}^{(2)}$ are diagonal positive definite matrices, then their
    convex combination also is. Let then $\widehat{\bm{X}}^{(\alpha)}$ denote
    the output of~\eqref{eq:channeldiagram} with the parameters
    in~\eqref{eq:proof-channelmix}. Notice that
    \begin{align}
      p_{\widehat{\bm{X}}^{(\alpha)}\vert \widehat{\bm{Y}}}
      =
      (1-\alpha)
      p_{\widehat{\bm{X}}\vert \widehat{\bm{Y}}}^{(1)}
      +
      \alpha
      p_{\widehat{\bm{X}}\vert \widehat{\bm{Y}}}^{(2)}\,.
      \label{eq:proof-conddistr}
    \end{align}
    We will show that $\widehat{\bm{X}}^{(\alpha)}$ and
    its probability distribution 
    \begin{align}
      p_{\widehat{\bm{X}}^{(\alpha)}}
      =
      (1-\alpha)
      p_{\widehat{\bm{X}}}^{(1)}
      +
      \alpha
      p_{\widehat{\bm{X}}}^{(2)}\,,
      \label{eq:proof-distrout}
    \end{align}
    where $p_{\widehat{\bm{X}}}^{(1)}$ and $p_{\widehat{\bm{X}}}^{(2)}$ are the
    distributions of the output of~\eqref{eq:channeldiagram} with the
    parameters in~\eqref{eq:proof-def-rdpcj},
    satisfy the constraints
    in~\eqref{eq:proof-main2}. Indeed, for the first constraint, conditioning
    on $\widehat{\bm{Y}}$,
    \begin{align}
      &
      \mathbb{E}\big[\Delta(\bm{X}, \widehat{\bm{X}}^{(\alpha)})\big]
      \notag
      \\
      &=
      \mathbb{E}_{\widehat{\bm{Y}}}
      \Big[
        \mathbb{E}\big[
        \Delta
        \big(
        \bm{X}, \widehat{\bm{X}}^{(\alpha)}\big)\, \vert\, \widehat{\bm{Y}}
        \big]
      \Big]   
      \label{eq:proof-distortion1}
      \\
      &=
      \mathbb{E}_{\widehat{\bm{Y}}}
      \Big[
        (1-\alpha)
        \mathbb{E}\big[
        \Delta
        \big(
        \bm{X}, \widehat{\bm{X}}^{(1)}\big)\, \vert\, \widehat{\bm{Y}}
        \big]
        \notag
        \\
      &\qquad\qquad\qquad\qquad\qquad\quad
        +
        \alpha
        \mathbb{E}\big[
        \Delta
        \big(
        \bm{X}, \widehat{\bm{X}}^{(2)}\big)\, \vert\, \widehat{\bm{Y}}
        \big]
      \Big]
      \label{eq:proof-distortion2}
      \\
      &=
      (1-\alpha)
      \mathbb{E}\big[
      \Delta
      \big(
      \bm{X}, \widehat{\bm{X}}^{(1)}\big)
      \big]
      + 
      \alpha
      \mathbb{E}\big[
      \Delta
      \big(
      \bm{X}, \widehat{\bm{X}}^{(2)}\big)
      \big]
      \label{eq:proof-distortion3}
      \\
      &\leq
      (1-\alpha)D_1 + \alpha D_2\,.
      \label{eq:proof-distortion4}
    \end{align}
    In~\eqref{eq:proof-distortion1} and~\eqref{eq:proof-distortion3}, we
    applied the tower property of expectation.
    From~\eqref{eq:proof-distortion1} to~\eqref{eq:proof-distortion2}, we
    used~\eqref{eq:proof-conddistr}. And from~\eqref{eq:proof-distortion3}
    to~\eqref{eq:proof-distortion4}, we used~\eqref{eq:proof-def-rdpcj}.
    For the second constraint, we use the assumption that $d(\cdot, \cdot)$ is
    convex in its second argument and, again, \eqref{eq:proof-conddistr}
    and~\eqref{eq:proof-def-rdpcj}:
    \begin{align*}
      d\big(p_{\bm{X}},\, p_{\widehat{\bm{X}}^{(\alpha)}}\big)
      &=
      d\big(p_{\bm{X}},\,
        (1-\alpha)
        p_{\widehat{\bm{X}}}^{(1)}
        +
        \alpha
        p_{\widehat{\bm{X}}}^{(2)}
      \big)
      \\
      &\leq
      (1-\alpha)
      d\big(p_{\bm{X}},\,p_{\widehat{\bm{X}}}^{(1)} \big)
      +
      \alpha
      d\big(p_{\bm{X}},\,p_{\widehat{\bm{X}}}^{(2)} \big)\,.
      \\
      &\leq
      (1-\alpha)P_1 + \alpha P_2\,.
    \end{align*}
    Finally, for the last constraint, we plug $\widehat{\bm{X}}^{(\alpha)}$
    into~\eqref{eq:probclasserror}:
    \begin{align}
      \mathbb{E}\Big[\mathcal{\epsilon}_{c_0}&(\bm{X},\,\widehat{\bm{X}}^{(\alpha)})\Big]
      =
      \sum_{i < j}
      p_j\cdot
      \int_{\mathcal{R}_i}\dif\,p_{\widehat{\bm{X}}^{(\alpha)}\vert H_j}
      \label{eq:proof-classerror1}
      \\
      &=
      \sum_{i < j}
      p_j\cdot
      \int
      \int_{\mathcal{R}_i}\dif\,p_{\widehat{\bm{X}}^{(\alpha)}\vert \widehat{\bm{Y}},\,H_j}
      \dif\,p_{\widehat{\bm{Y}}}
      \label{eq:proof-classerror2}
      \\
      &=
      \sum_{i < j}
      p_j\cdot
      \int
      \int_{\mathcal{R}_i}\dif\,p_{\widehat{\bm{X}}^{(\alpha)}\vert \widehat{\bm{Y}}}
      \dif\,p_{\widehat{\bm{Y}}}
      \label{eq:proof-classerror3}
      \\
      &=
      (1-\alpha)
      \sum_{i < j}
      p_j\cdot
      \int
      \int_{\mathcal{R}_i}\dif\,p_{\widehat{\bm{X}}^{(1)}\vert \widehat{\bm{Y}}}
      \dif\,p_{\widehat{\bm{Y}}}
      \notag
      \\
      &\qquad\qquad\qquad
      +
      \alpha
      \sum_{i < j}
      p_j\cdot
      \int
      \int_{\mathcal{R}_i}\dif\,p_{\widehat{\bm{X}}^{(2)}\vert \widehat{\bm{Y}}}
      \dif\,p_{\widehat{\bm{Y}}}
      \label{eq:proof-classerror4}
      \\
      &=
      (1-\alpha)\mathbb{E}\big[\epsilon_{c_0}\big(\bm{X},\,
        \widehat{\bm{X}}^{(1)}\big)\big]
      +
      \alpha\mathbb{E}\big[\epsilon_{c_0}\big(\bm{X},\,
        \widehat{\bm{X}}^{(2)}\big)\big] 
      \label{eq:proof-classerror5}
      \\
      &\leq
      (1-\alpha)C_1 + \alpha C_2\,.
      \label{eq:proof-classerror6}
    \end{align}
    From~\eqref{eq:proof-classerror1} to~\eqref{eq:proof-classerror2}, we
    conditioned on $\widehat{\bm{Y}}$. From~\eqref{eq:proof-classerror2}
    to~\eqref{eq:proof-classerror3}, we used the Markov property
    of~\eqref{eq:channeldiagram}. From~\eqref{eq:proof-classerror3}
    to~\eqref{eq:proof-classerror4}, we used~\eqref{eq:proof-conddistr}.
    From~\eqref{eq:proof-classerror4} to~\eqref{eq:proof-classerror5}, we
    applied the same steps as in~\eqref{eq:probclasserror}, but in reverse
    order. And, finally, from~\eqref{eq:proof-classerror5}
    to~\eqref{eq:proof-classerror6}, we used~\eqref{eq:proof-def-rdpcj}.
\end{proof}

\section{Proof of Lemma~\ref{lem:Wasserstein}}
\label{app:wasserstein}

For convenience, we restate Lemma~\ref{lem:Wasserstein}: 
\begin{Lemma*}
  Let $p_{\bm{X}}$ (resp.\ $p_{\widehat{\bm{X}}}$) be a
  Gaussian mixture model following~\eqref{eq:GMMModel} [resp.\
  \eqref{eq:outputsignals}], in which the probability of hypothesis $H_0$ is
  $p_0$ and of hypothesis $H_1$ is $p_1 = 1 - p_0$. Then, 
  \begin{align*}
    W_1(p_{\bm{X}},\, p_{\widehat{\bm{X}}})
    \leq
    \big\|
    \widehat{\bm{\Sigma}}^{\frac{1}{2}}
    - 
    \bm{I_n}
    \big\|_F
    +
    \big\|\bm{D}\bm{E}\bm{c_n}-\bm{c_n}\big\|_2\cdot p_1\,.
  \end{align*}
\end{Lemma*}
\begin{proof}
  We use the dual form of Wasserstein-$1$ distance
  in~\eqref{eq:WassersteinDistance}:
  \begin{align}
    &
    W_1\big(p_{\bm{X}},\, p_{\widehat{\bm{X}}}\big)
    =
    \underset{\|f\|_L \leq 1}{\sup}\,\,\,
    \mathbb{E}_{\bm{X}\sim p_{\bm{X}}}
    \big[f(\bm{X})\big]
    -
    \mathbb{E}_{\bm{X}\sim p_{\widehat{\bm{X}}}}
    \big[f(\bm{X})\big]
    \notag
    \\
    &\leq
    \underset{\|f\|_L \leq 1}{\sup}\,\,\,
    \bigg(
    \mathbb{E}_{\bm{X}\sim p_{\bm{X}}}
    \Big[f(\bm{X}) \,\vert\, H_0\Big]
    -
    \mathbb{E}_{\bm{X}\sim p_{\widehat{\bm{X}}}}
    \Big[f(\bm{X})\,\vert\, H_0\Big]
    \bigg)p_0
    \notag
    \\
    &\,+
    \underset{\|\widehat{f}\|_L \leq 1}{\sup}\,\,\,
    \bigg(
    \mathbb{E}_{\bm{X}\sim p_{\bm{X}}}
    \Big[\widehat{f}(\bm{X}) \,\vert\, H_1\Big]
    -
    \mathbb{E}_{\bm{X}\sim p_{\widehat{\bm{X}}}}
    \Big[\widehat{f}(\bm{X})\,\vert\, H_1\Big]
    \bigg)p_1
    \label{eq:proof-wasserstein-gmm1}
    \\
    &=:
    W_1\Big(p_{\bm{X}},\, p_{\widehat{\bm{X}}}\,\vert\,
    H_0\Big)\cdot p_0
    +
    W_1\Big(p_{\bm{X}},\, p_{\widehat{\bm{X}}}\,\vert\,
    H_1\Big)\cdot p_1
    \label{eq:proof-wasserstein-gmm2}
    \\
    &\leq
    W_2\Big(p_{\bm{X}},\, p_{\widehat{\bm{X}}}\,\vert\,
    H_0\Big)\cdot p_0
    +
    W_2\Big(p_{\bm{X}},\, p_{\widehat{\bm{X}}}\,\vert\,
    H_1\Big)\cdot p_1
    \label{eq:proof-wasserstein-gmm3}
    \\
    &=
    \big\|
    \widehat{\bm{\Sigma}}^{\frac{1}{2}}
    - 
    \bm{I_n}
    \big\|_2 p_0
    +
    p_1
    \sqrt
    {
    \big\|
    \widehat{\bm{\Sigma}}^{\frac{1}{2}}
    - 
    \bm{I_n}
    \big\|_F^2
    +
    \big\|\bm{c}_n - \bm{D}\bm{E}\bm{c}_n \big\|_2^2
    }
    \label{eq:proof-wasserstein-gmm4}
    \\
    &\leq
    \big\|
    \widehat{\bm{\Sigma}}^{\frac{1}{2}}
    - 
    \bm{I_n}
    \big\|_F
    +
    \big\|\bm{c}_n - \bm{D}\bm{E}\bm{c_n}\big\|_2\cdot p_1\,.
    \notag
  \end{align}
  In~\eqref{eq:proof-wasserstein-gmm1}, we first conditioned on $H_0$ and
  $H_1$, and then used the subadditivity of the supremum. The inequality is due
  to using different variables $f$ and $\widehat{f}$.
  From~\eqref{eq:proof-wasserstein-gmm1} to~\eqref{eq:proof-wasserstein-gmm2},
  we defined the Wasserstein-$1$ conditional on an event.
  From~\eqref{eq:proof-wasserstein-gmm2} to~\eqref{eq:proof-wasserstein-gmm3},
  we used the fact that $W_p(\cdot,\, \cdot) \leq W_q(\cdot,\, \cdot)$ whenever
  $p \leq q$; see~\cite[Remark 6.6]{Villani09-OptimalTransport}.
  From~\eqref{eq:proof-wasserstein-gmm3} to~\eqref{eq:proof-wasserstein-gmm4},
  we applied~\eqref{eq:Wasserstein-simple} to the models in~\eqref{eq:GMMModel}
  and~\eqref{eq:outputsignals}. And in the last step, we used the triangular
  inequality. 
\end{proof}



\end{document}